\documentclass[11pt, twoside]{article}
\usepackage{arxiv}

\usepackage{amsfonts}
\usepackage{amsmath}
\usepackage{comment}
\usepackage{amssymb}
\usepackage{amsthm}
\usepackage{bbm}
\usepackage{tabularx}
\usepackage{booktabs}
\usepackage{placeins}
\usepackage[caption=false]{subfig}
\usepackage{color,colortbl}
\usepackage[dvipsnames]{xcolor}
\usepackage{afterpage}
\usepackage{pdflscape}
\usepackage{hyperref}
\usepackage{cleveref}
\usepackage{acronym}
\usepackage[numbers]{natbib}
\usepackage{algorithm2e}
\usepackage{graphicx}
\usepackage{pifont}
\usepackage[font=footnotesize]{caption}
\usepackage{fancyhdr}
\crefformat{footnote}{#2\footnotemark[#1]#3}

\definecolor{LightGray}{gray}{0.9}

\newcommand\R{\mathbb{R}}

\newcommand\N{\mathbb{N}}

\newcommand{\interior}[1]{%
  {\kern0pt#1}^{\mathrm{o}}%
}

\newcommand{\closure}[1]{%
  \mkern 1.5mu\overline{\mkern-1.5mu#1\mkern-1.5mu}\mkern 1.5mu%
}

\DeclareMathOperator{\tr}{tr}
\DeclareMathOperator{\diag}{diag}
\DeclareMathOperator{\kl}{KL}
\DeclareMathOperator{\jf}{JF}
\DeclareMathOperator{\im}{im}
\newcommand{\fr}{\partial}

\DeclareMathOperator{\gl}{GL}

\DeclareMathOperator{\cov}{Cov}

\newcommand{\matrixnorm}[1]{{\left\vert\kern-0.25ex\left\vert\kern-0.25ex\left\vert #1 
    \right\vert\kern-0.25ex\right\vert\kern-0.25ex\right\vert}}

\newcommand{\istargetof}{\rightsquigarrow}

\theoremstyle{definition}
\newtheorem{definition}{Definition}
\newtheorem{example}{Example}
\theoremstyle{remark}
\newtheorem{remark}{Remark}

\theoremstyle{plain}
\newtheorem{theorem}{Theorem}
\newtheorem{proposition}{Proposition}
\newtheorem{corollary}{Corollary}
\newtheorem{lemma}{Lemma}

\title{A Tutorial on Distance Metric Learning: Mathematical Foundations, Algorithms, Experimental Analysis, Prospects and Challenges (with Appendices on Mathematical Background and Detailed Algorithms Explanation)}

\author{Juan Luis Su\'arez-D\'iaz \And Salvador Garc\'ia \And Francisco Herrera \\
DaSCI, Andalusian Research Institute in Data Science and Computational Intelligence \\
University of Granada \\ Granada, Spain
\\ \texttt{\{jlsuarezdiaz,salvagl,herrera\}@decsai.ugr.es}}

\hypersetup{
pdftitle={A Tutorial on Distance Metric Learning: Mathematical Foundations, Algorithms, Experimental Analysis, Prospects and Challenges (with Appendices on Mathematical Background and Detailed Algorithms Explanation)},
pdfsubject={Machine Learning},
pdfauthor={Juan Luis Su\'arez-D\'iaz, Salvador Garc\'ia, Francisco Herrera},
pdfkeywords={Distance Metric Learning, Classification, Mahalanobis Distance, Dimensionality, Similarity},
hypertexnames=false,
}

\date{August, 2020}

\renewcommand{\headeright}{}
\renewcommand{\undertitle}{}
\begin{document}
\maketitle



                  

\begin{abstract}


Distance metric learning is a branch of machine learning that aims to learn distances from the data, which enhances the performance of similarity-based algorithms.
This tutorial provides a theoretical background and foundations on this topic and a comprehensive experimental analysis of the most-known algorithms. We start by describing the distance metric learning problem and its main mathematical foundations, divided into three main blocks: convex analysis, matrix analysis and information theory. Then, we will describe a representative set of the most popular distance metric learning methods used in classification. All the algorithms studied in this paper will be evaluated with exhaustive testing in order to analyze their capabilities in standard classification problems, particularly considering dimensionality reduction and kernelization. The results,  verified by Bayesian statistical tests, highlight a set of outstanding algorithms. Finally, we will discuss several potential future prospects and challenges in this field. This tutorial will serve as a starting point in the domain of distance metric learning from both a theoretical and practical perspective.

\end{abstract}





\section{Introduction} \label{sec:intro}

The use of distances in machine learning has been present since its inception. Distances provide a similarity measure between the data, so that close data will be considered similar, while remote data will be considered dissimilar. One of the most popular examples of similarity-based learning is the well-known nearest neighbors rule for classification, where a new sample is labeled with the majority class within its nearest neighbors in the training set. This classifier was presented in 1967 by \citet{cover1967nearest}, even though this idea had already been mentioned in earlier publications \cite{sebestyen1962decision,nilsson1965learning}.

Algorithms in the style of the nearest neighbors classifier are among the main motivators of distance metric learning. These kinds of algorithms have usually used a standard distance, like the Euclidean distance, to measure the data similarity. However, a standard distance may ignore some important properties available in our dataset, and therefore the learning results might be non optimal. The search for a distance that brings similar data as close as possible, while moving non similar data away, can significantly increase the quality of these algorithms.

During the first decade of the 21st century some of the most well-known distance metric learning algorithms were developed, and perhaps the algorithm from \citet{lsi} is responsible for drawing attention to this concept for the first time. Since then, distance metric learning has established itself as a promising domain in machine learning, with applications in many fields, such as medicine \cite{ma2019nasopharyngeal,wei2020multi}, security \cite{li2020improving,luo2020transforming}, social media mining \cite{liu2017metric,liu2019multi}, speech recognition \cite{li2020automatic,bai2020speaker}, information retrieval \cite{lopez2019visual,hu2019semi}, recommender systems \cite{wu2020effective,li2020social} and many areas of computer vision, such as person re-identification \cite{Nguyen2019589,zhao2020similarity}, kinship verification \cite{Liang20191149,dornaika2020transfer} or image classification \cite{wang2019hybrid,wang2020deep}.

Although distance metric learning has proven to be an alternative to consider with small and medium datasets \cite{du2019detection}, one of its main limitations is the treatment of large data sets, both at the level of number of samples and at the level of number of attributes \cite{wells2020simple}. In recent years, several alternatives have been explored in order to develop algorithms for these areas \cite{nguyen2020scalable,liu2019escaping}.

Several surveys on distance metric learning have been proposed. Among the well-known surveys we can find the work of \citet{yang2006distance}, \citet{kulis2013metric}, \citet{bellet2013survey} and \citet{moutafis2017overview}. However, we must point out several `loose ends' present in these studies. On the one hand, they do not provide an in-depth study of the theory behind distance metric learning. Such a study would help to understand the motivation behind the mechanics that give rise to the various problems and tools in this field. On the other hand, previous studies do not carry out enough experimental analyses that evaluate the performance of the algorithms on sufficiently diverse datasets and circumstances.

In this paper we undergo a theoretical study of supervised distance metric learning, in which we show the mathematical foundations of distance metric learning and its algorithms. We analyze several distance metric learning algorithms for classification, from the problems and the objective functions they try to optimize, to the methods that lead to the solutions to these problems. Finally, we carry out several experiments involving all the algorithms analyzed in this study. In our paper, we want to set ourselves apart from previous publications by focusing on a deeper analysis of the main concepts of distance metric learning, trying to get to its basic ideas, as well as providing an experimental framework with the most popular metric learning algorithms. We will also discuss some opportunities for future work in this topic.

Regarding the theoretical background of distance metric learning, we have studied three mathematical fields that are closely related to this topic. The first one is convex analysis \cite{convexanalysis,convexoptimization}. Convex analysis is present in many distance metric learning algorithms, as these algorithms try to optimize convex functions over convex sets. Some interesting properties of convex sets, as well as how to deal with constrained convex problems, will be shown in this study. We will also see how the use of matrices is a fundamental part when modeling our problem. Matrix analysis \cite{matrix_analysis} will therefore be the second field studied. The third is information theory \cite{information_theory}, which is also used in some of the algorithms we will present. 

As explained before, our work focuses on supervised distance metric learning techniques. A large number of algorithms have been proposed over the years. These algorithms were developed with different purposes and based on different ideas, so that they could be classified into different groups. Thus, we can find algorithms whose main goal is dimensionality reduction \cite{lda,anmm,cunningham2015linear}, algorithms specifically oriented to improve distance based classifiers, such as the nearest neighbors classifier \cite{lmnn,nca}, or the nearest centroid classification \cite{ncmml}, and a few techniques that are based on information theory \cite{itml,dmlmj,mcml}. Some of these algorithms also allow kernel versions \cite{klmnn,kda,anmm,dmlmj}, that allow for the extension of distance metric learning to highly dimensional spaces. 

As can be seen in the experiments, we compare all the studied algorithms using up to 34 different datasets. In order to do so, we define different settings to explore their performance and capabilities when considering maximum dimension, centroid-based methods, different kernels and dimensionality reduction. Bayesian statistical tests are used to assess the significant differences among algorithms \cite{benavoli2017time}.

In summary, the aims of this tutorial are:

\begin{itemize}
    \item To know and understand the discipline of distance metric learning and its foundations.
    \item To gather and study the foundations of the main supervised distance metric learning algorithms.
    \item To provide experiments to evaluate the performance of the studied algorithms under several case studies and to analyze the results obtained. The code of the algorithms is available in the Python library \texttt{pydml} \cite{suarez2020pydml}.
\end{itemize}

Our paper is organized as follows. Section \ref{sec:dml} introduces the distance metric problem and its mathematical foundations, explains the family of distances we will work with and shows several examples and applications. Section \ref{sec:algs} discusses all the distance metric learning algorithms chosen for this tutorial. Section \ref{sec:experiments} describes the experiments done to evaluate the performance of the algorithms and shows the obtained results. Finally, Sections \ref{sec:future} and \ref{sec:conclusions} conclude the paper by indicating possible future lines of research in this area and summarizing the work done, respectively. We also provide a glossary of terms in Appendix \ref{app:glossary} with the acronyms used in the paper. So as to avoid overloading the sections with a large theoretical content a theoretical supplement is provided, in which Appendix \ref{app:math} presents in a rigorous and detailed manner the mathematical foundations of distance metric learning, structured in the three blocks discussed previously, and Appendix \ref{app:algs} provides a detailed explanation of the algorithms included in Section \ref{sec:algs}.

\section{Distance Metric Learning and Mathematical Foundations} \label{sec:dml}

In this section we will introduce the distance metric learning problem. To begin with, we will provide reasons to support the distance metric learning problem in Section \ref{ssec:dml_why}. Then, we will go over the concept of distance, with special emphasis on the Mahalanobis distances, which will allow us to model our problem (Section \ref{ssec:distances}). Once these concepts are defined, in Section \ref{ssec:dml_desc} we will describe the distance metric learning problem, explaining what it consists of, how it is handled in the supervised case and how it is modeled so that it can be treated computationally. To understand the basics of distance metric learning we provide a summary of the mathematical foundations underlying this field in Section \ref{ssec:intromaths}. These foundations support the theoretical description of this discipline as well as the numerous distance metric learning algorithms that will be discussed in Section \ref{sec:algs} and Appendix \ref{app:algs}. The mathematical background is then developed extensively in Appendix \ref{app:math}. Finally, we will finish with Section \ref{ssec:dmlinml} by detailing some of the uses of distance metric learning in machine learning.

\subsection{Distance Metric Learning: Why and What For?} \label{ssec:dml_why}

Similarity-based learning algorithms are among the earliest used in the field of machine learning. They are inspired by one of the most important components in many human cognitive processes: the ability to detect similarities between different objects. This ability has been adapted to machine learning by designing algorithms that learn from a dataset according to the similarities present between the data. These algorithms are present in most areas of machine learning, such as classification and regression, with the \ac{kNN} rule \cite{cover1967nearest}; in clustering, with the $k$-means algorithm \cite{macqueen1967some}; in recommender systems, with collaborative approaches based also on nearest neighbors \cite{dou2016survey}; in semi-supervised learning, to construct the graph representations \cite{ssl2}; in some kernel methods such as the radial basis functions \cite{hofmann2008kernel}, and many more.

To measure the similarity between data, it is necessary to introduce a distance, which allows us to establish a measure whereby it is possible to determine when a pair of samples is more similar than another pair of samples. However, there is an infinite number of distances we can work with, and not all of them will adapt properly to our data. Therefore, the choice of an adequate distance is a crucial element in this type of algorithm.

Distance metric learning arises to meet this need, providing algorithms that are capable of searching for distances that are able to capture features or relationships hidden in our data, which possibly a standard distance, like the Euclidean distance, would not have been able to discover. From another perspective, distance metric learning can also be seen as the missing training step in many similarity-based algorithms, such as the \emph{lazy} nearest-neighbor approaches. The combination of both the distance learning algorithms and the distance-based learners allows us to build more complete learning algorithms with greater capabilities to extract information of interest from our data.

Choosing an appropriate distance learned from the data has proven to be able to greatly improve the results of distance-based algorithms in many of the areas mentioned above. In addition to its potential when adhering to these learners, a good distance allows data to be transformed to facilitate their analysis, with mechanisms such as dimensionality reduction or axes selection, as we will discuss later.

\subsection{Mahalanobis Distances} \label{ssec:distances}

We will start by reviewing the concept of distance and some of its properties.

\begin{definition}
  Let $X$ be a non-empty set. A \emph{distance} or \emph{metric} over $X$ is a map $d \colon X \times X \to \R$ that satisfies the following properties:
  \begin{enumerate}
    \item Coincidence: $d(x,y) = 0 \iff x = y$, for every $x, y \in X$.
    \item Symmetry: $d(x,y) = d(y,x)$, for every $x,y \in X$.
    \item Triangle inequality: $d(x,z) \le d(x,y) + d(y,z)$, for every $x,y,z \in X$.
  \end{enumerate}
  The ordered pair $(X,d)$ is called a \emph{metric space}.
\end{definition}

The coincidence property stated above will not be of importance to us. That is why we will also consider mappings known as \emph{pseudodistances}, which only require that $d(x,x) = 0$, instead of the coincidence property. In fact, pseudodistances are strongly related with dimensionality reduction, which is an important application of distance metric learning. From now on, when we talk about distances, we will be considering proper distances as well as pseudodistances.

\begin{remark}
  As an immediate consequence of the definition, we have the following additional properties of distances:
  \begin{enumerate}
    \setcounter{enumi}{3}
    \item Non negativity: $d(x,y) \ge 0$ for every $x,y \in X$.
    \item Reverse triangle inequality: $|d(x,y) - d(y,z)| \le d(x,z)$ for every $x,y,z \in X$.
    \item Generalized triangle inequality: $d(x_1,x_n) \le \sum_{i=1}^{n-1} d(x_i,x_{i+1})$ for $x_1,\dots,x_n \in X$.
  \end{enumerate}
\end{remark}

When we work in the $d$-dimensional Euclidean space, a family of distances become very useful in the computing field. These distances are parameterized by positive semidefinite matrices and are known as \emph{Mahalanobis distances}. In what follows, we will refer to $\mathcal{M}_{d'\times d}(\R)$ (resp. $\mathcal{M}_{d}(\R)$) as the set of matrices of dimension $d' \times d$ (resp. square matrices of dimension $d$), and to $S_d(\mathbb{R})^+_0$ as the set of positive semidefinite matrices of dimension $d$.

\begin{definition}
  Let $d \in \N$ and $M \in S_d(\R)^+_0$. The \emph{Mahalanobis distance} corresponding to the matrix $M$ is the map $d_M \colon \R^d \times \R^d \to \R$ given by
  \[ d_M(x,y) = \sqrt{(x-y)^TM(x-y)}, \quad x,y \in \R^d. \]
\end{definition}

Mahalanobis distances come from the (semi-)dot products in $\R^d$ defined by the positive semidefinite matrix $M$. When $M$ is full-rank, Mahalanobis distances are proper distances. Otherwise, they are pseudodistances. Note that the Euclidean usual distance is a particular example of a Mahalanobis distance, when $M$ is the identity matrix $I$. Mahalanobis distances have additional properties specific to distances over normed spaces.

\begin{enumerate}
  \setcounter{enumi}{6}
  \item Homogeneousness: $d(ax,ay) = |a|d(x,y)$, for $a \in \R$, and $x, y \in \R^d$.
  \item Translation invariance: $d(x,y) = d(x+z,y+z)$, for $x, y, z \in \R^d$.
\end{enumerate}

Sometimes the term ``Mahalanobis distance'' is used to describe the squared distances of the form $d_M^2(x,y) = (x-y)^TM(x-y)$. In the area of computing, it is much more efficient to work with $d_M^2$ rather than with $d_M$, as this avoids the calculation of square roots. Although $d_M^2$ is not really a distance, it keeps the most useful properties of $d_M$ from the distance metric learning perspective, as we will see, such as the greater or lesser closeness between different pairs of points. That is why the use of the term ``Mahalanobis distance'' for both $d_M$ and $d_M^2$ is quite widespread.

To end this section, we return to the issue of dimensionality reduction that we mentioned when introducing the concept of pseudodistance. When we work with a pseudodistance $\sigma$ over a set $X$, it is possible to define an equivalence relationship given by $x \sim y$ if and only if $\sigma(x,y) = 0$, for each $x, y \in X$. Using this relationship we can consider the quotient space $X/_\sim$, and the map $\hat{\sigma}\colon X/_\sim \times X/_\sim \to \R$ given by $\hat{\sigma}([x],[y]) = \sigma(x,y)$, for each $[x],[y] \in X/_\sim$. This map is well defined and is a distance over the quotient space. When $\sigma$ is a Mahalanobis distance over $\R^d$, with rank $d' < d$ (we define the rank of a Mahalanobis distance as the rank of the associated positive semidefinite matrix matrix), then the previous quotient space becomes a vector space isomorphic to $\R^{d'}$, and the distance $\hat{\sigma}$ is a full-rank Mahalanobis distance over $\R^{d'}$. That is why, when we have a Mahalanobis pseudodistance on $\R^d$, we can view this as a proper Mahalanobis distance over a lower dimensional space, hence we have obtained a dimensionality reduction.

\subsection{Description of Distance Metric Learning} \label{ssec:dml_desc}

\emph{\ac{DML}} is a machine learning discipline with the purpose of learning distances from a dataset. In its most general version, a dataset $\mathcal{X} = \{x_1,\dots,x_N\}$ is available, on which certain similarity measures between different pairs or triplets of data are collected. These similarities are determined by the sets
\begin{align*}
  S &= \{(x_i,x_j) \in \mathcal{X}\times\mathcal{X} \colon x_i \text{ and } x_j \text{ are similar.} \}, \\
  D &= \{(x_i,x_j) \in \mathcal{X}\times\mathcal{X} \colon x_i \text{ and } x_j \text{ are not similar.} \}, \\
  R &= \{(x_i,x_j,x_l) \in \mathcal{X}\times\mathcal{X}\times\mathcal{X} \colon x_i \text{ is more similar to } x_j \text{ than to } x_l. \}.
\end{align*}

With these data and similarity constraints, the problem to be solved consists in finding, after establishing a family of distances $\mathcal{D}$, those distances that best adapt to the criteria specified by the similarity constraints. To do this, a certain loss function $\ell$ is set, and the sought-after distances will be those that solve the optimization problem
\[ \min_{d \in \mathcal{D}} \ell(d,S,D,R) .\]

When we focus on supervised learning, in addition to dataset $\mathcal{X}$ we have a list of labels $y_1,\dots,y_N$ corresponding to each sample in $\mathcal{X}$. The general formulation of the \ac{DML} problem is easily adapted to this new situation, just by considering the sets $S$ and $D$ as sets of pairs of same-class samples and different-class samples, respectively. Two main approaches are followed to establish these sets. The \emph{global \ac{DML}} approach considers the sets $S$ and $D$ to be
\begin{align*}
  S &= \{(x_i,x_j) \in \mathcal{X}\times\mathcal{X} \colon y_i = y_j\}, \\
  D &= \{(x_i,x_j) \in \mathcal{X}\times\mathcal{X} \colon y_i \ne y_j\}. 
\end{align*}
On the other hand, the \emph{local \ac{DML}} approach replaces the previous definition of $S$ with
\begin{equation*}
    S = \{(x_i, x_j) \in \mathcal{X}\times\mathcal{X} \colon y_i = y_j \text{ and } x_j \in \mathcal{U}(x_i)\},
\end{equation*}
where $\mathcal{U}(x_i)$ denotes a \emph{neighborhood} of $x_i$, that is, a set of points that should be close to $x_i$, which has to be established before the learning process by using some sort of prior information, or a standard similarity measure. The set $D$ remains the same in the local approach, since different-class samples are not meant to be similar in a supervised learning setting. In addition, the set $R$ may be also available in both approaches by defining triplets $(x_i,x_j,x_l)$ where in general $y_i = y_j \ne y_l$, and they verify certain conditions imposed on the distance between $x_i$ and $x_j$, as opposed to the distance between $x_i$ and $x_l$. This is the case, for example, for impostors in the \acs{LMNN} algorithm (see Section \ref{dml:lmnn} and \cite{lmnn}). In any case, labels have all the necessary information in the field of supervised \ac{DML}. From now on we will focus on this kind of problem.

Furthermore, focusing on the nature of the dataset, practically all of the \ac{DML} theory is developed for numerical data. Although it is possible to define relevant distances for non-numerical attributes \cite{ahmad2007method, blumenthal2020exact} and although some learning processes can be performed with them \cite{norouzi2012hamming, ma2020discriminative}, the richness of the distances available to numerical features, their ability to be parameterized computationally, and the fact that nominal data can be converted to numerical variables or ordinal variables, with appropriate encoding \cite{zheng2018feature}, cause the relevant distances in this discipline to be those defined for numerical data. For this reason, from now on, we will focus on supervised learning problems with numerical datasets.

We will suppose then that $\mathcal{X} \subset \R^d$. As we saw in the previous section, for finite-dimensional vector spaces we have the family of Mahalanobis distances, $\mathcal{D} = \{d_M \colon M \in S_d(\R)^+_0\}$. With this family, we have at our disposal all the distances associated with dot products in $\R^d$ (and in lower dimensions). In addition, this family is determined by the set of positive semidefinite matrices, and therefore, we can use these matrices, which we will call \emph{metric matrices}, to parameterize distances. In this way, the general problem adapted to supervised learning with Mahalanobis distances can be rewritten as
\begin{equation*} \label{eq:metric_learning_eq}
        \min_{M \in S_{d}(\mathbb{R})^+_0} \ell(d_M,(x_1,y_1),\dots,(x_N,y_N)) .
\end{equation*}

However, this is not the only way to parameterize this type of problem. We know, from the \emph{matrix decomposition theorem} discussed in Section \ref{ssec:intromaths} and Theorem \ref{thm:decomposition_llt}, that if $M \in S_d(\R)^+_0$, then there is a matrix $L \in \mathcal{M}_d(\R)$ so that $M = L^TL$, and this matrix is unique except for an isometry. So, then we get
\[d_M^2(x,y) = (x-y)^TM(x-y) =(x-y)^TL^TL(x-y) = (L(x-y))^T(L(x-y)) = \|L(x-y)\|_2^2. \]

Therefore, we can also parameterize Mahalanobis distances through any matrix, although in this case the interpretation is different. When we learn distances through positive semidefinite matrices we are learning a new metric over $\R^d$. When we learn distances with the previous $L$ matrices, we are learning a linear map (given by $x \mapsto Lx$) that transforms the data in the space, and the corresponding distance is the usual Euclidean distance after projecting the data onto the new space using the linear map. Both approaches are equivalent thanks to the matrix decomposition theorem (Theorem \ref{thm:decomposition_llt}).

In relation to dimensionality, it is important to note that, when the learned metric $M$ is not full-rank, we are actually learning a distance over a space of lower dimension (as we mentioned in the previous section), which allows us to reduce the dimensionality of our dataset. The same occurs when we learn linear maps that are not full-rank. We can extend this case and opt to learn directly linear maps defined by $L \in \mathcal{M}_{d' \times d}(\R)$, with $d' < d$. In this way, we ensure that data are directly projected into a space of dimension no greater than $d'$.

Both learning the metric matrix $M$ and learning the linear transformation $L$, are useful approaches to model \ac{DML} problems, each one with its advantages and disadvantages. For example, parameterizations via $M$ usually lead to convex optimization problems. In contrast, convexity in problems parameterized by $L$ is not so easy to achieve. On the other hand, parameterizations using $L$ make it possible to learn projections directly onto lower dimensional spaces, while dimensional constraints for problems parameterized by $M$ are not so easy to achieve. Let us examine these differences with simple examples.

\begin{example} \label{ex:M_convex}
  Many of the functions we will want to optimize will depend on the squared distance defined by the metric $M$ or by the linear transformation $L$, that is, either they will have terms of the form $\|v\|_M^2 = v^TMv$, or of the form $\|v\|_L^2 = \|Lv\|_2^2$. Both the maps $M \mapsto \|v\|_M^2$ and $L \mapsto \|v\|_L^2$ are convex (the first is actually affine). However, if we want to substract terms in this way, we lose convexity in $L$, because the mapping $L \mapsto -\|v\|_L^2$ is no longer convex. In contrast, the mapping $M \mapsto -\|v\|_M^2$ is still affine and, therefore, convex.
\end{example}

\begin{example} \label{ex:rank_not_convex}
  Rank constraints are not convex, and therefore we may not dispose of a projection onto the set corresponding to those constraints, unless we learn the mapping (parameterized by $L$) directly to the space with the desired dimension, as explained before. For example, if we consider the set $C = \{M \in S_2(\R)^+_0 \colon r(A) \le 1 \}$, we get $A = \begin{pmatrix} 2 & 0 \\ 0 & 0 \end{pmatrix} \in C$ and $B = \begin{pmatrix} 0 & 0 \\ 0 & 2 \end{pmatrix} \in C$. However, $(1-\lambda)A+\lambda B = I \notin C$, for $\lambda=1/2$.
\end{example}

\subsection{Mathematical Foundations of Distance Metric Learning} \label{ssec:intromaths}

There are three main mathematical areas that support \ac{DML}: convex analysis, matrix analysis and information theory. The first provides the necessary tools so that many of the algorithms can address their optimization problems. Thanks to the second we can parameterize \ac{DML}, and in this way we can compute the different problems. It also provides us with some interesting results when it comes to solving certain problems related to dimensionality reduction. Finally, the third field provides us with concepts and tools that are very useful for designing algorithms that use probability distributions associated with the data.

\subsubsection{Convex Analysis}

Let us begin with convex analysis. One of the properties of convex sets that makes convex analysis of great interest in \ac{DML} is known as the \emph{convex projection theorem} (Theorem \ref{thm:convex_projection}), which ensures that for any non-empty convex closed set $K$ in $\mathbb{R}^d$ and for every point $x \in \mathbb{R}^d$ there is a single point $x_0 \in K$ for which the distance from $x$ to $K$ is the same to the distance from $x$ to $x_0$. That is, the distance from $x$ to $K$ is materialized in the point $x_0$, which is called the \emph{projection of $x$ to $K$}.

The existence of a projection mapping onto any closed and convex set in $\mathbb{R}^d$ is fundamental when optimizing convex functions with convex constraints, which are frequent, in particular, in many \ac{DML} algorithms. Let us first discuss optimization mechanisms when working with unconstrained differentiable functions, which, although they do not strictly take part in convex analysis, are also present in some \ac{DML} algorithms and are the basis for convex optimization mechanisms. In these cases, the most popular techniques are the well-known \emph{gradient descent methods}, which are iterative methods. The basic idea of gradient descent methods is to move in the direction of the gradient of the objective function in order to optimize it. We show in \ref{sssec:opt_methods} that, indeed, small displacements in the gradient direction guarantee the improvement of the objective function, proving the effectiveness of these methods.

Returning to the constrained case, we can see that gradient descent methods are no longer valid, since the displacement in the gradient direction can no longer fulfill the constraints. However, we will show that, if after the gradient step we project the obtained point onto the convex set determined by the constraints, the combination of both movements contributes to improving the value of the objective function as long as the initial gradient step is small enough. This extension of gradient descent to the constrained convex case is known as the \emph{projected gradient method}. There are also other approaches, such as the penalty methods, which allow these problems to be handled by transforming the constrained objective function into a new unconstrained objective function in which the violations of the previous constraints are converted into penalties that worsen the value of the new objective function \cite{yeniay2005penalty}. They will not usually, however, be preferred in the matrix problems we will be dealing with, as they may be computationally expensive and difficult to adapt \cite{yang2017richer}.

Finally, we must highlight other tools of interest for the optimization problems to be studied. Firstly, when working with convex problems with multiple constraints, the projections on each individual constraint are often known, but the projection onto the set determined by all the constraints is not. With the method known as the \emph{iterated projections method} (see Appendix \ref{sssec:opt_methods}) we can approach this projection by subsequently projecting onto each of the individual constraints until convergence (which is guaranteed) is obtained. Lastly, convex functions that are not differentiable everywhere can still be optimized following the approaches discussed here, as they admit sub-gradients at every point. Sub-gradient descent methods (see Appendix \ref{sssec:opt_methods}) can work in the same way as gradient descent methods and can therefore be applied with convex functions that may not be differentiable at some points.

\subsubsection{Matrix Analysis}

As we have already seen, matrices are a key element in \ac{DML}. There are several results that are essential for the development of the \ac{DML} theory and its algorithms. The first of these is the \emph{matrix decomposition theorem} (Theorem \ref{thm:decomposition_llt}), which was already mentioned in Section \ref{ssec:dml_desc}. This theorem states that for any positive matrix $M \in S_d(\R)^+_0$ there is a matrix $L \in \mathcal{M}_d(\R)$ so that $M = L^TL$ and $L$ is unique except for an isometry. This result allows us to approach \ac{DML} from the two perspectives (learning the metric $M$ or learning the linear map $L$) already discussed in Section \ref{ssec:dml_desc}.

An important aspect when designing \ac{DML} algorithms is the geometric manipulation of the matrices (for learning both $M$ and $L$). Observe that to be able to talk about the convex analysis concepts discussed previously over the set of matrices, we first need to establish an inner product over them. The \emph{Frobenius inner product} allows us to identify matrices as vectors where we add the matrix rows one after the other, and then compute the usual vectorial inner product with these vectors. With the Frobenius product we convert the matrices set in a Hilbert space, and therefore can apply the convex analysis theory studied in the previous section.

Staying on this subject, we have to highlight a case study of particular interest. We will see many situations where we want to optimize a convex function defined on a matrix space, with the restriction that the variable is positive semidefinite. These optimization problems are convex and are usually called \emph{semidefinite programming} problems. We can optimize these objective functions using the projected gradient descent method. The \emph{semidefinite projection theorem} (Theorem \ref{thm:psd_projection}) states that we can compute the projection of a matrix onto the positive semidefinite cone by performing an eigenvalue decomposition, nullifying the negative eigenvalues and recomposing the matrix with the new eigenvalues. Therefore, we know how to project onto the constraint set and consequently we can apply the projected gradient method.

Finally, we will see that certain algorithms, especially those associated with dimensionality reduction, use optimization problems with similar structures. These problems involve one or more symmetric matrices and the objective function is obtained as a trace after performing certain operations with these matrices. This is, for instance, the case of the objective function of the well-known \emph{principal components analysis}, which can be written as
\begin{equation*}
  \begin{split}
      \max_{L \in \mathcal{M}_{d'\times d}(\R)} &\quad \tr\left(LAL^T\right)  \\
      \text{s.t.: } &\quad LL^T = I,
   \end{split}
\end{equation*}
where $A$ is a symmetric matrix of dimension $d$. These problems have the property that they can be optimized without using gradient methods, since an optimum can be built by taking the eigenvectors associated with the largest eigenvalues of the matrices involved in the problem (in the case described here, the eigenvectors of $A$). In Appendix \ref{ssec:matrix_opt} we present these problems in more detail and show how their solution is obtained.

\subsubsection{Information Theory}

Information theory is a branch of mathematics and computer theory with the purpose of establishing a rigorous measure to quantify the information and disorder found in a communication message. It has been applied in most science and engineering fields. In \ac{DML}, information theory is used in several algorithms to measure the closeness between probability distributions. Then, these algorithms try to find a distance metric for which these probability distributions are as close as possible or as far as possible, depending on what distributions are defined. The measures used in this area, unlike distances, only require the properties of non-negativity and coincidence, and are called \emph{divergences}. We will use two different divergences throughout this study:

\begin{itemize}
  \item The \emph{relative entropy} or the \emph{Kullback-Leibler divergence}, defined for probability distributions $p$ and $q$, and $X$ the random variable corresponding to $p$ as
    \[ \kl(p\|q) = \mathbb{E}_p\left[\log\frac{p(X)}{q(X)}\right]. \]
  \item The \emph{Jeffrey divergence} or the \emph{symmetric relative entropy}, defined for $p$, $q$ and $X$ in the same conditions as above, as
    \[\jf(p\|q) =  \kl(p\|q) + \kl(q\|p).\]
\end{itemize}

The key fact that makes divergences very useful in \ac{DML} is that, when the distributions involved are multivariate gaussian, these divergences can be expressed in terms of matrix divergences, which give rise to problems that can be dealt with quite effectively using the tools described in this section. In Appendix \ref{ssec:information_theory} we present the matrix expressions obtained for the Kullback-Leibler and the Jeffrey divergences for the most remarkable cases.

\subsection{Use Cases in Machine Learning} \label{ssec:dmlinml}

This section describes some of the most prominent uses of \ac{DML} in machine learning, illustrated with several examples.
\begin{itemize}
  \item \textbf{Improve the performance of distance-based classifiers.} This is one of the main purposes of \ac{DML}. Using such learning, a distance that fits well with the dataset and the classifier can be found, improving the performance of the classifier \cite{lmnn,nca}. An example is shown in Figure \ref{fig:improve_knn}.

  \begin{figure}[htbp]
    \centering
    \makebox[\textwidth][c]{\includegraphics[width=21cm]{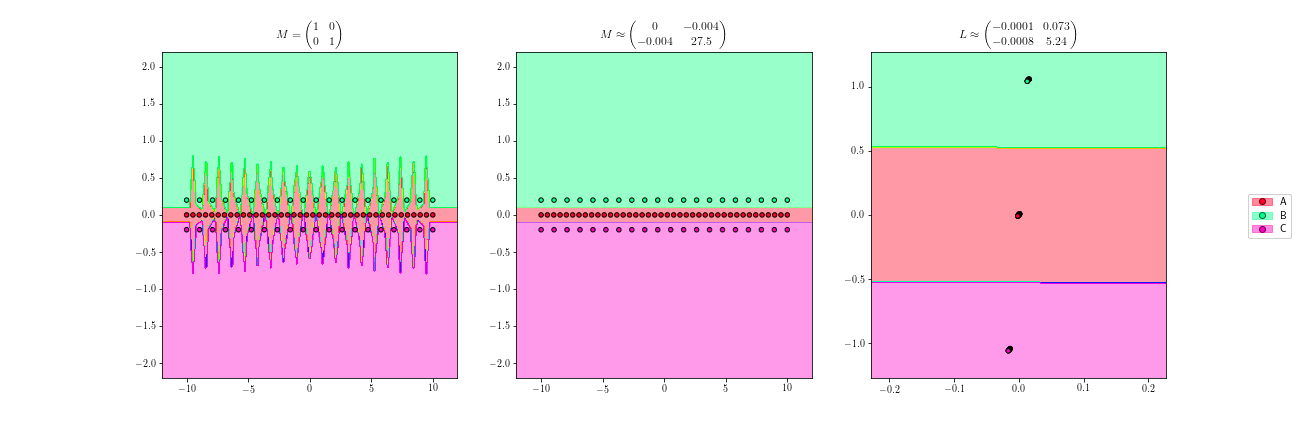}}%
    \caption{Suppose we have a dataset in the plane, where data can belong to three different classes, whose regions are defined by parallel lines. Suppose that we want to classify new samples using the one nearest neighbor classifier. If we use Euclidean distance, we would obtain the classification regions shown in the image on the left, because there is a greater separation between each sample in class B and class C than there is between the regions. However, if we learn an adequate distance and try to classify with the nearest neighbor classifier again, we obtain much more effective classification regions, as shown in the center image. Finally, as we have seen, learning a metric is equivalent to learning a linear map and to use Euclidean distance in the transformed space. This is shown in the right figure. We can also observe that data are being projected, except for precision errors, onto a line, thus we are also reducing the dimensionality of the dataset. } \label{fig:improve_knn}
  \end{figure}

  \item \textbf{Dimensionality reduction.} As we have already commented, learning a low-rank metric implies a dimensionality reduction on the dataset we are working with. This dimensionality reduction provides numerous advantages, such as a reduction in the computational cost, both in space and time, of the algorithms that will be used later, or the removal of the possible noise introduced when picking up the data. In addition, some distance-based classifiers are exposed to a problem called \emph{curse of dimensionality} (see, for example, \cite{understandingml}, sec.~19.2.2). By reducing the dimension of the dataset, this problem also becomes less serious. Finally, if deemed necessary, projections onto dimension 1, 2 and 3 would allow us to obtain visual representations of our data, as shown in Figure \ref{fig:reduc_dim}. In general, many real-world problems arise with a high dimensionality, and need a dimensionality reduction to be handled properly. 

  \begin{figure}[htbp]
    \centering
    \includegraphics[width=0.75\textwidth]{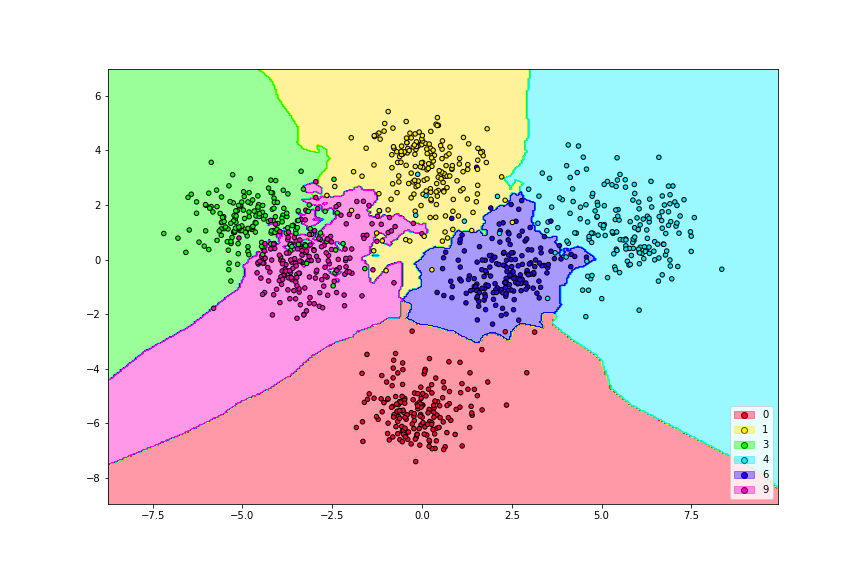}
    \caption{'Digits' dataset consists of 1797 examples. Each of them consists of a vector with 64 attributes, representing intensity values on an 8x8 image. The examples belong to 10 different classes, each of them representing the numbers from 0 to 9. By learning an appropriate transformation we are able to project most of the classes on the plane, so that we can clearly perceive the differentiated regions associated with each of the classes.} \label{fig:reduc_dim}
  \end{figure}

  \item \textbf{Axes selection and data rearrangement.} Closely related to dimensionality reduction, this application is a result of algorithms that learn transformations which allow the coordinate axes to be moved (or selected according to the dimension), so that in the new coordinate system the vectors concentrate certain measures of information on their first components \cite{kokiopoulou2011trace}. An example is shown in Figure \ref{fig:axes_move}.

  \begin{figure}[htbp]
    \centering
    \makebox[\textwidth][c]{\includegraphics[width=20cm]{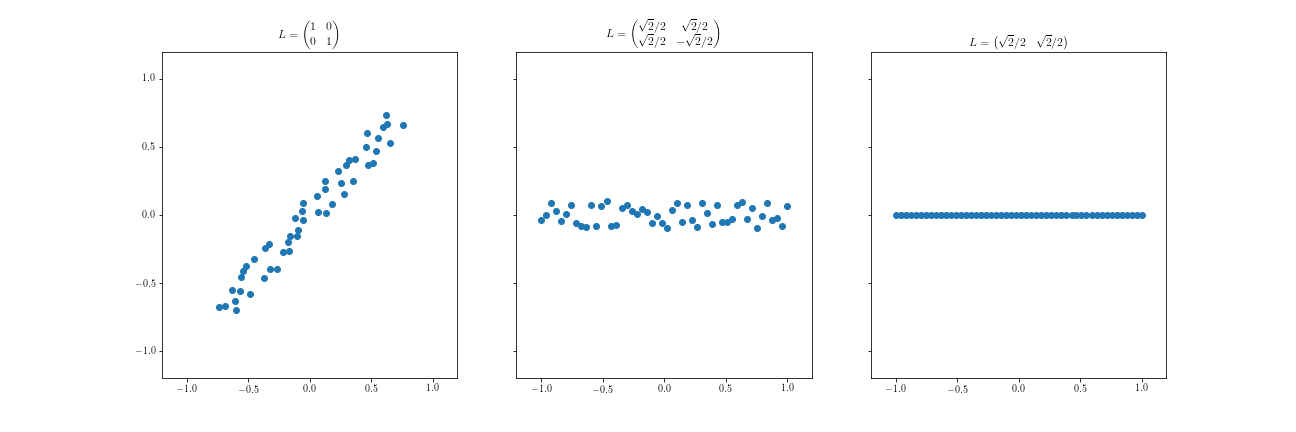}}
    \caption{The dataset in the left figure seems to concentrate most of its information on the diagonal line that links the lower left and upper right corners. By learning an appropriate transformation, we can get that direction to fall on the horizontal axis, as shown in the center image. As a result, the first coordinate of the vectors in this new basis concentrates a large part of the variability of the vector. In addition, it seems reasonable to think that the values introduced by the vertical coordinate might be due to noise, and so, we can even just keep the first component, as shown in the right image.} \label{fig:axes_move}
  \end{figure}

  \item \textbf{Improve the performance of clustering algorithms.} Many of the clustering algorithms use a distance to measure the closeness between data, and thus establish the clusters so that data in the same cluster are considered close for that distance. Sometimes, although we do not know the ideal groupings of the data or the number of clusters to establish, we can know that certain pairs of points must be in the same cluster and that other specific pairs must be in different clusters \cite{lsi}. This happens in numerous problems, for example, when clustering web documents \cite{aggarwal2012text}. These documents have a lot of additional information, such as links between documents, which can be included as similarity constraints. Many clustering algorithms are particularly sensitive to the distance used, although many also depend heavily on the parameters with which they are initialized \cite{bradley1998refining,probst2019tunability}. It is therefore important to seek a balance or an appropriate aggregation between these two components. In any case, the parameter initialization is beyond the scope of this paper.

  \item \textbf{Semi-supervised learning.} Semi-supervised learning is a learning model in which there is one set of labeled data and another set (generally much larger) of unlabeled data. Both datasets are intended to learn a model that allows new data to be labeled. Semi-supervised learning arises from the fact that in many situations collecting unlabeled data is relatively straightforward, but assigning labels can require a supervisor to assign them manually, which may not be feasible. In contrast, when a lot of unlabeled data is used along with a small amount of labeled data, it is possible to improve learning outcomes considerably, as exemplified in Figure \ref{fig:ssl}. Many of these techniques consist of constructing a graph with weighted edges from the data, where the value of the edges depends on the distances between the data. From this graph we try to infer the labels of the whole dataset, using different propagation algorithms \cite{ssl2}. In the construction of the graph, the choice of a suitable distance is important, thus \ac{DML} comes into play \cite{idml}.

  \begin{figure}[htbp]
    \centering
    \includegraphics[width=\textwidth]{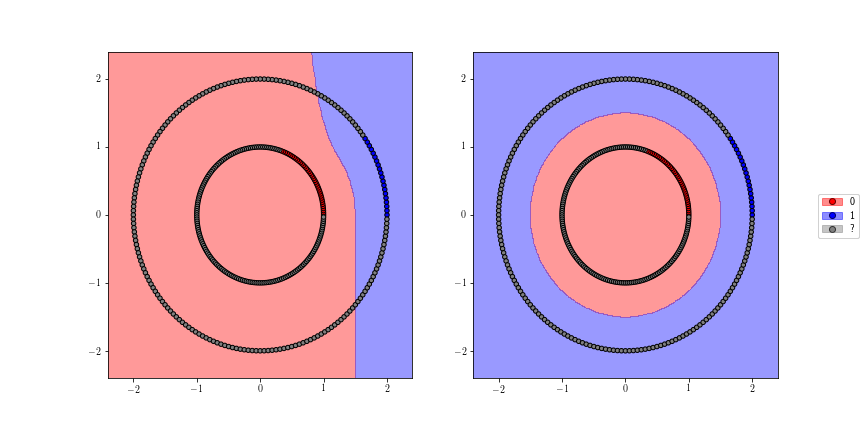}
    \caption{Learning with only supervised information (left) versus learning with all unsupervised information (right).} \label{fig:ssl}
  \end{figure}
\end{itemize}

From the applications we have seen, we can conclude that \ac{DML} can be viewed as a preprocessing step for many distance-based learning algorithms. The algorithms analyzed in our work focus on the first three applications of the above enumeration. It should be noted that, although the fields above are those where \ac{DML} has traditionally been used, today, new prospects and challenges for \ac{DML} are being considered. They will be discussed in Section \ref{sec:future}.

\section{Algorithms for Distance Metric Learning} \label{sec:algs}

This section introduces some of the most popular techniques currently being used in supervised \ac{DML}. Due to space issues, the section will give a brief description of each of the algorithms, while a detailed description can be found in Appendix \ref{app:algs}, where the problems the algorithms try to optimize are analyzed, together with their mathematical foundations and the techniques used to solve them.

Table \ref{tbl:algs_intro} shows the algorithms studied throughout this work, including name, references and a short description. These algorithms will be empirically analyzed in the next section. This study is not intended to be exhaustive and therefore only some of the most popular algorithms have been selected for the theoretical study and the subsequent experimental analysis.

\begin{table}
  \resizebox{\linewidth}{!}{%
    \begin{tabular}{p{2cm}p{4cm}p{2cm}p{8cm}}
    \toprule
      Name & References & Appendix section & Description \\
    \midrule
      \acs{PCA} & \citet{pca} & \ref{desc:pca} & A dimensionality reduction technique that obtains directions maximizing variance. Although not supervised, it is important to allow dimensionality reduction in other algorithms that are not able to do so on their own. \\

      \acs{LDA} & \citet{lda} & \ref{desc:lda} & A dimensionality reduction technique that obtains direction maximizing a ratio involving \emph{between-class} variances and \emph{within-class} variances. \\

      \acs{ANMM} & \citet{anmm} & \ref{desc:anmm} & A dimensionality reduction technique that aims at optimizing an average neighborhood margin between same-class neighbors and different-class neighbors, for each of the samples. \\

      \acs{LMNN} & \citet{lmnn} & \ref{desc:lmnn} & An algorithm aimed at improving the accuracy of the $k$-neighbors classifier. It tries to optimize a two-term error function that penalizes, on the one hand, large distance between each sample and its \emph{target neighbors}, and on the other hand, small distances between different-class samples.\\

      \acs{NCA} & \citet{nca} & \ref{desc:nca} & An algorithm aimed at improving the accuracy of the $k$-neighbors classifier, by optimizing the expected \emph{leave-one-out} accuracy for the nearest neighbor classification. \\

      \acs{NCMML} & \citet{ncmml} & \ref{desc:ncmml} & An algorithm aimed at improving the accuracy of the \emph{nearest class mean} classifier, by optimizing a log-likelihood for the labeled data in the training set. \\

      \acs{NCMC} & \citet{ncmml} & \ref{desc:ncmc} & A generalization of the \acs{NCMML} algorithm aimed at improving nearest centroids classifiers that allow multiple centroids per class. \\

      \acs{ITML} & \citet{itml} & \ref{desc:itml} & An information theory based technique that aims at minimizing the Kullback-Leibler divergence with respect to an initial gaussian distribution, but while keeping certain similarity constraints between data. \\

      \acs{DMLMJ} & \citet{dmlmj} & \ref{desc:dmlmj} & An information theory based technique that aims at maximizing the Jeffrey divergence between two distributions, associated to similar and dissimilar points, respectively. \\

      \acs{MCML} & \citet{mcml} & \ref{desc:mcml} & An information theory based technique that tries to collapse same-class points in a single point, as far as possible from the other classes collapsing points. \\

      \acs{LSI} & \citet{lsi} & \ref{desc:lsi} & A \ac{DML} algorithm that globally minimizes the distances between same-class points, while fulfilling minimum-distance constraints for different-class points. \\

      \acs{DML-eig} & \citet{dmleig} & \ref{desc:dmleig} & A \ac{DML} algorithm similar to \acs{LSI} that offers a different resolution method based on eigenvalue optimization. \\

      \acs{LDML} & \citet{ldml} & \ref{desc:ldml} & A probabilistic approach for \ac{DML} based on the logistic function. \\

      \acs{KLMNN} & \citet{lmnn}; \citet{klmnn} & \ref{desc:klmnn} & The kernel version of \acs{LMNN}. \\

      \acs{KANMM} & \citet{anmm} & \ref{desc:kanmm} & The kernel version of \acs{ANMM}. \\

      \acs{KDMLMJ} & \citet{dmlmj} & \ref{desc:kdmlmj} & The kernel version of \acs{DMLMJ}. \\

      \acs{KDA} & \citet{kda} & \ref{desc:kda} & The kernel version of \acs{LDA}. \\

      \bottomrule
    \end{tabular}
    }
  \caption{Description of the \ac{DML} algorithms analyzed in this study. \label{tbl:algs_intro}}
\end{table}

We will now provide a brief introduction to these algorithms. According to the main purpose of each algorithm, we can group them into different categories: dimensionality reduction techniques (Section \ref{sssec:dim_red}), algorithms directed at improving the nearest neighbors classifiers (Section \ref{sssec:algs_knn}), algorithms directed at improving the nearest centroid classifiers (Section \ref{sssec:algs_ncm}), or algorithms based on information theory (Section \ref{sssec:algs_information_theory}). These categories are not necessarily exclusive, but we have considered each of the algorithms in the category associated with their dominant purpose. We also introduce, in Section \ref{sssec:other_dmls} several algorithms with less specific goals, and finally, in Section \ref{sssec:kernel_dml}, the kernel versions for some of the algorithms studied.

We will explain the problems that each of these techniques try to solve. For a more detailed description of each algorithm, the reader can refer to the corresponding section of the appendix as shown in Table \ref{tbl:algs_intro}.

\subsection{Dimensionality Reduction Techniques} \label{sssec:dim_red}

Dimensionality reduction techniques try to learn a distance by searching a linear transformation from the dataset space to a lower dimensional Euclidean space. We will describe the algorithms \acs{PCA} \cite{pca}, \acs{LDA} \cite{lda} and  \acs{ANMM} \cite{anmm}.

\subsubsection{\acf{PCA}}

\acs{PCA} \cite{pca} is one of the most popular dimensionality reduction techniques in unsupervised \ac{DML}. Although \acs{PCA} is an unsupervised learning algorithm, it is necessary to talk about it in our paper, firstly because of its great relevance, and more particularly, because when a supervised \ac{DML} algorithm does not allow a dimensionality reduction, \acs{PCA} can be first applied to the data in order to be able to use the algorithm later in the lower dimensional space.

The purpose of \acs{PCA} is to learn a linear transformation from the original space $\mathbb{R}^d$ to a lower dimensional space $\mathbb{R}^{d'}$ for which the loss when recomposing the data in the original space is minimized. This has been proven to be equivalent to iteratively finding orthogonal directions for which the projected variance of the dataset is maximized. The linear transformation is then the projection onto these directions. The optimization problem can be formulated as

\begin{equation*}
    \max_{\substack{L \in \mathcal{M}_{d'\times d}(\R) \\LL^T = I}} \quad \tr\left(L \Sigma L^T\right),
\end{equation*}

where $\Sigma$ is, except for a constant, the covariance matrix of $\mathcal{X}$, and $\tr$ is the trace operator. The solution to this problem can be obtained by taking as the rows of $L$ the eigenvectors of $\Sigma$ associated with its largest eigenvalues.

\subsubsection{\acf{LDA}}

\acs{LDA} \cite{lda} is a classical \ac{DML} technique with the purpose of learning a projection matrix that maximizes the separation between classes in the projected space using \emph{within-class} and \emph{between-class} variances. It follows a scheme similar to the one proposed by \acs{PCA}, but in this case it takes the supervised information provided by the labels into account.

The optimization problem of \acs{LDA} is formulated as
\begin{equation*}
    \max_{\substack{L \in \mathcal{M}_{d'\times d}(\R) }} \quad \tr\left((LS_wL^T)^{-1}(L S_b L^T)\right),
\end{equation*}
where $S_b$ and $S_w$ are, respectively, the between-class and within-class \emph{scatter matrices}, which are defined as
\begin{align*}
   S_b &= \sum_{c \in \mathcal{C}} N_c(\mu_c - \mu)(\mu_c - \mu)^T, \\
   S_w &= \sum_{c \in \mathcal{C}} \sum_{i \in \mathcal{C}_c}(x_i- \mu_c)(x_i - \mu_c)^T,
\end{align*}
where $\mathcal{C}$ is the set of all the labels, $\mathcal{C}_c$ is the set of indices $i$ for which $y_i = c \in \mathcal{C}$, $N_c$ is the number of samples in $\mathcal{X}$ with class $c$, $\mu_c$ is the mean of the training samples in class $c$ and $\mu$ is the mean of the whole training set. The solution to this problem can be found by taking the eigenvectors of $S_w^{-1}S_b$ associated with its largest eigenvalues to be the rows of $L$.

\subsubsection{\acf{ANMM}}

\acs{ANMM} \cite{anmm} is another \ac{DML} technique specifically oriented to dimensionality reduction that tries to solve some of the limitations of \acs{PCA} and \acs{LDA}. 

The objective of \acs{ANMM} is to learn a linear transformation $L \in \mathcal{M}_{d' \times d}(\R)$, with $d' \le d$ that maximizes an \emph{average neighborhood margin} defined, for each sample, by the difference between its average distance to its nearest neighbors of different class and the average distance to its nearest neighbors of same class.

If we consider the set $\mathcal{N}_i^o$ of the $\xi$ samples in $\mathcal{X}$ nearest to $x_i$ and with the same class as $x_i$, and the set $\mathcal{N}_i^e$ of the $\zeta$ samples in $\mathcal{X}$ nearest to $x_i$ and with a different class to $x_i$, we can express the global average neighborhood margin, for the distance defined by $L$, as

\[\gamma^L = \sum_{i=1}^N\left( \sum\limits_{k \colon x_k \in \mathcal{N}_i^e} \frac{\|Lx_i - Lx_k \|^2}{|\mathcal{N}_i^e|}- \sum\limits_{j \colon x_j \in \mathcal{N}_i^o} \frac{\|Lx_i - Lx_j \|^2}{|\mathcal{N}_i^o|} \right).\]

In this expression, each summand is associated with each sample $x_i$ in the training set, and the positive term inside each summand represents the average distance to its $\zeta$ nearest neighbors of different classes, while the negative term represents the average distance to its $\xi$ nearest neighbors of the same class. Therefore, these differences constitute the average neighborhood margins for each sample $x_i$. The global margin $\gamma^L$ can be expressed in terms of a scatterness matrix containing the information related to different-class neighbors, and a compactness matrix that stores the information corresponding to the same-class neighbors, that is,

\begin{equation*}
    \gamma^L = \tr(L(S-C)L^T),
\end{equation*}

where $S$ and $C$ are, respectively, the \emph{scatterness} and \emph{compactness} matrices, defined as
\begin{align*}
S &= \sum_{i}\sum_{k\colon x_k \in \mathcal{N}_i^e}\frac{(x_i-x_k)(x_i-x_k)^T}{|\mathcal{N}_i^e|} \\
C &= \sum_{i}\sum_{j\colon x_j \in \mathcal{N}_i^o}\frac{(x_i-x_j)(x_i-x_j)^T}{|\mathcal{N}_i^o|}.
\end{align*}

If we impose the scaling restriction $LL^T = I$ (scaling would increase the average neighborhood margin indefinitely), the average neighborhood margin can be maximized by taking the eigenvectors of $S-C$ associated with its largest eigenvalues to be the rows of $L$.

\subsection{Algorithms to Improve Nearest Neighbors Classifiers} \label{sssec:algs_knn}

One of the main applications of \ac{DML} is to improve other distance based learning algorithms. Since the nearest neighbors classifier is one of the most popular distance based classifiers many \ac{DML} algorithms are designed to improve this classifier, as is the case with \acs{LMNN} \cite{lmnn} and \acs{NCA} \cite{nca}.

\subsubsection{\acf{LMNN}} \label{dml:lmnn}

\acs{LMNN} \cite{lmnn} is a \ac{DML} algorithm aimed specifically at improving the accuracy of the $k$-nearest neighbors classifier. 

\acs{LMNN} tries to bring each sample as close as possible to its \emph{target neighbors}, which are $k$ pre-selected same-class samples requested to become the nearest neighbors of the sample, while trying to prevent samples from other classes from invading a margin defined by those target neighbors. This setup allows the algorithm to locally separate the classes in an optimal way for $k$-neighbors classification.

Assuming the sets of target neighbors are chosen (usually they are taken as the nearest neighbors for Euclidean distance), the error function that \acs{LMNN} minimizes is a two-term function. The first term is the target neighbors pulling term, given by
\[ \varepsilon_{pull}(M) = \sum_{i=1}^N \sum_{j \istargetof i} d_M(x_i, x_j)^2, \]
where $d_M$ is the Mahalanobis distance corresponding to $M \in S_d(\R)^+_0$ and $j \istargetof i$ iff $x_j$ is a target neighbor of $x_i$. The second term is the \emph{impostors} pushing term, given by
\[ \varepsilon_{push}(M) = \sum_{i=1}^N\sum_{j \istargetof i}\sum_{l=1}^N(1-y_{il})[1 + d_M(x_i,x_j)^2 - d_M(x_i, x_l)^2]_+,\]
where $y_{il} = 1$ if $y_i = y_l$ and $0$ otherwise, and $[\cdot]_+$ is defined as $[z]_+ = \max\{z, 0\}$. Finally, the objective function is given by
\[ \varepsilon(M) = (1-\mu)\varepsilon_{pull}(M) + \mu\varepsilon_{push}(M), \quad \mu \in ]0,1[.\]
This function can be optimized using semidefinite programming. It is possible to optimize this function in terms of $L$, using gradient methods, as well. By optimizing in terms of $M$ we gain convexity in the problem, while by optimizing in terms of $L$ we can use the algorithm to force a dimensionality reduction.

\subsubsection{\acf{NCA}}

\acs{NCA} \cite{nca} is another \ac{DML} algorithm aimed specifically at improving the accuracy of the nearest neighbors classifiers. It is designed is to learn a linear transformation with the goal of minimizing the leave-one-out error expected by the nearest neighbor classification.

To do this, we define the probability that a sample $x_i \in \mathcal{X}$ has $x_j \in \mathcal{X}$ as its nearest neighbor for the distance defined by $L \in \mathcal{M}_d(\R)$, $p_{ij}^L$, as the softmax
\begin{equation*}
    \begin{split}
    p_{ij}^L = \frac{\exp\left( - \|Lx_i - Lx_j \|^2 \right)}{\sum\limits_{k \ne i} \exp\left(-\|Lx_i - Lx_k \|^2\right)}\ \ (j \ne i),  
    \end{split}
    \quad\quad
    \begin{split}
    p_{ii}^L = 0.
    \end{split}
\end{equation*}

The expected number of correctly classified samples according to this probability is obtained as
\begin{equation*}
    f(L) = \sum_{i=1}^N \sum_{j \in C_i} p_{ij}^L,
\end{equation*}
where $C_i$ is the set of indices $j$ so that $y_j = y_i$. The function $f$ can be maximized using gradient methods, and the distance resulting from this optimization is the one that minimizes the expected leave-one-out error, and therefore, the one that \acs{NCA} learns.

\subsection{Algorithms to Improve Nearest Centroids Classifiers} \label{sssec:algs_ncm}

Apart from the nearest neighbors classifiers, other distance-based classifiers of interest are the so-called nearest centroid classifiers. These classifiers obtain a set of centroids for each class and classify a new sample by considering the nearest centroids to the sample. There are also \ac{DML} algorithms designed for these classifiers, as is the case for \acs{NCMML} and \acs{NCMC} \cite{ncmml}.

\subsubsection{\acf{NCMML}}

\acs{NCMML} \cite{ncmml} is a \ac{DML} algorithm specifically designed to improve the \ac{NCM} classifier. To do this, it uses a probabilistic approach similar to that used by \acs{NCA} to improve the accuracy of the nearest neighbors classifier.

In this case, we define the probability that a sample $x_i \in \mathcal{X}$ will be labeled with the class $c$, according to the nearest class mean criterion, for the distance defined by $L \in \mathcal{M}_{d'\times d}(\R)$, as

\begin{equation*}
    p_L(c|x) = \frac{\exp\left(-\frac{1}{2} \|L(x - \mu_c)\|^2\right)}{\sum\limits_{c' \in \mathcal{C}} \exp\left(-\frac{1}{2} \|L(x - \mu_{c'})\|^2\right)},
\end{equation*} 
where $\mathcal{C}$ is the set of all the classes and $\mu_c$ is the mean of the training samples with class $c$. The objective function that \acs{NCMML} tries to maximize is the log-likelihood for the labeled data in the training set, according to the probability defined above, that is,
\begin{equation*}
\mathcal{L}(L) = \frac{1}{N}\sum_{i=1}^N\log p_L(y_i|x_i).
\end{equation*} 
This function can be optimized using gradient methods.

\subsubsection{\acf{NCMC}}

\acs{NCMC} is the generalization of the nearest class mean classifier. In this classifier, a set with an arbitrary number of centroids is calculated for each class, using a clustering algorithm. Then, a new sample is classified by assigning the label of its nearest centroid.

An immediate generalization of \acs{NCMML} allows us to learn a distance directed at improving \acs{NCMC}. This \ac{DML} algorithm is also referred to as \acs{NCMC}. In this case, instead of the class means, we have a set of centroids $\{m_{c_j}\}_{j=1}^{k_c}$, for each class $c \in \mathcal{C}$. The generalized probability that a sample $x_i \in \mathcal{X}$ will be labeled with the class $c$ is now given by $p_L(c|x) = \sum_{j=1}^{k_c} p_L(m_{c_j}|x)$, where $p_L(m_{c_j}|x)$ are the probabilities that $m_{c_j}$ is the closest centroid to $x$, and is given by
\begin{equation*}
    p_L(m_{c_j}|x) = \frac{\exp\left( -\frac{1}{2} \|L(x-m_{c_j})\|^2 \right)}{ \sum\limits_{c \in \mathcal{C}} \sum\limits_{i=1}^{k_c} \exp\left( -\frac{1}{2} \|L(x-m_{c_i})\|^2 \right) }.
\end{equation*}

Again, \acs{NCMC} maximizes the log-likelihood function $\mathcal{L}(L) = \frac{1}{N}\sum_{i=1}^N p_L(y_i|x_i)$ using gradient methods.

\subsection{Information Theory Based Algorithms} \label{sssec:algs_information_theory}

Several \ac{DML} algorithms rely on information theory to learn their corresponding distances. The information theory concepts used in the algorithms we will introduce below are described in Appendix \ref{ssec:information_theory}. These algorithms have similar working schemes. First, they establish different probability distributions on the data, and then they try to bring these distributions closer or further away using divergences. The information theory based algorithms we will study are \acs{ITML} \cite{itml}, \acs{DMLMJ} \cite{dmlmj} and \acs{MCML} \cite{mcml}.

\subsubsection{\acf{ITML}}

\acs{ITML} \cite{itml} is a \ac{DML} technique intended to find a distance metric as close as possible to an initial pre-defined distance, on which similarity and dissimilarity constraints for same-class and different-class samples are satisfied. This approach tries to preserve the properties of the original distance while adapting it to our dataset thanks to the restrictions it adds.

We will denote the positive definite matrix associated with the initial distance as $M_0$. Given any positive definite matrix $M \in S_d(\R)^+$ and a fixed mean vector $\mu$, we can construct a normal distribution $p(x|M)$ with mean $\mu$ and covariance $M$. \acs{ITML} tries to minimize the Kullback-Leibler divergence between $p(x|M_0)$ and $p(x|M)$, subject to several similarity constraints on the data, that is

\begin{equation*}
    \begin{split}
    \min_{M \in S_d(\R)^+} &\quad \kl(p(x|M_0)\|p(x|M))  \\
    \text{s.t.: } &\quad d_M(x_i,x_j) \le u, \quad (i,j) \in S \\
                  &\quad d_M(x_i,x_j) \ge l, \quad (i,j) \in D,
    \end{split}
\end{equation*} 
where $S$ and $D$ are sets of pairs of indices on the elements of $\mathcal{X}$ that represent the samples considered similar and not similar, respectively (normally, same-labeled pairs and different-labeled pairs), and $u$ and $l$ are, respectively, upper and lower bounds for the similarity and dissimilarity constraints. This problem can be optimized using gradient methods combined with iterated projections in order to fulfill the constraints.

\subsubsection{\acf{DMLMJ}}

\acs{DMLMJ} \cite{dmlmj} is another \ac{DML} technique based on information theory that tries to separate, with respect to the Jeffrey divergence, two probability distributions, the first associated with similar points while the second is associated with dissimilar points. 

\acs{DMLMJ} defines two difference spaces: a $k$-\emph{positive difference space} that contains the differences between each sample in the dataset and its $k$-nearest neighbors from the same class, and a $k$-\emph{negative difference space} that contains the differences between each sample and its $k$-nearest neighbors from different classes. Over these spaces, for a distance determined by a linear transformation $L \in \mathcal{M}_{d' \times d}(\R)$, two gaussian distributions $P_L$ and $Q_L$ with equal mean are assumed. Then, the problem that \acs{DMLMJ} optimizes is

\[ \max_{L \in \mathcal{M}_{d'\times d}(\R)} \quad f(L) =  \jf(P_L\|Q_L) = \kl(P_L\|Q_L) + \kl(Q_L\|P_L).\]

This problem can be transformed into a trace optimization problem similar to those of \acs{PCA} and \acs{LDA}, and can also be solved by taking eigenvectors from the covariance matrices involved in the problem.

\subsubsection{\acf{MCML}}

\acs{MCML} \cite{mcml} is another \ac{DML} technique based on information theory. The key idea of this algorithm is the fact that we would obtain an ideal class separation if we could project all the samples from the same class on a same point, far enough away from the points on which the rest of the classes would be projected.

In order to try to achieve this, \acs{MCML} defines a probability that a sample $x_j$ will be classified with the same label as $x_i$, with the distance given by a positive semidefinite matrix $M \in S_d(\R)^+_0$, $p^M(j|i)$, as the softmax
\begin{equation*}
    p^{M}(j|i) = \frac{\exp(-d_M(x_i, x_j)^2)}{\sum\limits_{k\ne i} \exp(-d_M(x_i,x_k)^2)}.
\end{equation*}

Then, it also defines a probability $p_0(j|i)$ for the ideal situation in which all the same-class samples collapse into the same point, far enough away from the collapsing points of the other classes, given by
\begin{equation*}
    p_0(j|i) \propto \begin{cases}1, &\quad y_i = y_j \\ 0, &\quad y_i \ne y_j\end{cases}.
\end{equation*}

\acs{MCML} tries to bring $p^M(\cdot|i)$ as close to the ideal $p^0(\cdot|i)$ as possible, for each $i$, using the Kullback-Leibler divergence between them. Therefore, the optimization problem is formulated as
\begin{equation*}
    \min_{M \in S_d(\R)^+_0}\quad f(M) = \sum_{i=1}^N \kl \left[ p_0(\cdot|i) \| p^M(\cdot|i) \right].
\end{equation*} 

This function can be minimized using semidefinite programming.

\subsection{Other Distance Metric Learning Techniques} \label{sssec:other_dmls}

In this section we will study some different proposals for \ac{DML} techniques. The algorithms we will analyze are \acs{LSI} \cite{lsi}, \acs{DML-eig} \cite{dmleig} and \acs{LDML} \cite{ldml}.

\subsubsection{\acf{LSI}}

\acs{LSI} \cite{lsi}, also sometimes referred to as \acf{MMC} is possibly one of the first algorithms that has helped make the concept of \ac{DML} more well known. This algorithm is a global approach that tries to bring same-class data closer together while keeping data from different classes far enough apart.

Assuming that the sets $S$ and $D$ represent, respectively, pairs of samples that should be considered similar or dissimilar (i.e. samples that belong to the same class or to different classes, respectively), \acs{LSI} looks for a positive semidefinite matrix $M \in S_d(\R)^+_0$ that optimizes the following problem:

\begin{equation*}
\begin{split}
    \min_{M} &\quad \sum_{(x_i,x_j)\in S}  d_M(x_i, x_j)^2 \\
    \text{s.t.: } &\quad \sum_{(x_i,x_j) \in D} d_M(x_i, x_j) \ge 1 \\
                  &\quad M \in S_d(\R)^+_0.
\end{split}
\end{equation*}

This problem can be optimized using gradient descent together with iterated projections in order to fulfill the constraints.

\subsubsection{\acf{DML-eig}}

\acs{DML-eig} \cite{dmleig} is a \ac{DML} algorithm inspired by the \acs{LSI} algorithm of the previous section, proposing a very similar optimization problem but offering a completely different resolution method, based on eigenvalue optimization.

We will once again consider the two sets $S$ and $D$, of pairs of samples that are considered similar and dissimilar, respectively. \acs{DML-eig} proposes an optimization problem that slightly differs from that proposed by the \acs{LSI} algorithm, given by

\begin{equation*}
\begin{split}
    \max_{M} &\quad \min_{(x_i,x_j)\in D}  d_M(x_i, x_j)^2 \\
    \text{s.t.: } &\quad \sum_{(x_i,x_j) \in S} d_M(x_i,x_j)^2 \le 1 \\
                  &\quad M \in S_d(\R)^+_0.
\end{split}
\end{equation*}

This problem can be transformed into a minimization problem for the largest eigenvalue of a symmetric matrix. This is a well-known problem and there are  some iterative methods that allow this minimum to be reached \cite{overton1988minimizing}.

\subsubsection{\acf{LDML}}

\acs{LDML} \cite{ldml} is a \ac{DML} algorithm in which the optimization model makes use of the logistic function.

Recall that the \emph{logistic} or \emph{sigmoid} function is the map $\sigma \colon \R \to \R$ given by
\[ \sigma(x) = \frac{1}{1+e^{-x}}. \]
In \acs{LDML}, the logistic function is used to define a probability, which will assign the greater probability the smaller the distance between points. Given a positive semidefinite matrix $M \in S_d(\R)^+_0$, this probability is expressed as
\begin{equation*}
    p_{ij,M} = \sigma(b - \|x_i-x_j\|_M^2),
\end{equation*} 
where $b$ is a positive threshold value that will determine the maximum value achievable by the logistic function, and that can be estimated by cross validation. \acs{LDML} tries to maximize the log-likelihood given by
\begin{equation*}
    \mathcal{L}(M) = \sum_{i,j=1}^N y_{ij}\log p_{ij,M} + (1-y_{ij})\log(1-p_{ij,M}),
\end{equation*}
where $y_{ij}$ is a binary variable that takes the value 1 if $y_i = y_j$ and 0 otherwise. This function can be optimized using semidefinite programming.

\subsection{Kernel Distance Metric Learning} \label{sssec:kernel_dml}

Kernel methods constitute a paradigm within machine learning that is very useful in many of the problems addressed in this discipline. They usually arise in problems where the learning algorithm capability is reduced, typically due to the shape of the dataset. A classic learning algorithm where the kernel trick is very useful is the \emph{\ac{SVM}} classifier \cite{burges1998tutorial}. An example for this case is given in Figure \ref{fig:ker_svm}.

\begin{figure}[htbp]
    \centering
    \subfloat{\includegraphics[width=0.4\textwidth]{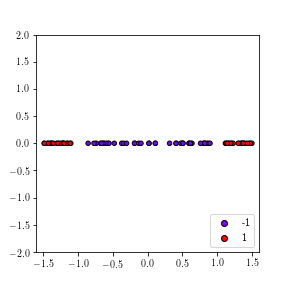}}
    \subfloat{\includegraphics[width=0.4\textwidth]{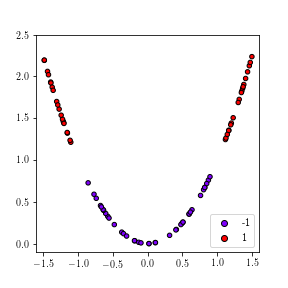}} \\
    \subfloat{\includegraphics[width=0.4\textwidth]{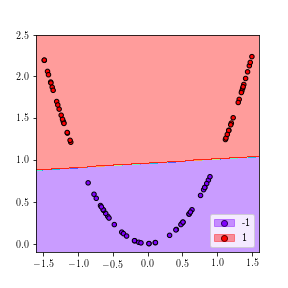}}
    \subfloat{\includegraphics[width=0.4\textwidth]{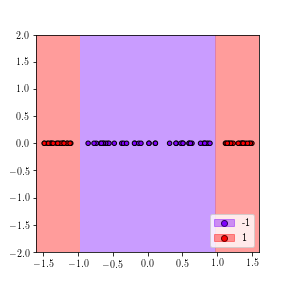}}

    \caption{\ac{SVM} and kernel trick. This binary classifier looks for the hyperplane that best separates both classes. Therefore, it is highly limited when the dataset is not separable by hyperplanes, as is the case for the dataset in the upper-left image. A solution consists of sending the data into a higher dimensional space, where data can be separated by hyperplanes, and apply the algorithm there, as it can be seen in the remaining images. The kernel trick allows us to execute the algorithm only in terms of the dot products of the samples in the new space, which makes it possible to work on very high dimensional spaces, or even infinite dimensional spaces. The existence of a \emph{representer theorem} for \ac{SVM} also allows the solution to be rewritten in terms of a vector with the size of the number of samples.} \label{fig:ker_svm}
\end{figure}

    In \ac{DML}, the usefulness of kernel learning is a consequence of the limitations given by the Mahalanobis distances. Although learned metrics can later be used with non-linear classifiers, such as the nearest neighbors classifier, the metrics themselves are determined by linear transformations, which, in turn, are determined by the image of a basis in the departure space, which results in the fact that we only have the freedom to choose the image of as much data as the dimension has the space, mapping the rest of the vectors by linearity. When the amount of data is much larger than the space dimension this can become a limitation.

    The kernel approach for \ac{DML} follows a similar scheme to that of \ac{SVM}. If we work with a dataset $\mathcal{X} = \{x_1,\dots,x_N\} \subset \R^d$, the idea is to send the data to a higher dimensional space, using a mapping $\phi \colon \R^d\to \mathcal{F}$, where $\mathcal{F}$ is a Hilbert space called the \emph{feature space}, and then to learn in the feature space using a \ac{DML} algorithm. The way we will learn a distance in the feature space will be via a continuous linear transformation $L \colon \mathcal{F} \to \R^{d'}$, where $d' \le d$ (observe that $L$ is not necessarily a matrix, since $\mathcal{F}$ is not necessarily finite dimensional), which we will also denote as $L \in \mathcal{L}(\mathcal{F},\R^{d'})$.

    As occurs with \ac{SVM}, a great inconvenience arises when sending the data to the feature space, and that is that the problem dimension can greatly increase, and therefore the application of the algorithms can be very expensive computationally. In addition, if we want to work in infinite dimensional feature spaces, it is impossible to deal with the data in this case, unless we turn to the kernel trick.

    We define the \emph{kernel function} as the mapping $K \colon \R^d\times \R^d \to \R$ given by $K(x,x') = \langle \phi(x),\phi(x')\rangle$. The success of kernel functions is due to the fact that many learning algorithms only need to know the dot products between the elements in the training set to be able to work. This will happen in the \ac{DML} algorithms we will study later. We can observe, as an example, that the calculation of Euclidean distances, which is essential in many \ac{DML} algorithms, can be made using only the kernel function. Indeed, for $x,x' \in \R^d$, we have
    \begin{equation} \label{eq:dist_features}
    \begin{split}
    \|\phi(x)-\phi(x')\|^2 &= \langle \phi(x)-\phi(x'), \phi(x) - \phi(x') \rangle \\
                              &= \langle \phi(x),\phi(x) \rangle - 2 \langle\phi(x), \phi(x') \rangle + \langle \phi(x'), \phi(x') \rangle \\
                              &= K(x,x) + K(x',x') -2K(x,x').
    \end{split}
    \end{equation}

    The next common problem for all the kernel-based \ac{DML} algorithms is how to deal with the learned transformation. Since we are trying to learn a map $L \in \mathcal{L}(\mathcal{F},\R^{d'})$, we may not be able to write it as a matrix, and when we can, this matrix may have dimensions that are too large. However, as $L$ is continuous and linear, using the Riesz representation theorem, we can rewrite $L$ as a vector of dot products by fixed vectors, that is, $L = (\langle \cdot, w_1 \rangle, \dots, \langle \cdot, w_{d'} \rangle)$, where $w_1,\dots,w_{d'} \in \mathcal{F}$. Furthermore, for the algorithms we will study, several \emph{representer theorems} are known \cite{neurocomputing,hofmann2008kernel,kda,dmlmj,kpca}. These theorems allow the vectors $w_i$ to be expressed as a linear combination of the samples in the feature space, that is, for each $i \in \{1,\dots,d'\}$, there is a vector $\alpha^i = (\alpha^i_1,\dots,\alpha^i_N) \in \R^N$ so that $w_i = \sum_{j=1}^N \alpha_j^i \phi(x_j)$. Consequently, we can see that
    \begin{equation} \label{eq:representer}
      L\phi(x) = A \begin{pmatrix} K(x_1,x) \\ \vdots \\ K(x_N,x) \end{pmatrix},
    \end{equation}
    where $A \in \mathcal{M}_{d'\times N}(\R)$ is given by $A_{ij} = \alpha_j^i$.

    Thanks to these theorems, we can address the problem computationally as long as we are able to calculate the coefficients of matrix $A$. When transforming a new sample it will suffice to construct the previous column matrix by evaluating the kernel function between the sample and each element in the training set, and then multiplying $A$ by this matrix. On a final note, when training it is useful to view the kernel map as a matrix $K \in S_N(\R)$, where $K_{ij} = K(x_i,x_j)$. A similar (in this case not necessarily square) matrix can be constructed when testing, with all the dot products between the train and test samples. By choosing the appropriate column of this matrix, we will be able to transform the corresponding test sample using Eq. \ref{eq:representer}.

    Each \ac{DML} technique that supports the use of kernels will use different tools for its performance, each one based on the original algorithms. In \ref{ssec:kernel_dml} we describe the kernelizations of some of the algorithms already introduced, namely, \acs{LMNN}, \acs{ANMM}, \acs{DMLMJ} and \acs{LDA}.


\section{Experimental Framework and Results} \label{sec:experiments}

With the algorithms introduced in the previous section, several experiments have been carried out. This section describes these experiments and shows the results.

\subsection{Description of the Experiments}

For the \ac{DML} algorithms studied, a collection of experiments has been developed, consisting of the following procedures.

\begin{enumerate}
  \item Evaluation of all the algorithms capable of learning at maximum dimension, applied to the \ac{kNN} classification, for different values of $k$. \label{exp:normal}
  \item Evaluation of the algorithms aimed at improving nearest centroid classifiers, applied to the corresponding centroid-based classifiers. \label{exp:ncm}
  \item Evaluation of kernel-based algorithms, experimenting with different kernels, applied to the nearest neighbors classification. \label{exp:ker}
  \item Evaluation of algorithms capable of reducing dimensionality, for different dimensions, applied to the nearest neighbors classification. \label{exp:dim}
\end{enumerate}

When we mention in experiment \ref{exp:normal} that an algorithm is ``capable of learning at maximum dimension'' we are excluding those dimensionality reduction algorithms that only learn a change of axes, as is the case with both \acs{PCA} and \acs{ANMM}, which at maximum dimension learn a transformation whose associated distance is still the Euclidean. \acs{LDA} is kept, assuming that it will always take the maximum dimension that it is able to, which will be the number of classes of the problem. The algorithms directed at centroid-based classifiers are also excluded from experiment \ref{exp:normal}, together with those based on kernels, which will be analyzed in experiments \ref{exp:ncm} and \ref{exp:ker}, respectively.

The stated experiments indicate that the magnitude with which we will measure the performance of the algorithms is the result of the $k$-neighbors classification, except in the case of the algorithms based on centroids, which will use their corresponding classifier. These classifiers will be evaluated by a 10-fold cross validation. The results obtained from the predictions on the training set will also be included, in order to evaluate possible overfitting.

To evaluate the algorithms, we will use the implementations available in the Python library \texttt{pyDML} \cite{suarez2020pydml}. The algorithms will be executed using their default parameters, which can be found in the \texttt{pyDML} documentation\footnote{\url{https://pydml.readthedocs.io/}}. These default parameters have been set with standard values. The following exceptions to the default parameters have been made:

\begin{itemize}
  \item The \acs{LSI} algorithm will have the parameter \texttt{supervised = True}, as it will be used for supervised learning.
  \item In the dimensionality reduction experiment (\ref{exp:dim}), the algorithms will have the dimension number parameter set according to the dimension being evaluated.
  \item The parameter \texttt{k} of \acs{LMNN} and \acs{KLMNN} will be equal to the number of neighbors being considered in the nearest neighbors classification.
  \item \acs{LMNN} will be executed with stochastic gradient descent, instead of semidefinite programming, in dimensionality reduction experiments, thus learning a linear transformation instead of a metric.
  \item The parameters \texttt{n\_friends} and \texttt{n\_enemies} of \acs{ANMM} and \acs{KANMM} will be equal to the number of neighbors being considered in the nearest neighbors classification.
  \item The parameter \texttt{n\_neighbors} of \acs{DMLMJ} and \acs{KDMLMJ} will be equal to the number of neighbors being considered in the nearest neighbors classification.
  \item The parameter \texttt{centroids\_num} of \acs{NCMC} will be equal to the parameter \texttt{centroids\_num} being considered in its corresponding classifier, \texttt{NCMC\_Classifier}.
\end{itemize}

As for the datasets used in the experiments, up to 34 datasets have been collected and all of them are available in KEEL\footnote{KEEL, \emph{knowledge extraction based on evolutionary learning} \cite{triguero2017keel}: \url{http://www.keel.es/}.}. All these datasets are numeric, do not contain missing values, and are oriented to standard classification problems. In addition, although some of the \ac{DML} algorithms scale well with the number of samples, others cannot deal with datasets that are too large, so it was decided that for sets with a high number of samples, a subset of a size that all algorithms can deal with, keeping the class distribution the same, would be selected. The characteristics of these datasets are described in Table \ref{fig:exp:datasets}. All datasets have been \emph{min-max} normalized to the interval $[0,1]$, feature to feature, prior to the execution of the experiments.

\begin{table}[htbp]
\centering
    \begin{tabular}{lrrr}
\toprule
Dataset              &  Number of samples & Number of features & Number of classes \\
\midrule
appendicitis              &    106      &    7       &   2    \\
balance                   &    625      &    4       &   3    \\
bupa                      &    345      &    6       &   2    \\
cleveland                 &    297      &    13       &  5     \\
glass                     &    214      &    9       &   7    \\
hepatitis                 &    80      &     19      &   2    \\
ionosphere                &    351      &    33       &   2    \\
\bottomrule
\end{tabular}
\caption{Datasets used in the experiments. \label{fig:exp:datasets}}
\end{table}

Finally, we describe the details of the experiments \ref{exp:normal}, \ref{exp:ncm}, \ref{exp:ker} and \ref{exp:dim}:
\begin{enumerate}
  \item Algorithms will be evaluated with the classifiers 3-NN, 5-NN and 7-NN.
  \item \acs{NCMML} will be evaluated with the \texttt{Scikit-Learn} \acs{NCM} classifier, while \acs{NCMC} will be evaluated with its associated classifier, available in \texttt{pyDML}, for two different values: 2 centroids per class and 3 centroids per class.
  \item Algorithms will be evaluated with 3-NN classifier, using the following kernels: linear (\texttt{Linear}), grade-2 (\texttt{Poly-2}) and grade-3 (\texttt{Poly-3}) polynomials, gaussian (\texttt{RBF}) and laplacian (\texttt{Laplacian}). The kernel version of \acs{PCA}\footnote{It is implemented in \texttt{Scikit-Learn}: \url{http://scikit-learn.org/stable/modules/generated/sklearn.decomposition.KernelPCA.html}. Its theoretical details can be found in \cite{kpca}.} will be also included in the comparison. Only the smallest datasets will be considered, so that they can be applicable to the algorithms that scale the worst with the dimension (recall that the kernel trick forces algorithms to work in dimensions of the order of the number of samples).
  \item Algorithms will be evaluated with the classifiers 3-NN, 5-NN and 7-NN. The dimensions used are: $1, 2, 3, 5, 10, 20, 30, 40, 50, $ the maximum dimension of the dataset, and the number of classes of the dataset minus 1. In this case, the following high-dimensionality datasets are selected: \texttt{sonar}, \texttt{movement\_libras} and \texttt{spambase}. The algorithms to be evaluated in this experiment are: \acs{PCA}, \acs{LDA}, \acs{ANMM}, \acs{DMLMJ}, \acs{LMNN} and \acs{NCA}.
\end{enumerate}

\subsection{Results}

This section shows the results of the cross-validation for the different experiments. We will only show the results of the 3-NN classifier in the experiments that use nearest neighbors classifiers in this text. The results obtained for the remaining \ac{kNN} used in the experiments are available on the pyDML-Stats\footnote{\label{pydml-stats}Source code: \url{https://github.com/jlsuarezdiaz/pyDML-Stats}. The current website is located at \url{https://jlsuarezdiaz.github.io/software/pyDML/stats/versions/0.0.1-1/}} website, where the results of all these experiments have been stored. The scripts used to do the experiments can also be found on this website. We have added the average score obtained and the average ranking to the results of the experiments \ref{exp:normal}, \ref{exp:ncm} and \ref{exp:ker}. The ranking has been made by assigning integer values between 1 and $m$, where $m$ is the number of algorithms being compared in each experiment (adding half fractions in case of a tie), according to the position of the algorithms over each dataset, 1 being the best algorithm, and $m$ the worst. The content of the different tables elaborated is described below.

\begin{itemize}
    \item Table \ref{results:normal:3nn} shows the cross-validation results obtained for experiment \ref{exp:normal}, using the 3-NN score as the evaluation measure. Some cells do not show results because the algorithm did not converge.

    \item Table \ref{results:ncm} shows the results of experiment \ref{exp:ncm}. \acs{NCM} and \acs{NCMC} classifiers with 2 and 3 centroids per class were used as evaluation measures. For each classifier, the Euclidean distance (\texttt{Euclidean + CLF}) and the distance learning algorithm associated with the classifier (\texttt{NCMML / NCMC (2 ctrd) / NCMC (3 ctrd)}) have been evaluated.

    \item Table \ref{results:ker:3nn:train} shows the cross-validation results obtained on the training set for the kernel-based algorithms using the 3-NN classifier. Table \ref{results:ker:3nn:test} shows the corresponding results obtained on the test set.

    \item Table \ref{results:dim:sonar} shows the cross-validation results for experiment \ref{exp:dim} in dataset \texttt{sonar}, using the classifier 3-NN. On the left are the results for the training set, and on the right, the results for the test set. Each row shows the results for the different dimensions evaluated. Tables \ref{results:dim:movement_libras} and \ref{results:dim:spambase} show the corresponding dimensionality results over the datasets \texttt{movement\_libras} and \texttt{spambase}, respectively.
\end{itemize}

\begin{landscape}
\begin{table}[!htbp]
\resizebox{1.42\textwidth}{!}{%
    \newcolumntype{s}{>{\columncolor{LightGray}}r}
\begin{tabular}{lsrsrsrsrsrsrsrsrsrsr}
\toprule
{} &  \multicolumn{2}{c}{Euclidean} &      \multicolumn{2}{c}{LDA} &      \multicolumn{2}{c}{ITML} &     \multicolumn{2}{c}{DMLMJ} &       \multicolumn{2}{c}{NCA} &      \multicolumn{2}{c}{LMNN} &       \multicolumn{2}{c}{LSI} &   \multicolumn{2}{c}{DML-eig} &      \multicolumn{2}{c}{MCML} &      \multicolumn{2}{c}{LDML} \\

{} & Train & Test & Train & Test & Train & Test & Train & Test & Train & Test & Train & Test & Train & Test & Train & Test & Train & Test & Train & Test \\
\midrule
appendicitis    & .8428 & .8339 & .8428 & .8522 & .8533 & .8604 & .8490 & .8256 & \textbf{.8700} & .8504 & .8407 & .8422 & .8659 & \textbf{.8630} & .8585 & .8622 & .8502 & .8513 & .8669 & .8422 \\
balance         & .8049 & .8082 & .8885 & .8992 & .8986 & .8943 & .8286 & .8191 & \textbf{.9592} & \textbf{.9584} & .8202 & .8175 & .9182 & .9280 & .8947 & .8945 & .8816 & .8737 & .8874 & .8895 \\
bupa            & .6231 & .6546 & .6338 & .6465 & .6466 & .6281 & .6660 & \textbf{.6776} & \textbf{.6943} & .5994 & .6099 & .6342 & .6363 & .6284 & .5993 & .6120 & .5716 & .5742 & .5826 & .5854 \\
cleveland       & .5570 & .5468 & .5694 & .5502 & .5488 & .5523 & .5626 & .5636 & \textbf{.6804} & .5436 & .5780 & .5803 & .5518 & .5722 & .5896 & .5829 & .5978 & .5785 & .5784 & \textbf{.5972} \\
glass           & .6759 & .7015 & .6235 & .6231 & .6453 & .6549 & \textbf{.7092} & .7041 & .7065 & .6917 & .6780 & \textbf{.7067} & .6495 & .6235 & .6407 & .6263 & .6319 & .5850 & .6242 & .6063 \\
hepatitis       & .8236 & .8325 & .9402 & .8609 & .9002 & .8815 & .8821 & .8894 & \textbf{.9569} & .8325 & .9514 & .8418 & .9139 & .9130 & .9125 & \textbf{.9176} & .9250 & .8829 & .9458 & .8547 \\
\end{tabular}
}

\caption{Results of cross-validation with 3-NN.} \label{results:normal:3nn}
\end{table}
\end{landscape}

\begin{table}[!htbp]
\makebox[\textwidth][c]{%
\resizebox{1.1\textwidth}{!}{%
    \newcolumntype{s}{>{\columncolor{LightGray}}p{15mm}}
\newcolumntype{t}{>{\columncolor{LightGray}}p{10mm}}
\newcolumntype{u}{>{\columncolor{LightGray}}p{20mm}}
\begin{tabular}{lsp{15mm}tp{10mm}up{20mm}tp{10mm}up{20mm}tp{10mm}}
\toprule
{} &  \multicolumn{2}{c}{Euclidean + NCM} &     \multicolumn{2}{c}{NCMML} &  \multicolumn{2}{c}{Euclidean + NCM (2 ctrd)} &  \multicolumn{2}{c}{NCMC (2 ctrd)} &  \multicolumn{2}{c}{Euclidean + NCM (3 ctrd)} &  \multicolumn{2}{c}{NCMC (3 ctrd)} \\
{} & Train & Test & Train & Test & Train & Test & Train & Test & Train & Test & Train & Test \\
\midrule
appendicitis    & .8365          &  \textbf{.8640} & \textbf{.8512} & .8440          & .6479          & .6031          & .6248          & .6246          & .8102          & .7800          & .7557          & .7266          \\
balance         & .7496          &  .7475          & .6865          & .6864          & .7004          & .6721          & .8280          & \textbf{.8225} & .7160          & .6654          & \textbf{.8321} & .8222          \\
bupa            & .5996          &  .6004          & \textbf{.6628} & \textbf{.6407} & .6270          & .6058          & .6247          & .6260          & .6589          & .5903          & .6409          & .5826          \\
cleveland       & .5742          &  .5377          & .6296          & \textbf{.5516} & .5727          & .5050          & .6280          & .5093          & .6203          & .4871          & \textbf{.6435} & .5135          \\
glass           & .5218          &  .4862          & .6221          & .5290          & .6479          & .5849          & .5877          & .5372          & .6427          & .5208          & \textbf{.7461} & \textbf{.6421} \\
hepatitis       & .8500          &  .8293          & .9832          & \textbf{.8688} & .8792          & .8436          & .9513          & .8373          & .8986          & .7946          & \textbf{.9833} & .8672          \\
\end{tabular}
}%
}
\caption{Results of the experiments with \acs{NCMML} and \acs{NCMC}.} \label{results:ncm}
\end{table}

\begin{landscape}
\begin{table}[!htbp]
\makebox[\linewidth][c]{%
\resizebox{1.5\textwidth}{!}{%
    \begin{tabular}{lrrrrrrrrrrrrrrrrrrrrrrrrrr}
\toprule
{} &  EUC &  \multicolumn{5}{c}{KPCA} &  \multicolumn{5}{c}{KDA} &  \multicolumn{5}{c}{KANMM} &  \multicolumn{5}{c}{KDMLMJ} & \multicolumn{5}{c}{KLMNN} \\
{} & {} & Lin. & Poly-2 & Poly-3 & RBF & Lapl. & Lin. & Poly-2 & Poly-3 & RBF & Lapl. & Lin. & Poly-2 & Poly-3 & RBF & Lapl. & Lin. & Poly-2 & Poly-3 & RBF & Lapl. & Lin. & Poly-2 & Poly-3 & RBF & Lapl. \\
\midrule
appendicitis    &   .8428 &       .8428 &       .8407 &       .8355 &    .8428 &          .8438 &      .8481 &      .8491 &      .8502 &   .8679 &         .8564 &        .8543 &        .8543 &        .8523 &     .8763 &           .8889 &         .8491 &         .8397 &         .8407 &      .8575 &            .8805 &        .8355 &        .8397 &        .8355 &     .8523 &           \textbf{1.000} \\
balance         &   .8049 &       .8217 &       .7770 &       .7923 &    .8136 &          .8288 &      .4547 &      .6756 &      .6279 &   .7497 &         .8838 &        .7764 &        .7980 &        .8224 &     .8462 &           .8327 &         .8298 &         .9509 &         .9696 &      .9374 &            .9491 &        .8202 &        .8426 &        .9079 &     .8865 &           \textbf{.9852} \\
bupa            &   .6231 &       .6231 &       .6235 &       .6222 &    .6231 &          .6560 &      .5623 &      .5677 &      .5758 &   .5526 &         .5665 &        .5903 &        .5890 &        .5832 &     .5639 &           .5735 &         .6586 &         .6254 &         .6457 &      .6247 &            .7075 &        .6437 &        .6547 &        .6627 &     .6721 &           \textbf{.9987} \\
\end{tabular}
}%
}
\caption{Results of kernel experiments on the training set.} \label{results:ker:3nn:train}
\end{table}

\begin{table}[!htbp]
\makebox[\linewidth][c]{%
\resizebox{1.5\textwidth}{!}{%
    \begin{tabular}{lrrrrrrrrrrrrrrrrrrrrrrrrrr}
\toprule
{} &  EUC &  \multicolumn{5}{c}{KPCA} &  \multicolumn{5}{c}{KDA} &  \multicolumn{5}{c}{KANMM} &  \multicolumn{5}{c}{KDMLMJ} & \multicolumn{5}{c}{KLMNN} \\
{} & {} & Lin. & Poly-2 & Poly-3 & RBF & Lapl. & Lin. & Poly-2 & Poly-3 & RBF & Lapl. & Lin. & Poly-2 & Poly-3 & RBF & Lapl. & Lin. & Poly-2 & Poly-3 & RBF & Lapl. & Lin. & Poly-2 & Poly-3 & RBF & Lapl. \\
\midrule
appendicitis    &   .8339 &       .8339 &       .8339 &       .8248 &    .8339 &          .8546 &      .8613 &      .8813 &      .8813 &   .8713 &         .8622 &        .8646 &        .8446 &        .8446 &     .8813 &           \textbf{.8904} &         .8339 &         .8248 &         .8339 &      .8339 &            .8248 &        .8322 &        .8248 &        .8331 &     .8057 &           .8422 \\
balance         &   .8082 &       .8383 &       .7823 &       .8032 &    .8428 &          .8336 &      .4597 &      .6699 &      .6150 &   .7407 &         .8899 &        .7738 &        .8058 &        .8382 &     .8559 &           .8320 &         .8144 &         .9582 &         \textbf{.9711} &      .9374 &            .9519 &        .8222 &        .8367 &        .9118 &     .8736 &           .8464 \\
bupa            &   .6546 &       .6546 &       .6661 &       .6662 &    .6546 &          .6862 &      .5620 &      .5191 &      .5566 &   .5244 &         .5389 &        .5963 &        .6022 &        .5908 &     .6199 &           .5794 &         .6777 &         .6375 &         .6371 &      .6310 &            \textbf{.6924} &        .6574 &        .6402 &        .6516 &     .6632 &           .6467 \\
\end{tabular}
}%
}
\caption{Results of kernel experiments on the test set.} \label{results:ker:3nn:test}
\end{table}

\end{landscape}

\begin{table}[!htbp]
\makebox[\textwidth][c]{%
\subfloat{%
    \resizebox{0.53\textwidth}{!}{%
        \begin{tabular}{lrrrrrr}
\toprule
{} &       PCA &       LDA &      ANMM &     DMLMJ &       NCA &      LMNN \\
\midrule
1             &  .5016 &  .9011 &  .6965 &  .7826 &  .9214 &  .7237 \\
2             &  .5891 &  -     &  .7670 &  .8050 &  .9807 &  .8782 \\
3             &  .7729 &  -     &  .8359 &  .8333 &  .9770 &  .9513 \\
5             &  .8215 &  -     &  .8904 &  .9033 &  .9759 &  .9914 \\
10            &  .8600 &  -     &  .8958 &  .9652 &  .9764 &  .9994 \\
20            &  .8541 &  -     &  .8872 &  .9583 &  .9668 &  \textbf{1.000} \\
30            &  .8456 &  -     &  .8627 &  .9508 &  .9706 &  \textbf{1.000} \\
40            &  .8365 &  -     &  .8424 &  .9460 &  .9839 &  \textbf{1.000} \\
50            &  .8312 &  -     &  .8370 &  .9263 &  .9850 &  \textbf{1.000} \\
Max. Dimension &  .8317 &  -     &  .8317 &  .9097 &  .9823 &  \textbf{1.000} \\
N. Classes - 1 &  .5016 &  .9011 &  .6965 &  .7826 &  .9214 &  .7237 \\
\bottomrule
\end{tabular}
    }%
}%
\subfloat{%
    \resizebox{0.53\textwidth}{!}{%
            \begin{tabular}{lrrrrrr}
\toprule
{} &       PCA &       LDA &      ANMM &     DMLMJ &       NCA &      LMNN \\
\midrule
1             &  .5619 &  .7782 &  .6770 &  .7256 &  .8073 &  .6640 \\
2             &  .6293 &  -     &  .7541 &  .7113 &  .8077 &  .7593 \\
3             &  .7641 &  -     &  .8395 &  .7741 &  .8265 &  .8077 \\
5             &  .8075 &  -     &  .8263 &  .8182 &  .8220 &  .8408 \\
10            &  .8699 &  -     &  .8751 &  .8651 &  .8270 &  .8654 \\
20            &  .8601 &  -     &  .8749 &  \textbf{.8844} &  .8699 &  .8703 \\
30            &  .8610 &  -     &  .8649 &  .8749 &  .8653 &  .8754 \\
40            &  .8465 &  -     &  .8610 &  .8697 &  .8792 &  .8613 \\
50            &  .8565 &  -     &  .8515 &  .8654 &  .8558 &  .8706 \\
Max. Dimension   &  .8370 &  -     &  .8370 &  .8361 &  .8703 &  .8706 \\
N. Classes - 1   &  .5619 &  .7782 &  .6770 &  .7256 &  .8073 &  .6640 \\
\bottomrule
\end{tabular}
    }%
}%
}
\caption{Results of dimensionality reduction experiments on \texttt{sonar} with 3-NN (train - test)} \label{results:dim:sonar}
\end{table}

\begin{table}[!htbp]
\makebox[\textwidth][c]{%
\subfloat{%
    \resizebox{0.53\textwidth}{!}{%
        \begin{tabular}{lrrrrrr}
\toprule
{} &       PCA &       LDA &      ANMM &     DMLMJ &       NCA &      LMNN \\
\midrule
1             &  .1938 &  .3339 &  .2414 &  .2720 &  .3360 &  .2547 \\
2             &  .2813 &  .5362 &  .4597 &  .4720 &  .6638 &  .5416 \\
3             &  .5232 &  .6143 &  .6435 &  .6684 &  .7195 &  .6900 \\
5             &  .6873 &  .7211 &  .7473 &  .7918 &  .8188 &  .8156 \\
10            &  .7831 &  .8661 &  .8053 &  .8857 &  .8383 &  .8485 \\
20            &  .7978 &  -     &  .7972 &  .8705 &  .8442 &  .8490 \\
30            &  .7981 &  -     &  .7978 &  .8652 &  .8438 &  .8514 \\
40            &  .7972 &  -     &  .7975 &  .8594 &  .8469 &  .8526 \\
50            &  .7972 &  -     &  .7972 &  .8538 &  .8431 &  .8498 \\
Max. Dimension &  .7972 &  -     &  .7972 &  .8460 &  .8516 &  .8490 \\
N. Classes - 1 &  .7932 &  .8685 &  .8061 &  \textbf{.8901} &  .8398 &  .8438 \\
\bottomrule
\end{tabular}
    }%
}%
\subfloat{%
    \resizebox{0.53\textwidth}{!}{%
            \begin{tabular}{lrrrrrr}
\toprule
{} &       PCA &       LDA &      ANMM &     DMLMJ &       NCA &      LMNN \\
\midrule
1             &  .1747 &  .3169 &  .2675 &  .2694 &  .2606 &  .2673 \\
2             &  .2574 &  .4553 &  .4800 &  .4476 &  .6181 &  .5251 \\
3             &  .5483 &  .4978 &  .6680 &  .6684 &  .6920 &  .6499 \\
5             &  .7177 &  .5938 &  .7763 &  .7774 &  .7655 &  .8012 \\
10            &  .8007 &  .7001 &  .8119 &  .8711 &  .8017 &  .8220 \\
20            &  .8139 &  -     &  .8106 &  \textbf{.8829} &  .8143 &  .8333 \\
30            &  .8139 &  -     &  .8139 &  .8696 &  .8191 &  .8133 \\
40            &  .8139 &  -     &  .8139 &  .8605 &  .8323 &  .8233 \\
50            &  .8139 &  -     &  .8139 &  .8627 &  .8310 &  .8255 \\
Max. Dimension &  .8139 &  -     &  .8139 &  .8649 &  .8319 &  .8133 \\
N. Classes - 1 &  .8185 &  .6642 &  .8137 &  .8811 &  .8274 &  .8211 \\
\bottomrule
\end{tabular}
    }%
}%
}
\caption{Results of dimensionality reduction experiments on \texttt{movement\_libras} with 3-NN (train - test)} \label{results:dim:movement_libras}
\end{table}

\begin{table}[!htbp]
\makebox[\textwidth][c]{%
\subfloat{%
    \resizebox{0.53\textwidth}{!}{%
        \begin{tabular}{lrrrrrr}
\toprule
{} &       PCA &       LDA &      ANMM &     DMLMJ &       NCA &           LMNN \\
\midrule
1             &  .8369 &  .9215 &  .8567 &  .6995 &  \textbf{.9420} & .9340 \\
2             &  .8316 &  -     &  .8869 &  .7724 &  \textbf{.9420} & .9386 \\
3             &  .8487 &  -     &  .8973 &  .8886 &  .9388 & .9335 \\
5             &  .8784 &  -     &  .9079 &  .9009 &  .9415 & .9335 \\
10            &  .8681 &  -     &  .9222 &  .9195 &  .9400 & .9318 \\
20            &  .8700 &  -     &  .9067 &  .9217 &  .9400 & .9297 \\
30            &  .8586 &  -     &  .8787 &  .8867 &  .9403 & .9328 \\
40            &  .8572 &  -     &  .8654 &  .8727 &  .9369 & .9318 \\
50            &  .8536 &  -     &  .8560 &  .8596 &  .9374 & .9299 \\
Max. Dimension &  .8500 &  -     &  .8500 &  .8635 &  .9391 & .9285 \\
N. Classes - 1 &  .8369 &  .9215 &  .8567 &  .6995 &  \textbf{.9420} & .9340 \\
\bottomrule
\end{tabular}
    }%
}%
\subfloat{%
    \resizebox{0.53\textwidth}{!}{%
            \begin{tabular}{lrrrrrr}
\toprule
{} &       PCA &       LDA &      ANMM &     DMLMJ &       NCA &      LMNN \\
\midrule
1             &  .8106 &  .8871 &  .8587 &  .6588 &  .9044 &  .8872 \\
2             &  .8261 &  -     &  .8850 &  .7173 &  \textbf{.9197} &  .8958 \\
3             &  .8543 &  -     &  .9090 &  .8807 &  .9152 &  .9068 \\
5             &  .8782 &  -     &  .9049 &  .8700 &  .9111 &  .9069 \\
10            &  .8826 &  -     &  .9198 &  .9044 &  .9153 &  .9113 \\
20            &  .8695 &  -     &  .9048 &  .8937 &  .9155 &  .9005 \\
30            &  .8502 &  -     &  .8675 &  .8851 &  .9154 &  .9027 \\
40            &  .8547 &  -     &  .8567 &  .8611 &  .9111 &  .9005 \\
50            &  .8655 &  -     &  .8633 &  .8569 &  .9133 &  .9070 \\
Max. Dimension &  .8654 &  -     &  .8654 &  .8525 &  .9154 &  .9092 \\
N. Classes - 1 &  .8106 &  .8871 &  .8587 &  .6588 &  .9044 &  .8872 \\
\bottomrule
\end{tabular}
    }%
}%
}
\caption{Results of dimensionality reduction experiments on \texttt{spambase} with 3-NN (train - test)} \label{results:dim:spambase}
\end{table}

\subsection{Analysis of Results}

\subsubsection{In-depth analysis}

Below we will describe the main details observed in the algorithms for the different experiments carried out.

\begin{itemize}

\item \textbf{\acs{NCA}.} In terms of the results obtained in the first experiment, we can clearly see that \acs{NCA} has obtained the best results. This is partly due to the fact that the algorithms have been evaluated with nearest neighbors classifiers, and that \acs{NCA} was specifically designed to improve this classifier. \acs{NCA} came first in most of the validations over the training set, showing its ability to fit to the data, but it has also obtained clear victories in many of the datasets over the test set, thus also demonstrating a great capacity for generalization. We have to note that \acs{NCA} (and also other algorithms such as \acs{LMNN}) commits substantial errors in datasets such as \emph{titanic} \cite{triguero2017keel}. This is a numerical-transformed dataset, but of a categorical nature, and with many repeated elements that may belong to different classes. This may be causing highly discriminative algorithms such as \acs{NCA} or \acs{LMNN} not being able to transform the dataset appropriately. This justifies how in certain situations other algorithms can be more useful than those that show better behavior in general \cite{wolpert1997no}.

\item \textbf{\acs{LMNN} and \acs{DMLMJ}.} We can also see that \acs{DMLMJ} and \acs{LMNN} algorithms stand out, although not as much as \acs{NCA}. These algorithms are also directed at nearest neighbor classification, which justifies these good results. \acs{LMNN} seems to have a slow convergence with the projected gradient method, and it could have achieved better results with a greater number of iterations. In fact, in the analysis of dimensionality reduction experiments we will observe that \acs{LMNN} performs much better with the stochastic gradient descent method.

\item \textbf{\acs{LSI}.} \acs{LSI} is another algorithm that is capable of obtaining very good results on certain datasets, but it is penalized by many others, where it is not able to optimize enough, not even being able to converge in several datasets.

\item \textbf{\acs{ITML} and \acs{MCML}.} \acs{ITML} and \acs{MCML} are two algorithms that, despite getting the best results in a very small number of cases, they get decent results in most datasets, resulting in quite a stable performance. \acs{ITML} does not learn too much from the characteristics of the training set, but is able to generalize what has been learned in quite an effective way, being possibly the algorithm that loses the least accuracy over the test set, with respect to the training set. On the other hand, \acs{MCML} has more learning capacity, even showing a slight overfitting, as its results are worse than those of many algorithms on the test set.

\item \textbf{\acs{LDA}.} Another algorithm in which we can see overfitting, perhaps more clearly, is \acs{LDA}. This algorithm is capable of getting very good results on the training set, surpassing most of the algorithms, but it gets noticeably worse when evaluated on the test dataset. Recall that \acs{LDA} is able to learn only a maximum dimension equal to the number of classes of the dataset minus one. This may be causing a loss of important information on many datasets by the projection it learns.

\item \textbf{\acs{DML-eig} and \acs{LDML}.} Finally, although \acs{DML-eig} and \acs{LDML} are able to get better results than Euclidean distance on the training sets, on several datasets they have obtained quite low quality results. On many of the test datasets, they are surpassed by the Euclidean distance.

\item \textbf{Untrained \acs{kNN}}. The untrained \acs{kNN} or, equivalently, the classical \acs{kNN} with Euclidean distance, is always outperformed by some of the distance-learned \acsp{kNN} in the training set, and is also mostly outperformed in the test set. This shows the benefits of learning a distance as opposed to the traditional use of the nearest neighbor classifier. The untrained \acs{kNN} also shows better average results on the test set than on the training set, and is the only one among all the compared algorithms. This may be due to the fact that, as it is not using a pretrained distance, it is unlikely to overfit, although according to the results there is a lot of room for improvement in both training and test sets for this basic version.

\item \textbf{\acs{NCMML} and \acs{NCMC}.} If we analyze the results of the centroid-based classifiers, we can easily observe that in the vast majority of cases the classifier has worked much better after learning the distance with its associated learning algorithm, than it has by using the Euclidean distance. It can also be observed that the results are subject to great variability, depending on the number of centroids chosen. This shows that the choice of an adequate number of centroids that adapt well to the disposition of the different classes is fundamental to achieve successful learning with these algorithms.

\item \textbf{Kernel algorithms.} Focusing now on the kernel-based algorithms, it is interesting to note how \acs{KLMNN} with laplacian kernel is able to adjust as much as possible to the data, getting a 100 \% success rate on most of the datasets. This success rate is not transferred, in general, to the test data, showing that this algorithm overfits with laplacian kernel. We can also observe that the best results are distributed in a varied way among the different evaluated options. The choice of a suitable kernel that fits well with the disposition of the data is decisive for the performance of kernel-based algorithms.

\item \textbf{Dimensionality reduction experiments.} To conclude our analysis, dimensionality experiments allow us to observe that the best results are not always obtained when considering the maximum dimension. This may be due to the fact that the algorithms are able to denoise the data, ensuring that the classifier used later does not overfit. We also see that we cannot reduce the dimension as much as we want, because at some point we start losing information, which happens in many cases with \acs{LDA}, which is its great limitation. In general, we can observe that all algorithms improve their results by reducing dimensionality until a certain value, although the best results are provided by \acs{LMNN}, \acs{DMLMJ} and \acs{NCA}. The results obtained by \acs{LMNN} open the possibility of using this algorithm with stochastic gradient descent, instead of the semidefinite programming algorithm used in the first experiment, since the results it provides are quite good. Although these algorithms have obtained better results, the use of \acs{ANMM} and \acs{LDA} (as long as the dimension allows it) is important for the estimation of an adequate dimension, since they are much more efficient than the first ones. As for \acs{PCA}, it gets the worst results in low dimensions, probably due to not considering the information of the labels.

\end{itemize}

\subsubsection{Global analysis}

In order to complete the verbal analysis, we have developed a series of Bayesian statistical tests to assess the extent to which the performance of the different algorithms analyzed outperforms the other algorithms. To do this, we have elaborated several pairwise Bayesian sign tests \cite{benavoli2017time}. In these tests, we will consider the differences between the obtained scores of two algorithms, assuming that their prior distribution is a Dirichlet Process \cite{benavoli2014bayesian}, defined by a prior strength $s = 1$ and a prior pseudo-observation $z_0 = 0$. After considering the score observations obtained for each dataset, we obtain a posterior distribution which gives us the probabilities that one algorithm outperforms the other. We also introduce a \emph{rope} (region of practically equivalent) region, in which we consider the algorithms to have equivalent performance. We have designated the rope region to be the one where the score differences are in the interval $[-0.01, 0.01]$. In summary, from the posterior distribution we obtain three probabilities: the probability that the first algorithm outperforms the second, the probability that the second algorithm outperforms the first one, and the probability that both algorithms are equivalent. These probabilities can be visualized in a simplex plot for a sample of the posterior distribution, in which a greater tendency of the points towards one of the regions will represent a greater probability.

To do the Bayesian sign tests, we have used the R package \texttt{rNPBST} \cite{carrasco2017rnpbst}. In Figure \ref{fig:bayes:basic} we pairwise compare some of the algorithms that seem to have better performance in experiment \ref{exp:normal} with 3-NN (\acs{NCA}, \acs{DMLMJ} and \acs{LMNN}) with the results of the 3-NN classifier for Euclidean distance. In the comparison made between Euclidean distance and \acs{NCA}, we can clearly see that the points are concentrated close to the $[\text{NCA}, \text{rope}]$ segment. This shows us that Euclidean distance is unlikely to outperform \acs{NCA}, and there is also a high probability for \acs{NCA} to outperform Euclidean distance, since a big concentration of points is in the \acs{NCA} region. We obtain similar conclusions for \acs{DMLMJ} against Euclidean distance, although in this case, despite the fact that Euclidean distance is still unlikely to win, there is a greater concentration of points in the rope region. In the comparison made between \acs{LMNN} and Euclidean distance, we see a more centered concentration of points, that is slightly weighted towards the \acs{LMNN} region. In the comparisons made between the \ac{DML} algorithms we observe the points weighted to the $[\text{NCA}, \text{rope}]$ segment, which concludes the difficulty of outperforming \acs{NCA}, and between \acs{DMLMJ} and \acs{LMNN} we can see a pretty level playing field that is slightly biased to the \acs{DMLMJ} algorithm.

The outperforming of Euclidean distance is even clearer in the results from experiment \ref{exp:ncm}. For these algorithms, we can clearly observe that the points are concentrated in the region corresponding to the nearest centroid metric learning algorithm, as shown in Figure \ref{fig:bayes:ncm}. We have elaborated more pairwise Bayesian sign tests for the rest of the algorithms in experiment \ref{exp:normal}. The results of these tests can also be found on the pyDML-Stats website\cref{pydml-stats}.

\begin{figure}
\centering
\begin{tabular}{ccc}
  {}                                                     \includegraphics[width=4cm]{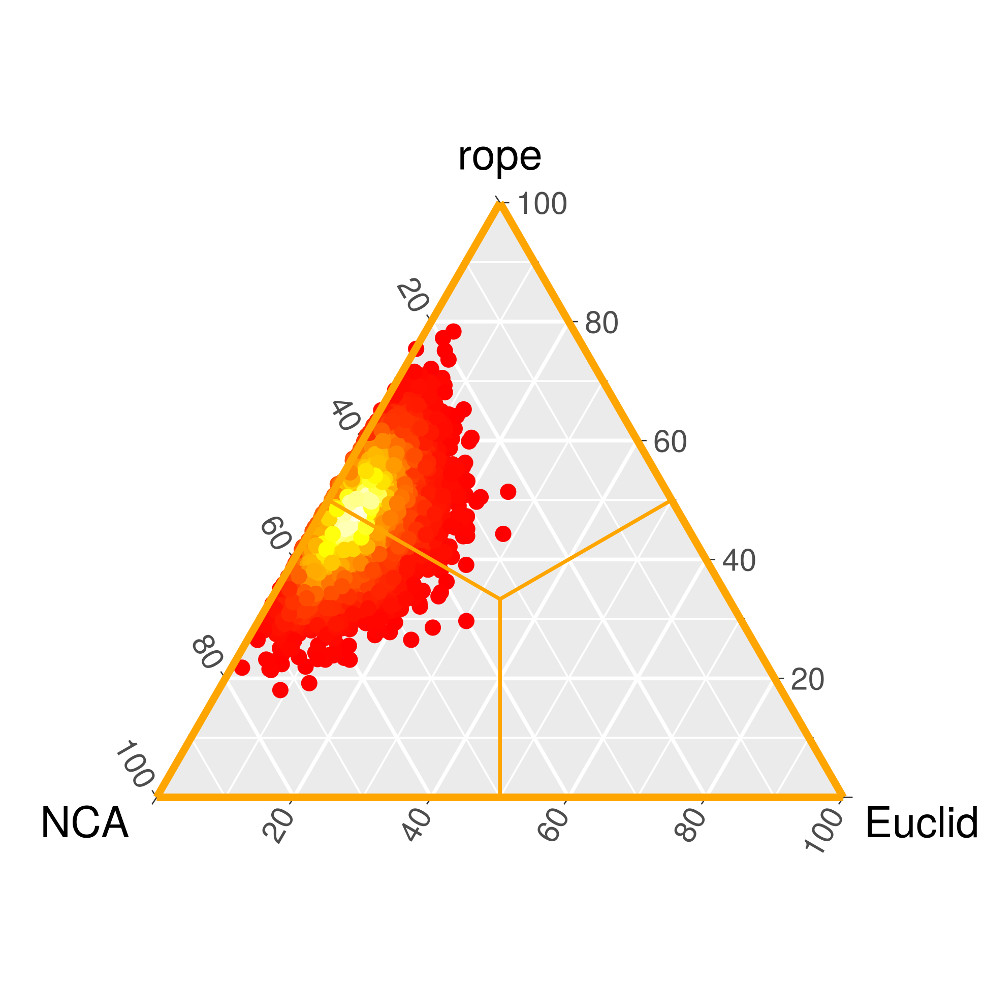} & \includegraphics[width=4cm]{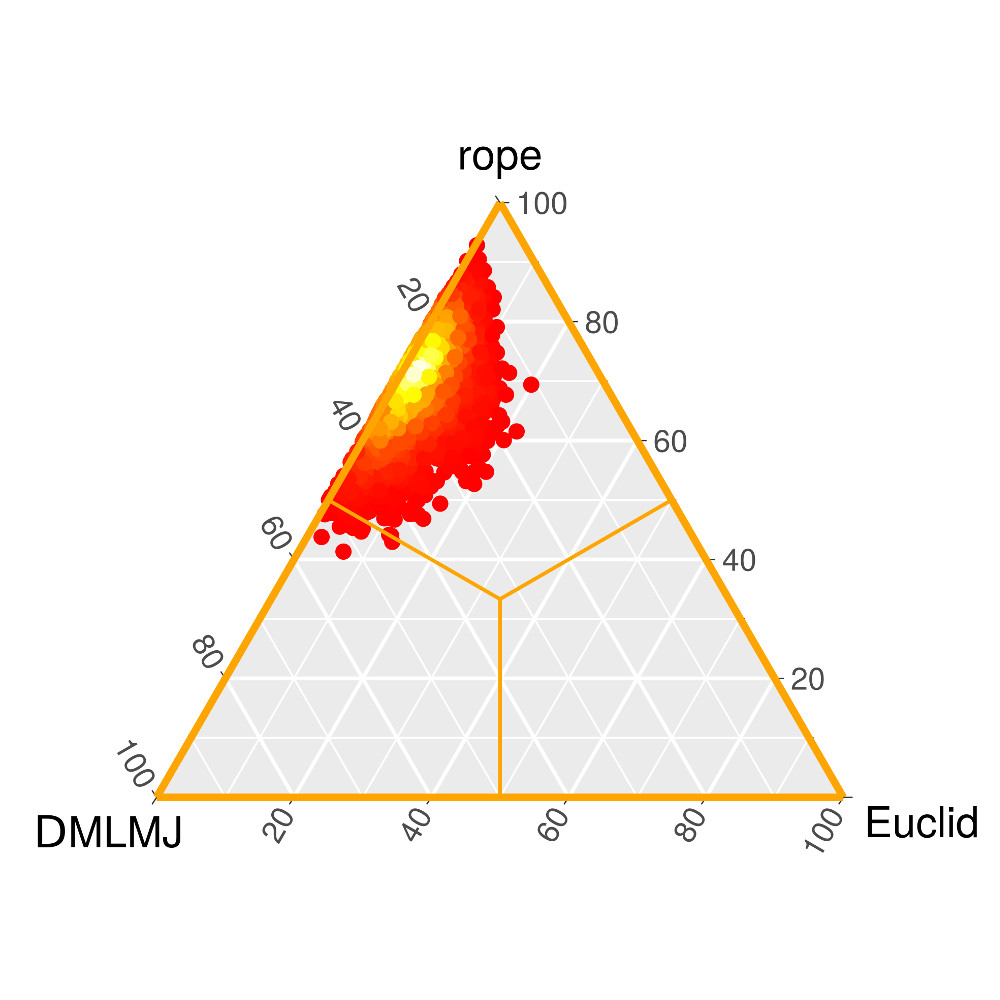} & \includegraphics[width=4cm]{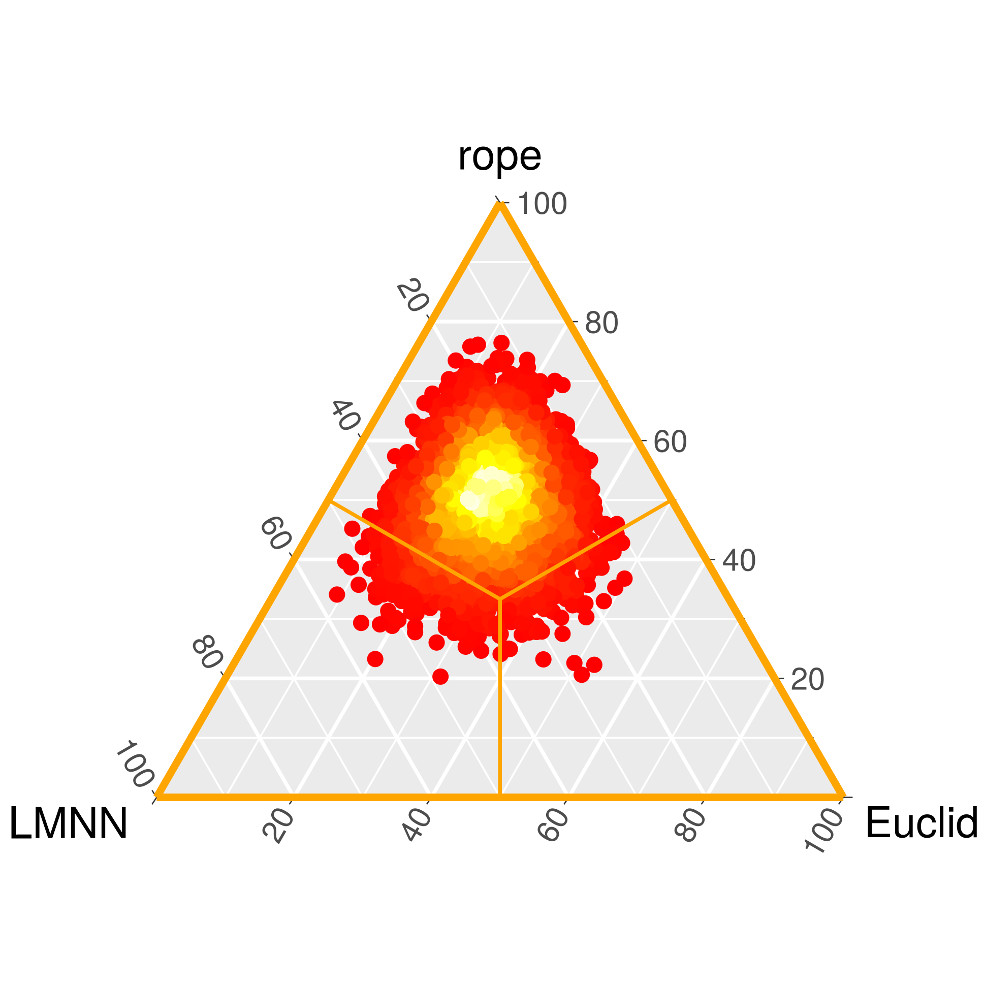} \\
  \includegraphics[width=4cm]{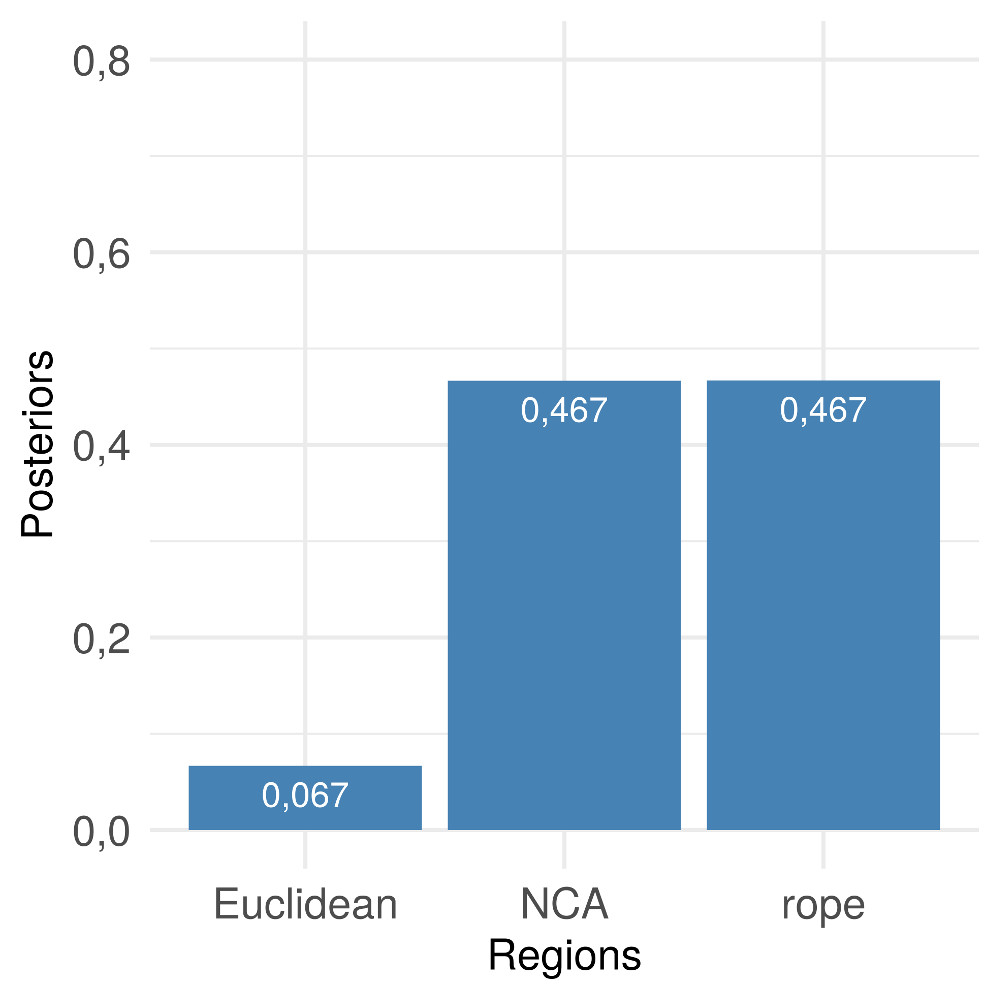}  {}                                                    & \includegraphics[width=4cm]{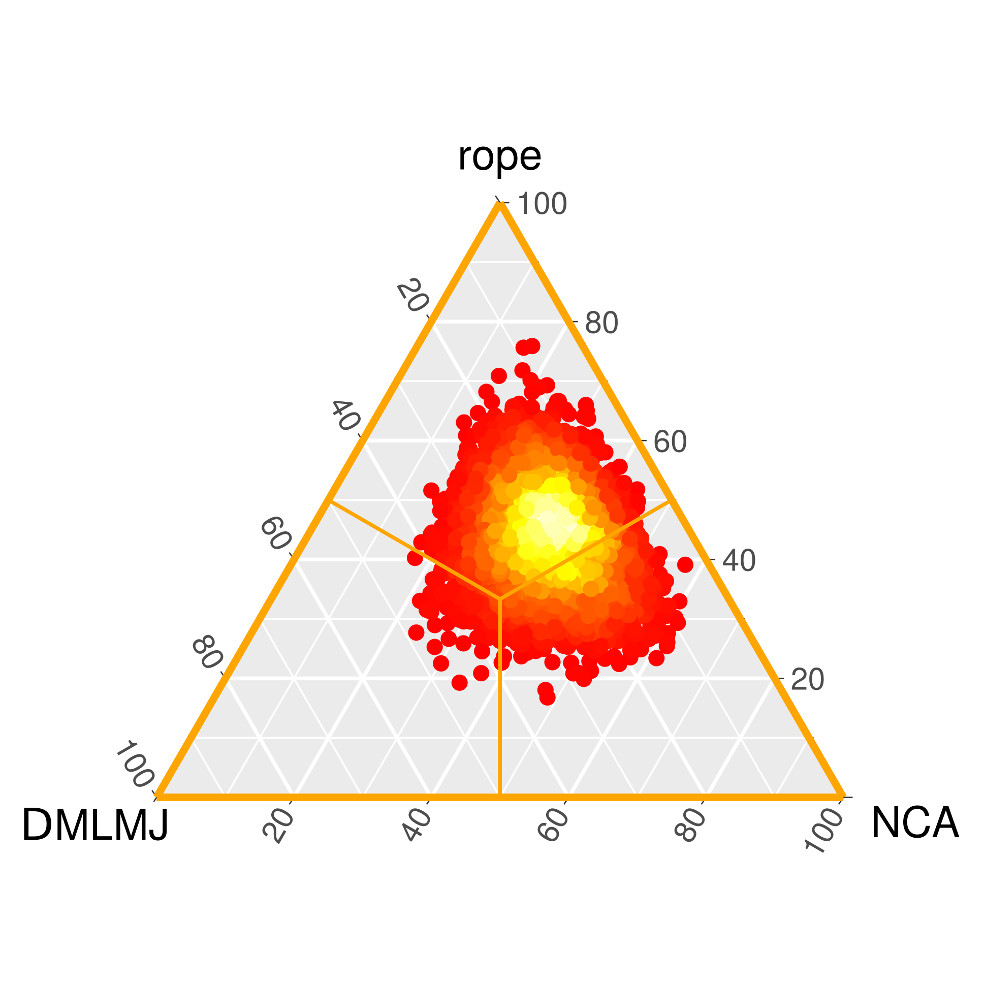} & \includegraphics[width=4cm]{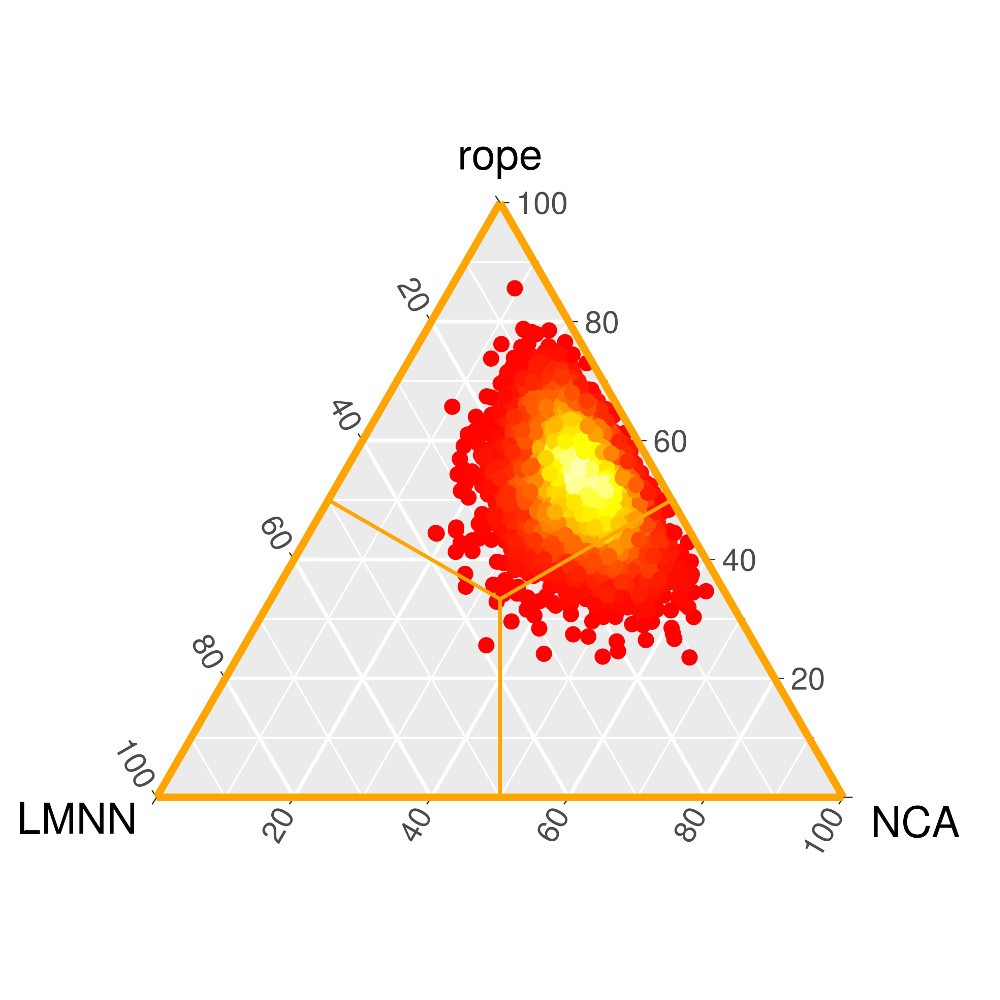} \\
  \includegraphics[width=4cm]{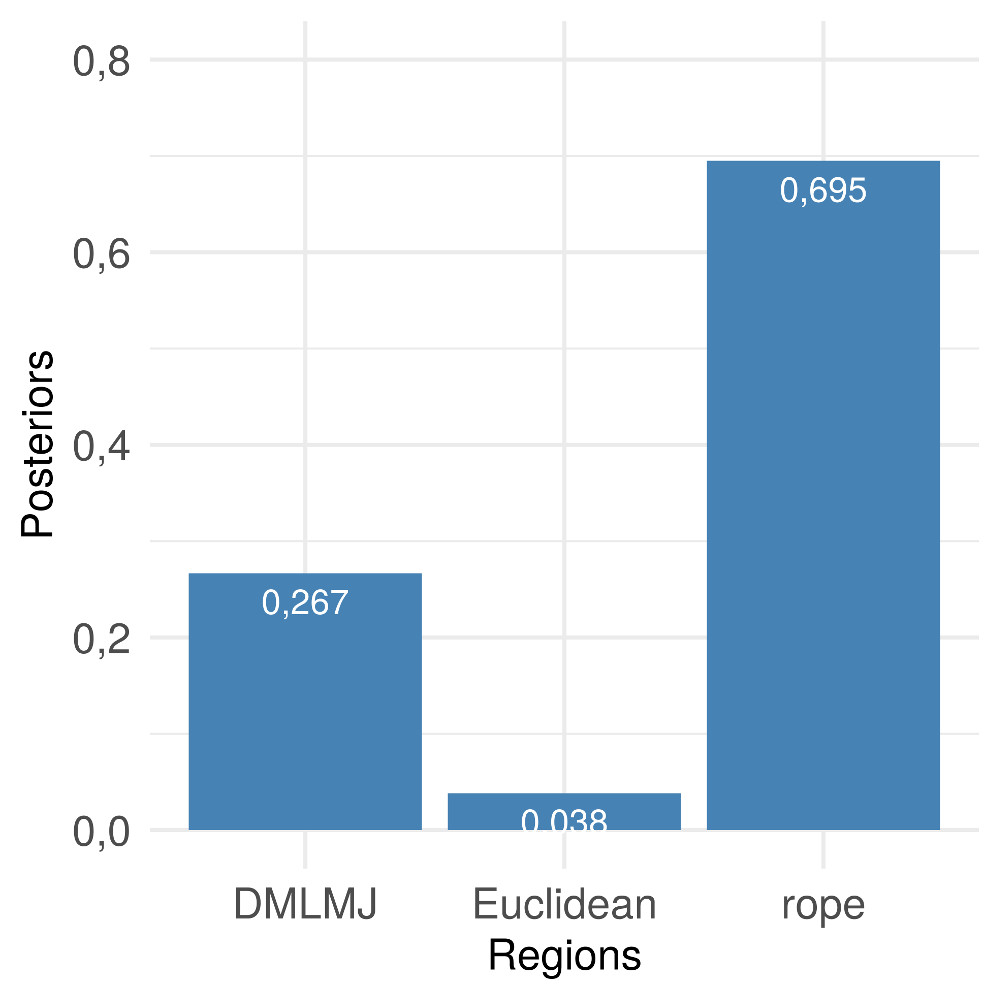} & \includegraphics[width=4cm]{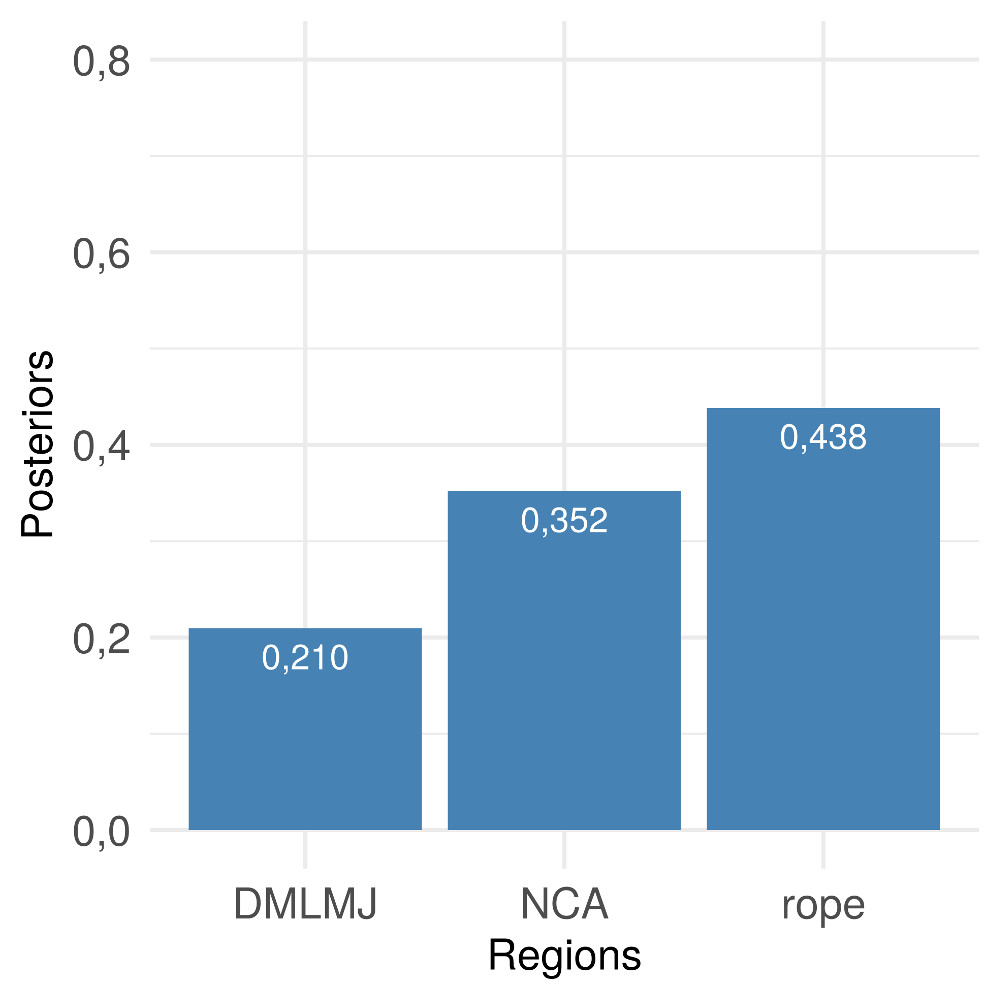} {}                                                    & \includegraphics[width=4cm]{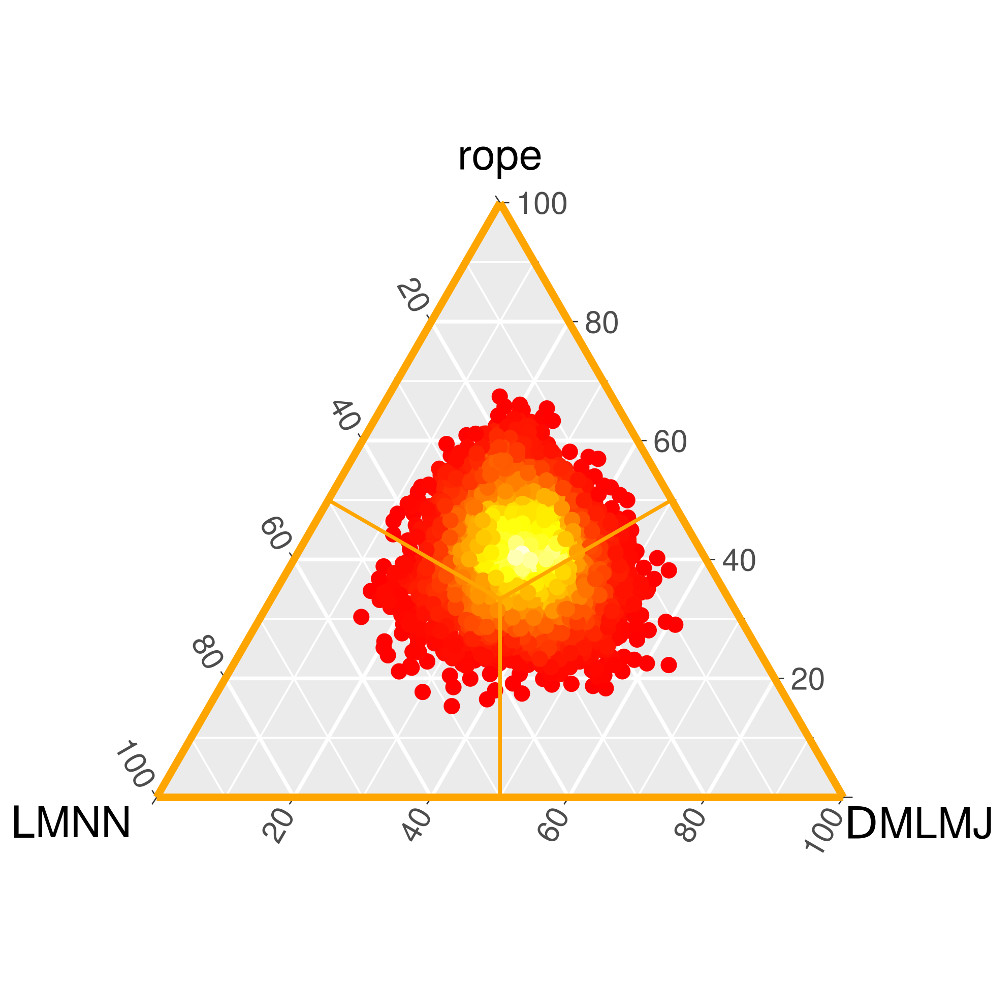}\\
  \includegraphics[width=4cm]{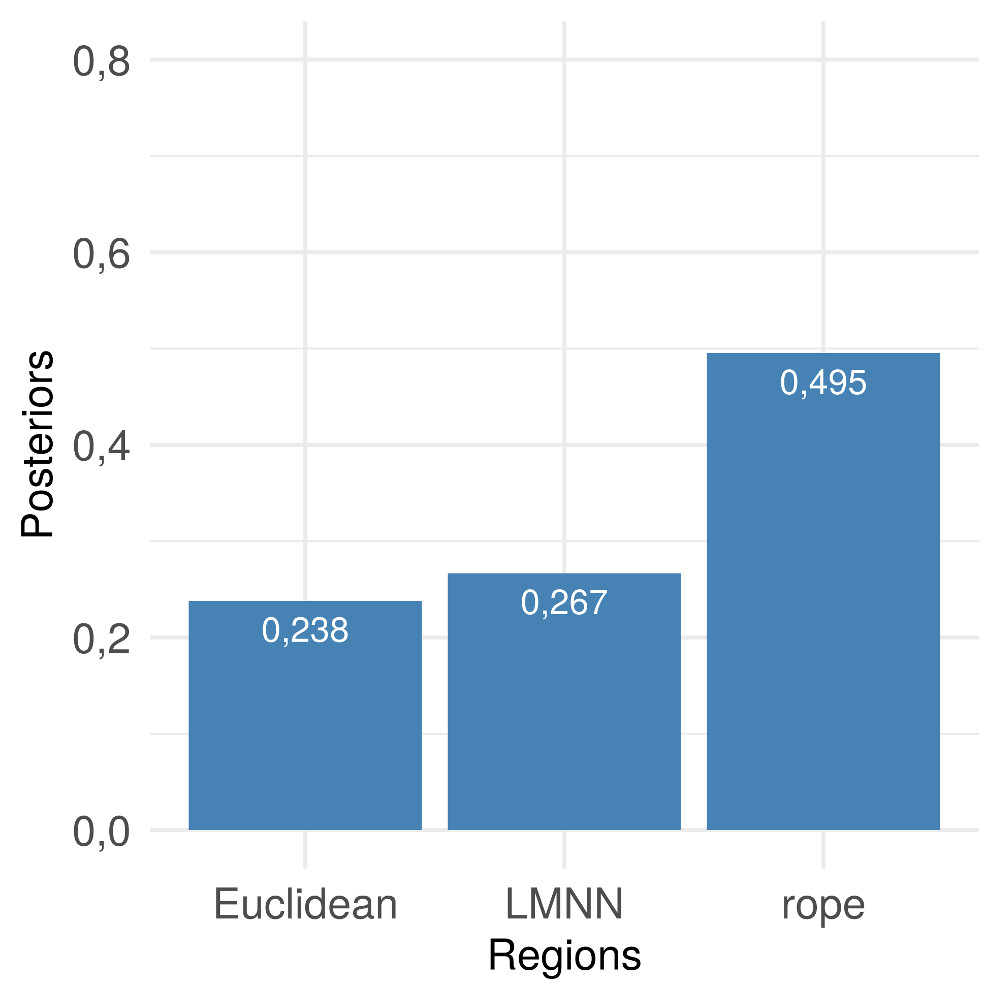} &\includegraphics[width=4cm]{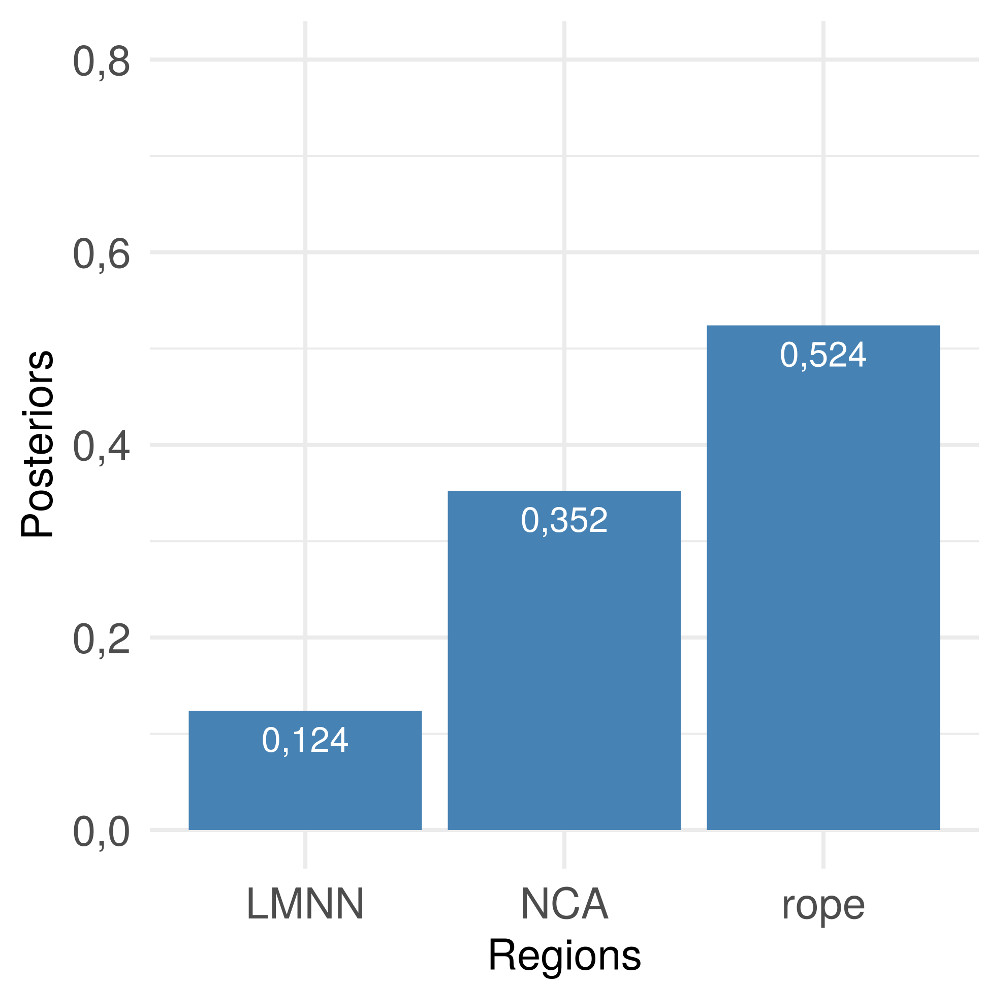}& \includegraphics[width=4cm]{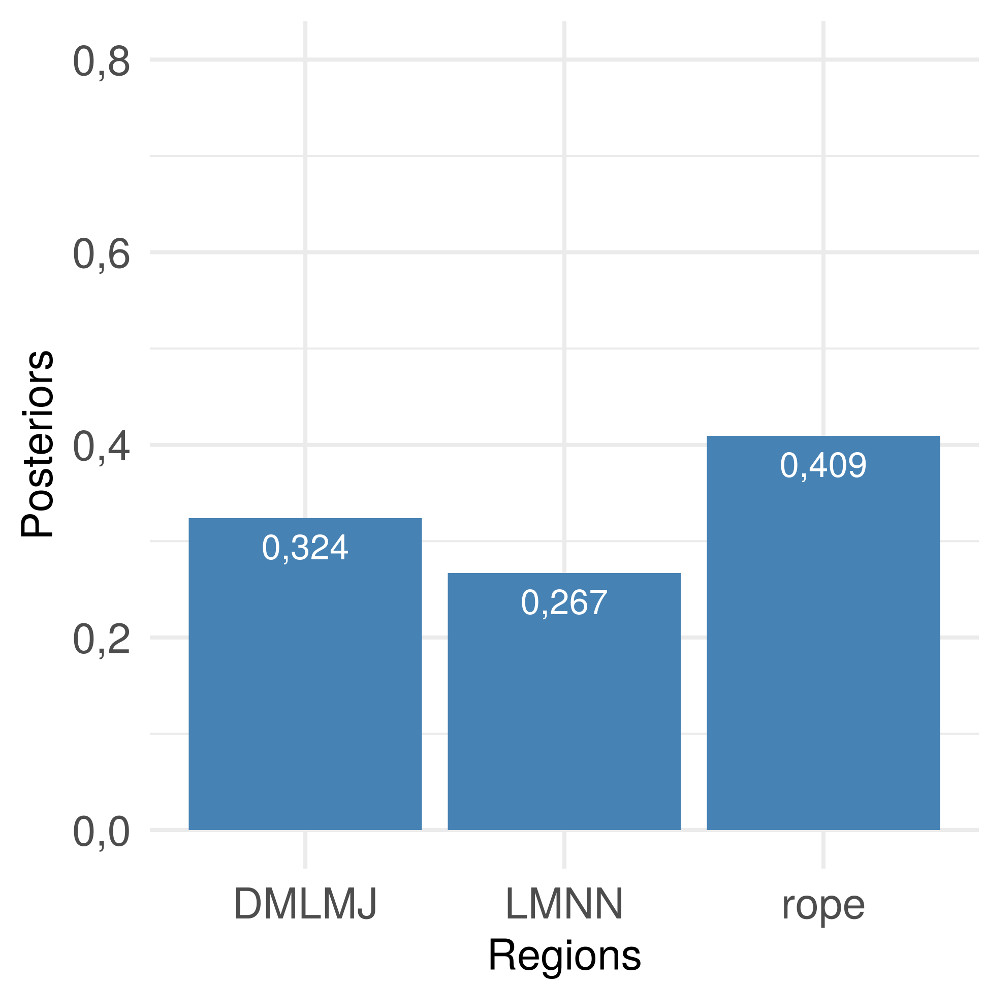} {}\\
\end{tabular}
\caption{Bayesian sign results for \acs{NCA}, \acs{DMLMJ}, \acs{LMNN} and Euclidean distance with 3-NN.} \label{fig:bayes:basic}
\end{figure}

\begin{figure}
\centering
\begin{tabular}{ccc}
  \includegraphics[width=4cm]{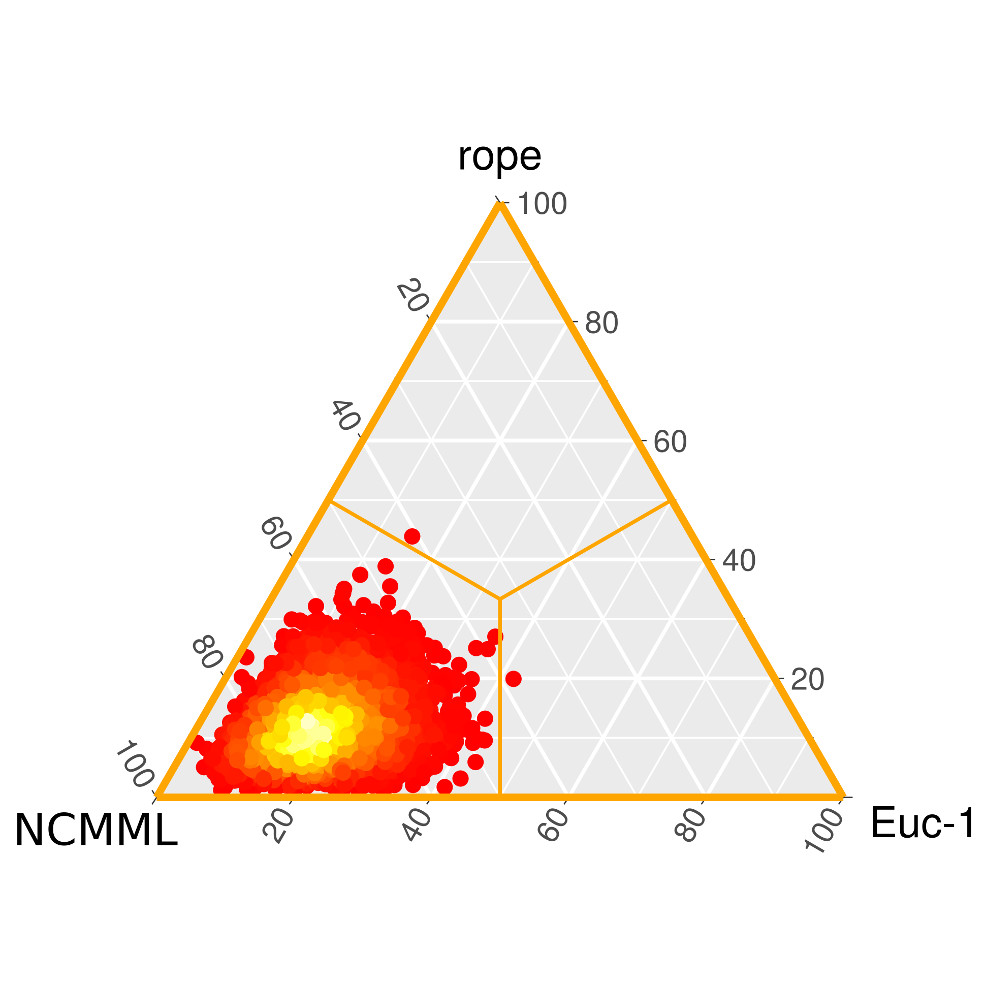} & \includegraphics[width=4cm]{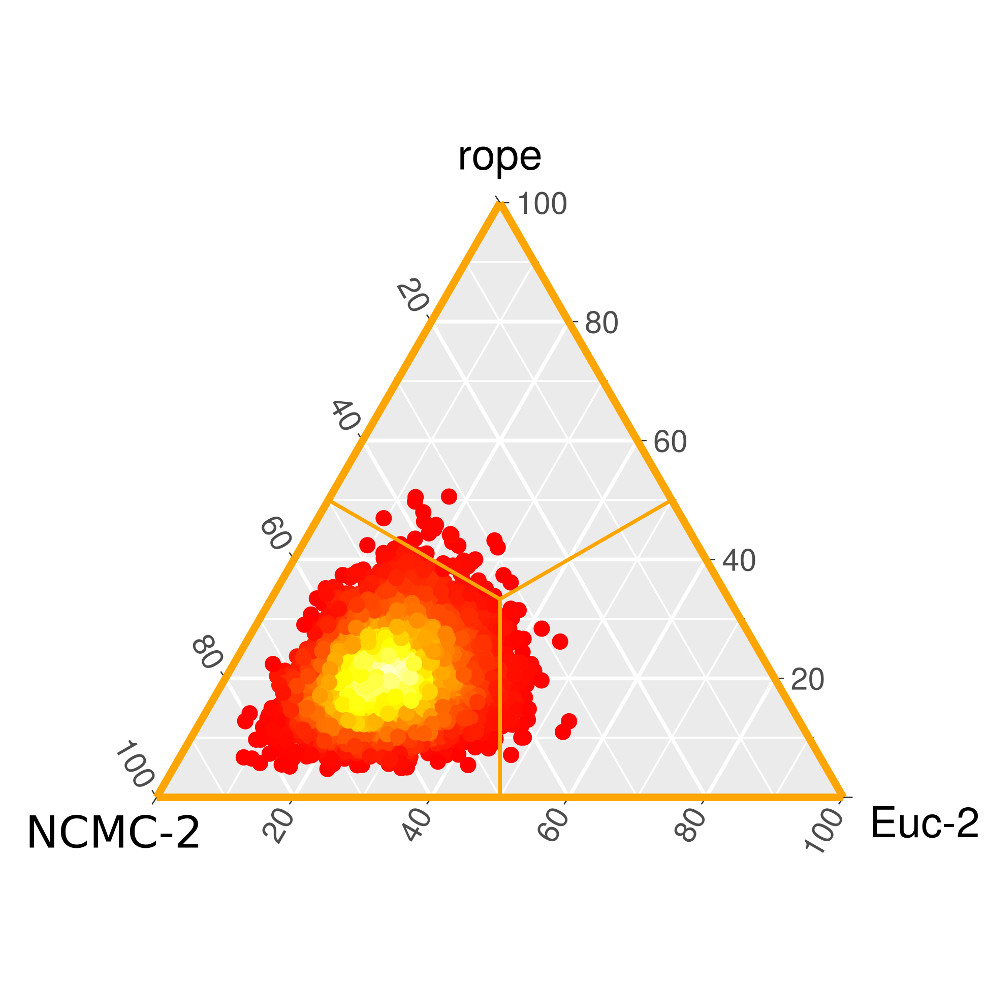} & \includegraphics[width=4cm]{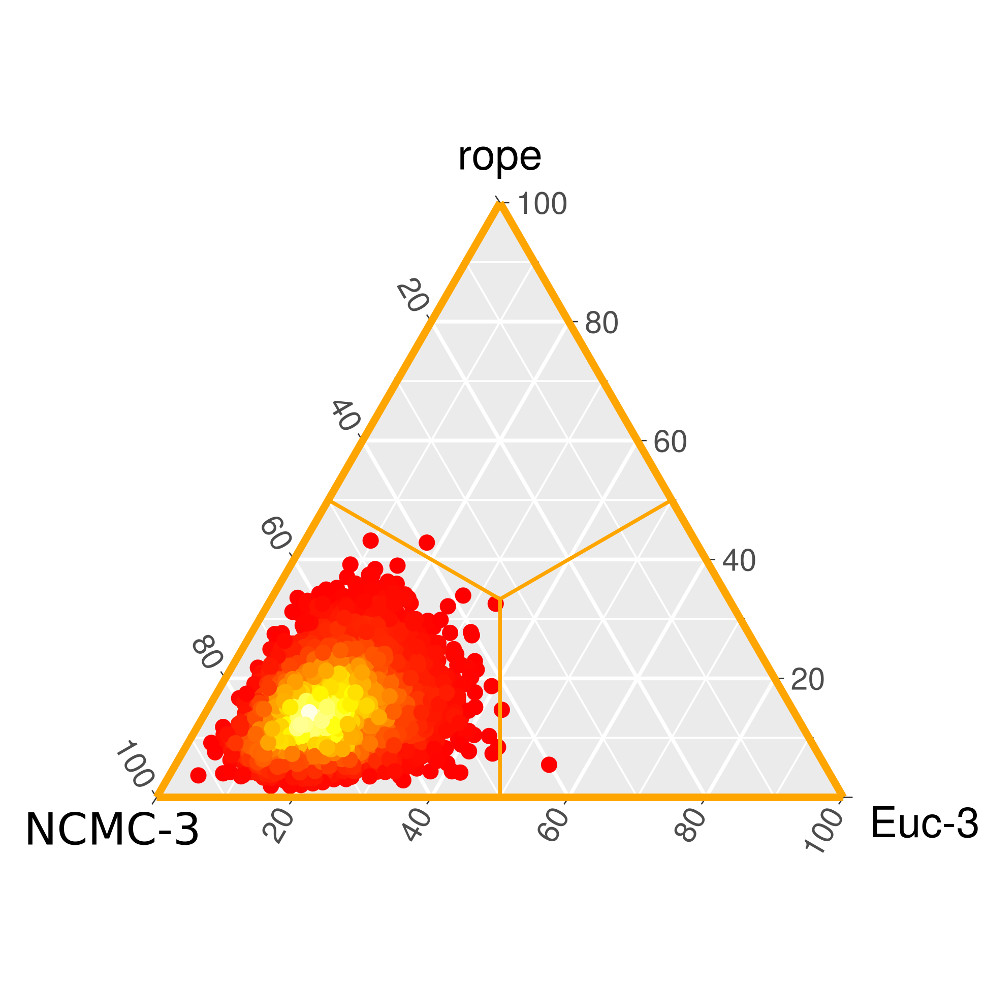} \\
  \includegraphics[width=4cm]{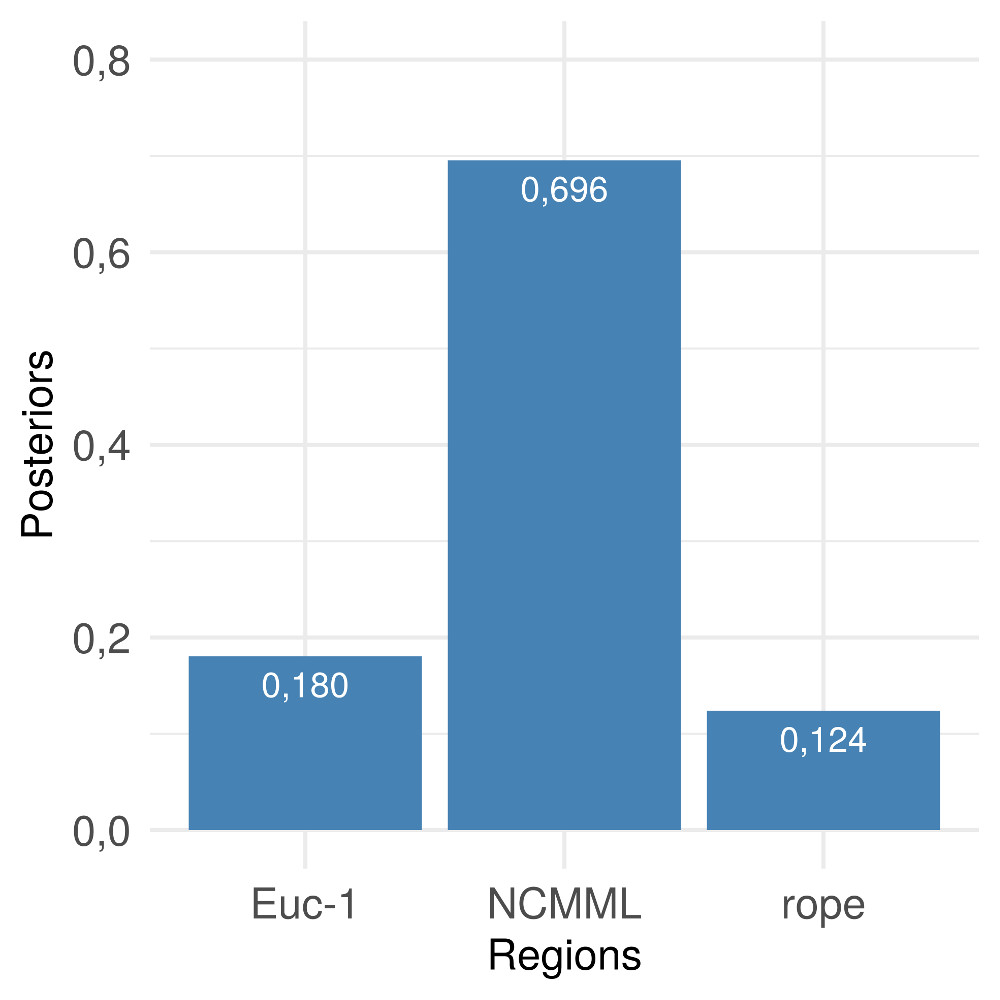} & \includegraphics[width=4cm]{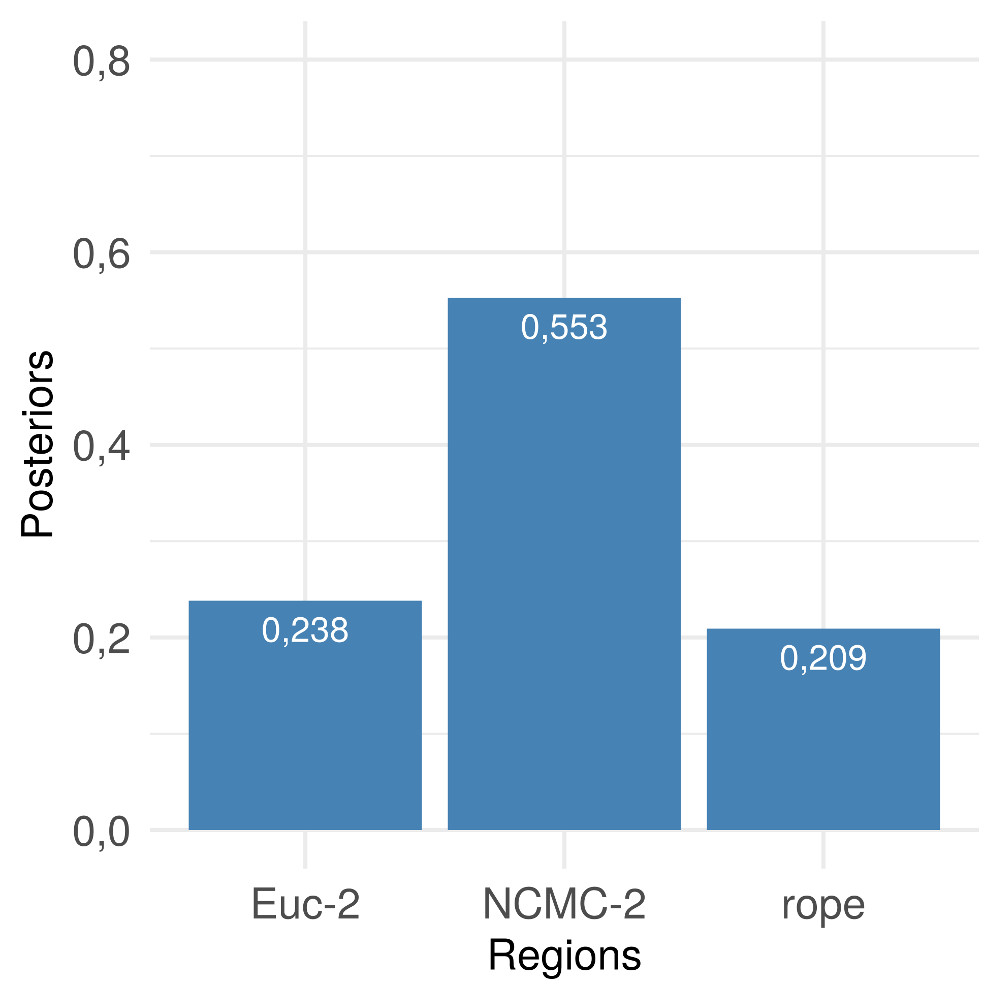} & \includegraphics[width=4cm]{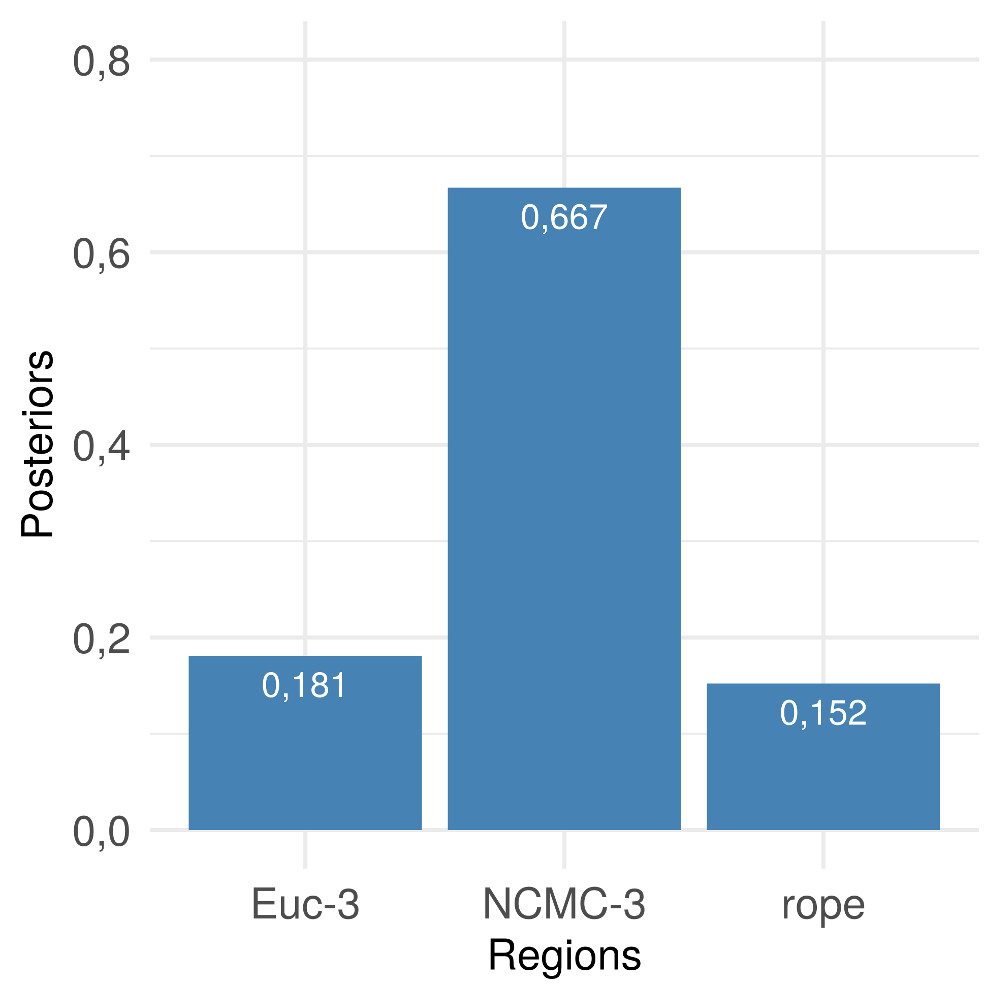} \\
\end{tabular}
\caption{Bayesian sign test results for the comparison between scores of nearest centroid classifiers with their corresponding \ac{DML} algorithm against the same classifier with Euclidean distance. The results are shown for nearest class mean classifier (left), nearest class with 2 centroids (center) and nearest class with 3 centroids (right).} \label{fig:bayes:ncm}
\end{figure}

\section{Prospects and Challenges in Distance Metric Learning} \label{sec:future}

Throughout this tutorial we have seen what \ac{DML} consists of and how it has traditionally been applied in machine learning. However, the development of technology in recent years has given rise to new problems that cannot be adequately addressed from the point of view of classic machine learning. In the same way, this technological development has led to new tools that are very useful when facing new problems, as well as allowing better results to be obtained with the more traditional problems.

Focusing on \ac{DML}, both the new problems and the new tools are generating new prospects where in which applying \ac{DML} could be of interest, and as well as generating new challenges in the design and application of \ac{DML}. Below we will describe some of the most outstanding ones.

\subsection{Prospects of Distance Metric Learning in Machine Learning}

Nowadays there are many fields where the further development of \ac{DML} might be of interest. On the one hand, the large volumes of data that are usually being handled today make it necessary to adapt or design new algorithms that can work properly with both high-dimensional data and huge amounts of examples. Similarly, new problems are arising, which make it necessary to reconsider the algorithms so that they can handle these problems in an appropriate way. On the other hand, many of the tools provided by machine learning, from the classical ones to the most modern ones, can be used in line with \ac{DML} to achieve better results. We outline these prospects below.

\begin{itemize}

    \item \textbf{Hybridization with feature selection techniques to solve high dimensional data problems.} \ac{DML} is of great interest in many real problems in high dimensionality, such as face recognition, where it is very useful to be able to measure the similarity between different images \cite{moutafis2017overview}. When we work with datasets of even greater dimensionality, the treatment of distances can become too expensive, since it would be necessary to store matrices of very large dimensions. In these situations, it may be of interest to combine \ac{DML} with feature selection techniques prepared for very high dimensional data \cite{tan2014towards,liu2019escaping}.

   \item \textbf{Big Data solutions.} The problem of learning when the amount of data we have is huge and heterogeneous is one of the challenges of machine learning nowadays \cite{wu2014data}. In the case of the \ac{DML} algorithms, although many of them, especially those based on gradient descent, are quite slow and do not scale well with the number of samples, they can be largely parallelized in both matrix computations and gradient descent batches. As a result, \ac{DML} can be extended to handle Big Data by developing specialized algorithms and integrating them with frameworks such as Spark \cite{meng2016mllib} and Cloud Computing architectures \cite{hashem2015rise}.

   \item \textbf{Application of distance metric learning to singular problems.} In this paper, we have focused on \ac{DML} for common problems, like standard classification and dimensionality reduction, and we have also mentioned its applications for clustering and semi-supervised learning. However, \ac{DML} can be useful in a wide variety of learning tasks \cite{charte2019nonstandard}, and can be carried out either by designing new algorithms or by adapting known algorithms from standard problems to these tasks. In recent years, several \ac{DML} proposals have been made in problems like regression \cite{nguyen2016large}, multi-dimensional classification \cite{ma2018multi}, ordinal classification \cite{nguyen2018distance}, multi-output learning \cite{Liu2019408} and even transfer learning \cite{Luo20191013,lopez2019visual}.

   \item \textbf{Hybridization with shallow learning techniques.} Over the years, some distance-based algorithms, or some of their ideas or foundations, have been combined with other algorithms in order to improve their learning capabilities in certain problems. For example, the concept of nearest-neighbors has been combined with classifiers such as Naive-Bayes, obtaining a Naive-Bayes classifier whose feature distributions are determined by the nearest neighbors of each class \cite{yang2012eigenjoints}; with neural networks, to find the best neural network architecture \cite{wang2017improving}; with random forests, by exploiting the relationship between voting points and potential nearest neighbors \cite{lin2006random}; with ensemble methods, like bootstrap \cite{steele2009exact,hamamoto1997bootstrap}; with support vector machines, training them locally in neighborhoods \cite{zhang2006svm}; or with rule-learning algorithms, obtaining the so-called \emph{nested generalized exemplar} algorithms \cite{wettschereck1995experimental}. The distances used in these combinations of algorithms can condition their performance, so designing appropriate distance learning algorithms for each of these tasks can help achieve good results. Staying on this subject, another option is to hybridize directly \ac{DML} with other techniques, like ensemble learning \cite{mu2013local}.

   \item \textbf{Hybridization with deep learning techniques.} In recent years, machine learning has experienced great popularity thanks to the development of deep learning, which is capable of obtaining very good results in different learning problems \cite{gomez2019towards}. As in the previous case, it is possible to combine distance-based algorithms or their foundations with deep learning techniques to improve their learning capabilities. For instance, \citet{papernot2018deep} use the $k$-nearest neighbors classifier to provide interpretability and robustness to deep neural networks. Another prospect that has gained popularity in recent years is based on the use of neural networks to learn distances, which is being referred to as \emph{deep metric learning} \cite{yi2014deep,zhe2019directional,cakir2019deep,cao2019hyperspectral,nguyen2020improved}. Deep learning is likely to play an important role in the future of machine learning, and thus its combination with \ac{DML} may lead to interesting advances in both fields.
   
   \item \textbf{Other approaches for the concept of distance.} Most of the current \ac{DML} theory focus on Mahalanobis distances. However, some articles open a door to learning about other possible distances, such as local Mahalanobis distances, that lead to a multi-metric learning \cite{lmnn}, or approaches beyond the Mahalanobis theory \cite{pan2020metric,shindo_metric_2020}. The \emph{deep metric learning} approach discussed above is another way of handling a wider range of distances. By developing new approaches, we will have a greater variety of distances to learn, and thus have a greater chance of success.
   
\end{itemize}

\subsection{Challenges in Distance Metric Learning}

In addition to the numerous action horizons, \ac{DML} presents several challenges in terms of the design of its algorithms, which can lead to substantial improvements. We describe these challenges below.

\begin{itemize}

   \item \textbf{Non-linear distance metric learning.} As we have already mentioned, since learning a Mahalanobis distance is equivalent to learning a linear map, there are many problems where these distances are not able to capture the inherent non-linearity of the data. Although the non-linearity of a subsequent learning algorithm, such as the nearest neighbors classifier, may mitigate this fact, that algorithm could benefit much more from a distance capable of capturing the non-linearity of the dataset. In this sense, we have already seen how the kernelization of \ac{DML} algorithms can be applied to fit non-linear data. Extending the kernel trick to other algorithms besides those presented, by searching for suitable parameterizations and representer theorems, is another possible task to carry out. Another possibility for non-linear \ac{DML} is to adapt classical objective functions so that they can work with non-linear distances, such as the $\chi^2$ histogram distance, or with non-linear transformations of the data learned by another algorithm, such as gradient boosting \cite{kedem2012non}.   


   \item \textbf{Multi-linear distance metric learning.} In learning problems where the data are images or videos, the traditional vector representation may not be the most appropriate to fit the data properly. Vector representation does not allow, for example, for the consideration of neighborhood relationships between pixels in an image. It is therefore better to consider images as matrices, or more generally, as multi-linear mappings or \emph{tensors}. Some \ac{DML} algorithms can be adapted so that they can learn distances in tensor spaces \cite{cai2005subspace,anmm,laiadi2020tensor}, which will be more suitable for similarity learning in datasets that support this representation. The development or extension of techniques for multi-linear \ac{DML} is a challenge that has many applications in a field such as computer vision, where \ac{DML} has been shown to be quite useful \cite{Nguyen2019589,Liang20191149,wang2019hybrid,ldml}.

   \item \textbf{Other optimization mechanisms.} The algorithms we have studied optimize their objective functions by applying gradient descent methods. However, the possibilities in terms of optimization mechanisms are very broad, and choosing the most appropriate method can contribute to achieving better values in the objective functions. In addition, the consideration of different optimization methods may lead to the design of new objective functions that may be appropriate for new problems or approaches and that cannot be optimized by classical gradient methods. In this way, we have studied several differentiable objective functions in this tutorial, unconstrained or with convex constraints, but for those non-convex functions the gradient descent methods (even the stochastic version) cannot guarantee a convergence to the global optimum. If we wanted to consider functions with even worse analytical characteristics or constraints, such as non-differentiability or integer constraints, we could not even use this type of method. For the non-convex and differentiable case, we are still able to use the information of the derivatives of the objective function, and some refinements of the classic gradient methods, such as AdaDelta, RMSprop or Adam have shown good performance in this type of problem \cite{sun2019survey}. In the most general case, we are only able to afford to evaluate the objective function, and sometimes not a very high number of times, due to its complexity. This general case is usually called \emph{black-box optimization}. To optimize these functions, a wide variety of proposals have been made. If we cannot afford to evaluate the objective function many times, Bayesian optimization may be an interesting alternative \cite{frazier2018tutorial}. If the objective function is not so complex, evolutionary algorithms can provide us with a great capability of exploration in the search space. Their repertoire is much broader and includes techniques such as \emph{simulated annealing}, \emph{particle swarm optimization} or \emph{response surface methods}, among others \cite{rios2013derivative}, thus many tools are available to address the most diverse optimization problems. These heuristics can also be used over differentiable optimization problems, and sometimes they can even outperform gradient methods, thanks to their greater ability to escape from local optima \cite{morse2016simple}. The evolutionary approach has been explored recently in several \ac{DML} problems \cite{kalintha2017kernelized,ali2018reinforcement}.


\end{itemize}

\section{Conclusions} \label{sec:conclusions}

In this tutorial we have studied the concept of distance metric learning, showing what it is, what its applications are, how to design its algorithms, and the theoretical foundations of this discipline. We have also studied some of the most popular algorithms in this field and their theoretical foundations, and explained different resolution techniques.

In order to understand the theoretical foundations of distance metric learning and its algorithms, it was necessary to delve into three different mathematical theories: convex analysis, matrix analysis and information theory. Convex analysis made it possible to present many of the optimization problems studied in the algorithms, along with some methods for solving them. Matrix analysis provided many useful tools to help understand this discipline, from how to parameterize Mahalanobis distances, to the optimization with eigenvectors, going through the most basic algorithm of semidefinite programming. Finally, information theory has motivated several of the algorithms we have studied.

In addition, several experiments have been developed that have allowed for the evaluation of the performance of the algorithms analyzed in this study. The results of these experiments have allowed us to observe how algorithms such as \acs{LMNN}, \acs{DMLMJ}, and especially \acs{NCA} can considerably improve the nearest neighbors classification, and how centroid-based distance learning algorithms also improve their corresponding classifiers. We have also seen the wide variety of possibilities offered by kernel-based algorithms, and the advantages that an appropriate reduction of the dimensionality of the datasets can offer.






\appendix

\section{Glossary of terms} \label{app:glossary}

\begin{acronym}
    \acro{ANMM}[ANMM]{Average Neighborhood Margin Maximization}
    \acro{DML}[DML]{Distance Metric Learning}
    \acro{DML-eig}[DML-eig]{Distance Metric Learning with eigenvalue optimization}
    \acro{DMLMJ}[DMLMJ]{Distance Metric Learning through the Maximization of the Jeffrey divergence}
    \acro{ITML}[ITML]{Information Theoretic Metric Learning}
    \acro{KANMM}[KANMM]{Kernel Average Neighborhood Margin Maximization}
    \acro{KDA}[KDA]{Kernel Discriminant Analysis}
    \acro{KDMLMJ}[KDMLMJ]{Kernel Distance Metric Learning through the Maximization of the Jeffrey divergence}
    \acro{KLMNN}[KLMNN]{Kernel Large Margin Nearest Neighbors}
    \acro{kNN}[$k$-NN]{$k$-Nearest Neighbors}
    \acro{LDA}[LDA]{Linear Discriminant Analysis}
    \acro{LDML}[LDML]{Logistic Discriminant Metric Learning}
    \acro{LMNN}[LMNN]{Large Margin Nearest Neighbors}
    \acro{LSI}[LSI]{Learning with Side Information}
    \acro{MCML}[MCML]{Maximally Collapsing Metric Learning}
    \acro{MMC}[MMC]{Mahalanobis Metric for Clustering}
    \acro{NCA}[NCA]{Neighborhood Components Analysis}
    \acro{NCM}[NCM]{Nearest Class Mean}
    \acro{NCMC}[NCMC]{Nearest Class with Multiple Centroids}
    \acro{NCMML}[NCMML]{Nearest Class Mean Metric Learning}
    \acro{PCA}[PCA]{Principal Components Analysis}
    \acro{SVM}[SVM]{Support Vector Machines}
    
\end{acronym}


\section{Mathematical Background} \label{app:math}

  In this appendix we will study three mathematical blocks that make up the foundations of distance metric learning: convex analysis, matrix analysis and information theory.

\subsection{Convex Analysis}

  Convex analysis is a fundamental field of study for many optimization problems. This field studies the convex sets, functions and problems. Convex functions have very useful properties in optimization tasks, and allow tools to be built to solve numerous types of convex optimization problems.

  We will highlight some results of convex analysis in our work. First, we will show some important geometric properties of convex sets, such as the convex projection theorem, and then we will analyze some optimization methods that will be used later.

  We start with the geometry of convex sets. We will work in the euclidean $d$-dimensional space, $\R^d$, where we note the dot product as $\langle \cdot, \cdot \rangle$.

  \subsubsection{Convex Set Results}

  Recall that convex sets are those for which any segment between two points in the set remains within the set, that is, a set $K \subset \R^d$ is convex iff $[x,y] = \{(1-\lambda) x + \lambda y \colon \lambda \in [0,1] \} \subset K$, for every $x, y \in K$. An important result from convex sets states that, at every point on the border of a closed set, we can setup a hyperplane so that the convex set and the hyperplane intersect only at the boundary of the set, and the whole set lies on one side of the hyperplane. Furthermore, this property characterizes the closed convex sets with non empty interiors. This result is known as the supporting hyperplane theorem and we discuss it below.

  \begin{definition}

  Let $T \colon \R^d \to \R$ be a linear map, $\alpha \in \R$ and $P = \{ x \in \R^d \colon T(x) = \alpha \}$ be an hyperplane. Associated with $P$, we define $P^+ = \{ x \in  \R^d \colon T(x) \ge \alpha\}$ and $P^- = \{ x \in \R^d \colon T(x) \le \alpha\}$.

  We say that $P$ is a \emph{supporting hyperplane} for the set $K \subset \R^d$ if $P \cap \closure{K} \ne \emptyset$, and either $K  \subset P^+$ or $K \subset P^-$. We refer to \emph{supporting half-space} as the half-space that contains $K$, between $P^+$ and $P^-$.

  \end{definition}

  \begin{theorem}[Supporting hyperplane theorem]~ \label{thm:support_hyperplane}
    \begin{enumerate}
      \item If $K \subset \R^d$ is a closed convex set, then for each $x_0 \in \fr K$ there is a supporting hyperplane $P$ for $K$ so that $x_0 \in P$.
      \item Every proper closed convex set in $\R^d$ is the intersection of all its supporting half-spaces.
      \item Let $K \subset \R^d$ be a closed set with non empty interior. Then, $K$ is convex if and only if for every $x \in \fr K$ there is a supporting hyperplane $P$ for $K$ with $x \in P$.
    \end{enumerate}
  \end{theorem}

  Proof of this result can be found in \citet{variations_convex} (chap. 2, theorem 2.7). We will use this theorem in the following results. The following property is fundamental to be able to make sense of the optimization tools shown in this paper. We will see that, given a closed convex set and a point in $\R^d$, we can find a nearest point to the given point in the convex set, and it is unique, that is, there is a projection for the given point onto the convex set. In other words, projections onto convex sets are well defined. We prove this result below. We will see that projections will help us to deal with constrained convex problems.

  \begin{theorem}[Convex projection] \label{thm:convex_projection}
    Let $K \subset \R^d$ be a non empty closed convex set. Then, for every $x \in \R^d$ there is a single point $x_0 \in K$ with $d(x,K) = d(x,x_0)$, where we have defined the distance to the set $K$ by
    \[ d(x,K) = \inf \{d(x,y) \colon y \in K\}.\]
    The point $x_0$ is called the \emph{projection} of $x$ onto $K$ and it is usually denoted by $P_K(x)$. The function $P_K \colon \R^d \to K$ given by the mapping $x \mapsto P_K(x)$ is therefore well defined and it is called the projection onto $K$.
    In addition, for each $x \in \R^d \setminus K$, the half-space $\{ y \in \R^d \colon \langle x - P_K(x), y - P_K(x) \rangle \le 0\}$ is a supporting half-space for $K$ in $P_K(x)$.
  \end{theorem}

  \begin{proof}
    First, we will prove the existence of a point in $K$ in which the distance to $K$ is achieved. In fact, this is true for every closed and not necessarily convex set. Let $x \in \R^d$. As $K$ is closed, we can choose $R > 0$ so that $K \cap \closure{B}(x,R)$ is a compact and non empty set. We consider the distance to $x$ in this set, that is, we define the map $d_x \colon K \cap \closure{B}(x,R) \to \R^+_0$ by $d_x(y) = d(x,y) = \|x-y\|$. $d_x$ is continuous and it is defined over a compact set, so it  attains a minimum at a point $x_0 \in K \cap \closure{B}(x,R)$.

    If we now take $y \in K \cap \closure{B}(x,R)$, we get $d(x,y) = d_x(y) \ge d_x(x_0) = d(x,x_0)$. On the other hand, if we take $y \in K \setminus \closure{B}(x,R)$, we get $d(x,y) > r \ge d(x,x_0)$. We have obtained that $d(x,y) \ge d(x,x_0)$ for every $y \in K$, and therefore $d(x,K) \ge d(x,x_0)$. The remaining inequality is clear, since $x_0 \in K$, that is, $x_0$ is the point we were looking for.

    We will see now the uniqueness of the point found. Suppose that $x_1,x_2 \in K$ verify that $d(x,x_1) = d(x,K) = d(x,x_2)$. We define $x_0$ as the half point in the segment $[x_1,x_2]$. We have that $x_0 \in K$, since $K$ is convex. Let us note that
    \[\langle x_1 - x_2, x - x_0 \rangle = \langle x_1 - x_2, x - \frac{1}{2}(x_1 + x_2) \rangle = \frac{1}{2}\langle x_1 - x_2, 2x - x_1 - x_2 \rangle.\]
    If we substitute $x_1 - x_2 = (x - x_2) - (x - x_1)$ and $2x - x_1 - x_2 = (x-x_2)+(x-x_1)$, we obtain
    \begin{align*}
        \langle x_1 - x_2, x - x_0 \rangle &= \frac{1}{2} \langle (x - x_2) - (x-x_1), (x - x_2)+(x - x_!) \rangle\\
                                           &= \frac{1}{2}(\|x - x_2\|^2 - \|x - x_1\|^2) \\
                                           &= \frac{1}{2}(d(x,K)^2 - d(x,K)^2) = 0.
    \end{align*}
    Therefore, the vectors $x_1 - x_2$ and $x - x_0$ are orthogonal, and consequently so are $x - x_0$ y $x_0 - x_2 = (x_1 - x_2)/2$. Applying Pythagorean theorem we have
    \[d(x,K)^2 = \|x - x_2\|^2 =  \|x - x_0\|^2 + \|x_0 - x_2\|^2 \ge \|x - x_0\|^2 \ge d(x,K)^2,\]
    that is, the equality holds in the previous inequality. In particular, we obtain that $\|x_0 - x_2\|^2 = 0$, and then $x_0 = x_2$. Since $x_0$ was the half point of $[x_1,x_2]$ we conclude that $x_1 = x_2$, proving the uniqueness.

    Finally we will prove the last assertion in the theorem. Let $x \in \R^d \setminus K$ and suppose that there exists $y \in K$ with $\langle x - P_K(x), y - P_K(x) \rangle > 0$. Since $K$ is convex, the segment $[y,P_K(x)]$ is contained in $K$, and therefore we have $y_t = P_K(x) + t(y - P_K(x)) \in K$, for every $t \in [0,1]$. We define the map $f \colon [0,1] \to \R$ by
    \begin{align*}
      f(t) &= \|y_t - x\|^2 = \|P_K(x) -x  + t(y - P_K(x))\|^2 \\
           &= \|P_K(x) - x \|^2 + 2t\langle P_K(x)-x,y-P_K(x) \rangle + t^2\|y - P_K(x)\|^2. 
    \end{align*}
    $f$ is a polynomial in $t$, so it is differentiable, and
    \[f'(0) = 2\langle P_K(x)-x,y-P_K(x) \rangle = -2 \langle x - P_K(x), y - P_K(x) \rangle < 0.\]
    Last expression implies that $f$ is strictly decreasing in a neighborhood of 0, that is, there exists $\varepsilon > 0$ so that $\|y_t - x\|^2 < \|y_0 - x\|^2 = \|P_K(x) - x\|^2$, for $0 < t < \varepsilon$, which results in a contradiction, since $P_K(x)$ minimizes the distance to $x$ in $K$ and the points $y_t$ lie on $K$.

  \end{proof}

  \subsubsection{Optimization Methods} \label{sssec:opt_methods}

  In the following paragraphs we will discuss some of the optimization methods that we will use in distance metric learning algorithms. These algorithms will generally try to optimize (we will focus on minimizing without loss of generality) differentiable functions without constraints, or convex functions subject to convex constraints. For the first case, it is well known that the gradient of a differentiable function has the direction of the maximum slope in the function graph, thus by advancing small quantities in the negative gradient direction we manage to reduce the value of our objective function. This iterative method is usually called the gradient descent method. The adaptation rule for this method, for a differentiable function $f \colon \R^d \to \R$, is given by $x_{t+1} = x_t - \eta \nabla f(x_t)$, $t \in \N \cup \{0\}$, where $\eta$ is the quantity we advance in the negative gradient direction, and it is called the \emph{learning rate}. This value can be either constant or adapted according to the evaluations of the objective function. For the first option, the choice of a value of $\eta$ that is too big or too small can lead to poor results. The second option needs to evaluate the objective function at each iteration, which can be computationally expensive.

  Foundations of gradient descent are based on the following ideas. Let us consider an objective function $f \colon \R^d \to \R$, $x \in \R^d$ and $v \in \R^d \setminus \{0\}$ an arbitrary direction. We also consider the function $g \colon \R \to \R$ given by $g(\eta) = f(x + \eta v/\|v\|)$. The rate of change or directional derivative of $f$ at $x$ in the direction of $v$ is given by $g'(0) = \langle \nabla f(x), v \rangle/\|v\|$. Applying Cauchy-Schwarz inequality, we have
  \[ -\|\nabla f(x)\| \le \frac{1}{\|v\|}\langle \nabla f(x), v \rangle \le \|\nabla f(x)\|, \]
  and equality in the left inequality holds when $v = - \nabla f(x)$, thus obtaining the maximum descent rate. In the same way, the maximum ascent rate is achieved when $v = \nabla f(x)$.

  If gradient at $x$ is non zero and we consider the first order Taylor approximation with the points $x$ and $x - \eta\nabla f(x)$, we have that
  \[ f(x - \eta\nabla f(x)) = f(x) - \eta \|\nabla f(x)\|^2 + o(\eta),\]
  with $\lim_{\eta \to 0} |o(\eta)|/\eta = 0$, then there is $\varepsilon > 0$ so that if $0 < \delta < \varepsilon$, we have
  \[  \frac{o(\delta)}{\delta} < \|\nabla f(x)\|,\]
  and therefore
  \[f(x - \delta\nabla f(x)) - f(x) = \delta\left(-\|\nabla f(x)\|^2+ \frac{o(\delta)}{\delta}\right) < \delta(-\|\nabla f(x)\|^2+\|\nabla f(x)\|^2) = 0, \]
  thus $f(x - \delta\nabla f(x)) < f(x)$ for $0 < \delta < \varepsilon$, so we are guaranteed that for an accurate learning rate the gradient method performs a descent at each iteration. Let us observe that the gradient direction is not the only valid descent direction, but the above calculations are still true for any direction $v \in \R^d$ with $\langle \nabla f(x), v \rangle < 0$. The choice of different descent directions, even if they are not the maximum slope direction, may provide better results in certain situations.

  Now we will discuss the constrained convex optimization problems. When we work with constrained problems, gradient descent method cannot be applied directly, as the gradient descent adaptation rule, $x_{t+1} = x_t - \eta \nabla f(x_t)$, does not guarantee $x_{t+1}$ to be a feasible point, that is, a point that fulfills all the constraints. When the optimization problem is convex, the set determined by the constraints is closed and convex, so we can take projections onto this feasible set. The projected gradient method tries to fix the gradient descent problem by adding a projection onto the feasible set in the gradient descent adaptation rule, that is, if $C$ is the feasible set, and $P_C$ is the projection onto this set, the projected gradient adaptation rule becomes $x_{t+1} = P_C(x_t - \eta \nabla f(x_t))$. To confirm that this method is successful, we have to show that the direction $v = P_C(x - \eta\nabla f(x)) - x$ is a descent direction, which is attained, thanks to the reasons given above, if $\langle \nabla f(x), v \rangle < 0$.

  We name $x_1 = x - \eta \nabla f(x)$. Then, $v = P_C(x_1) - x$. Note that $\langle \nabla f(x), v \rangle < 0 \iff \langle x_1 - x, P_C(x_1) - x \rangle = -\eta \langle \nabla f(x), v \rangle > 0$. If gradient is not null and $x_1 \in C$, we get $\langle x_1 - x, P_C(x_1) - x \rangle = \langle x_1 - x, x_1 - x \rangle = \|x_1 - x\|^2 > 0$. If $x_1 \notin C$, then the convex projection theorem (Theorem \ref{thm:convex_projection}) ensures that the half-space $H = \{y \in \R^d \colon \langle x_1 - P_C(x_1), y - P_C(x_1) \rangle \le 0\}$ contains $C$. In particular,
  \[ 0 \ge \langle x_1 - P_C(x_1), x - P_C(x_1) \rangle = \langle x_1 - x, x - P_C(x_1) \rangle + \|x - P_C(x_1)\|^2.  \]
  Consequently, $ \langle x_1 - x, P_C(x_1) - x \rangle \ge \|x - P_C(x_1)\|^2 \ge 0$. In addition, equality holds if and only if $x = P_C(x_1)$, in which case the iterative algorithm will have converged (observe that this happens when $x \in \fr C$ and the gradient descent direction points out of $C$ and orthogonally to the supporting hyperplane). Therefore, as long as the projected gradient iterations produce changes in the obtained points, an appropriate learning rate will ensure the descent in the objective function. Figure \ref{fig:gradient} visually compares the gradient descent method and the projected gradient method.

  \begin{figure}[htbp]
    \centering
    \includegraphics[width=0.75\textwidth]{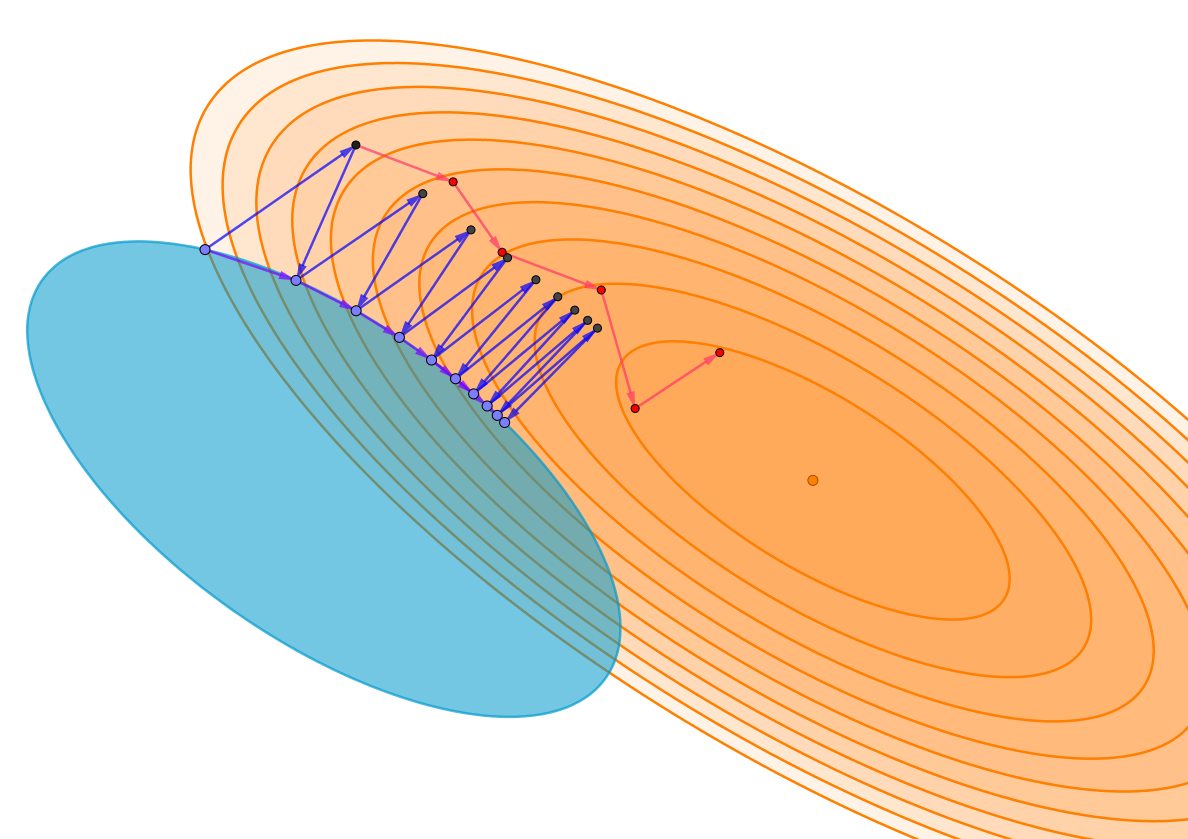}
    \caption{Orange shaded areas represent contour lines of the function $f(x,y) = 2(x+y)^2 + 2y^2$ for natural values between $0$ and $10$. The red path shows the behaviour of the unconstrained gradient descent method applied to $f$. The blue path shows the behaviour of the projected gradient descent, with the blue ellipse as the feasible set. In both cases we observe that we are obtaining descent directions.} \label{fig:gradient}
  \end{figure}

  Another problem we can find when trying to optimize constrained convex problems is that we may have multiple constraints, but we only know the projection onto each single restriction, without knowing the projection onto the intersection, which makes up the feasible set. In these cases, a popular method to find a point in the intersection is the so-called \emph{iterated projections method}, which consists of taking successive projections onto each constraint set, and repeating this procedure cyclically. We will analyze the simplest case, that is, let us suppose that we have a feasible set determined by two convex constraints. The following theorem states that, if the intersection of the sets determined by each constraint is not empty, then the sequence of iterated projections converge to a point in the intersection.

  \begin{theorem}[Convergence of the iterated projections method]
    Let $C, D \subset \R^d$ be closed convex sets, and let $P_C, P_D \colon \R^d \to \R^d$ be the projections onto $C$ and $D$, respectively. Suppose that $x_0 \in C$ and we build the sequences $\{x_n\}$ and $\{y_n\}$ given by $y_n = P_D(x_n)$ and $x_{n+1} = P_C(y_n)$, for each $n \in \N \cup \{0\}$.

    Then, if $C \cap D \ne \emptyset$, both sequences converge to a point $x^* \in C \cap D$.
  \end{theorem}

  Proof of this result is provided by \citet{boyd2003alternating}. The extension to the general case can be made following a similar argument, and it is discussed by \citet{bregman_projections}. That is why the general case is also called the \emph{Bregman projections method}. Figure \ref{fig:iterproj} shows a graphical example of the iterated projections method.

  \begin{figure}[htbp]
    \centering
    \includegraphics[width=0.75\textwidth]{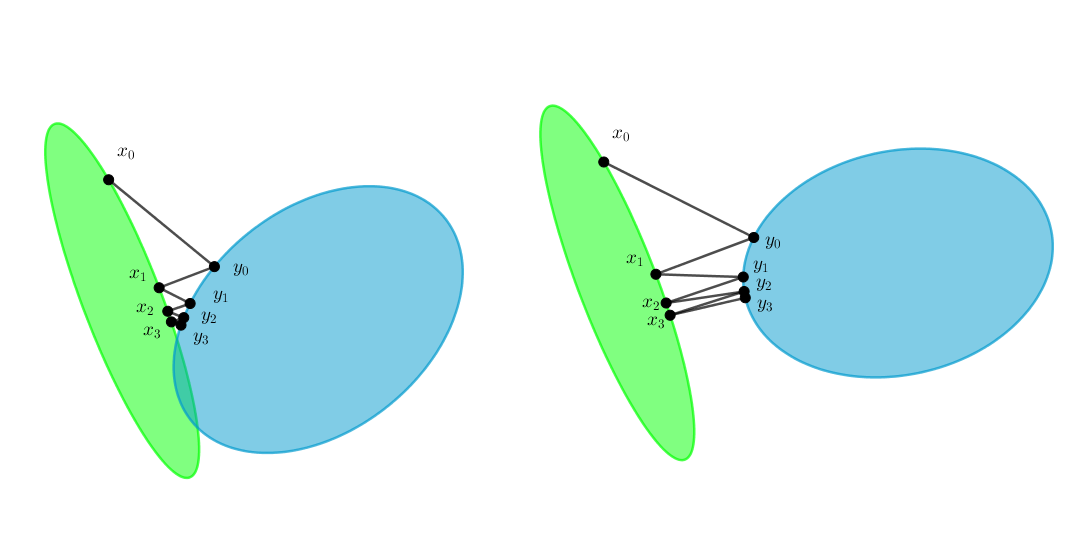}
    \caption{The iterated projections method. The second image shows how the algorithm works if the sets do not intersect.} \label{fig:iterproj}
  \end{figure}

  To conclude this section, we need to make a last remark. Recall that a convex and differentiable function $f \colon \Omega \to \R$ defined on a convex open set verifies that $f(x) \ge \langle \nabla f(x_0), x - x_0 \rangle$, for every $x, x_0 \in \Omega$. Let $x_0 \in \Omega$ be fix. When we work with convex but non differentiable functions, there are still vectors $v \in \R^d$ for which $f(x) \ge f(x_0) + \langle v, x - x_0 \rangle$, for every $x \in \Omega$. This is a consequence of the supporting hyperplane theorem applied to the epigraph of $f$ (recall that $f$ is convex iff its epigraph is too). In this case we say that $v$ is a \emph{subgradient} of $f$ at $x_0$ and we note it as $\partial f(x_0)$, or $\partial f(x_0) / \partial x$, if we need to specify the variable.

  Subgradients and gradients have similar behaviours, although we cannot always guarantee that subgradients are descent directions. Subgradient methods work in a similar way to gradient methods, replacing the gradient in the adaptation rule by a subgradient. In subgradient methods it is useful to keep track of the best value obtained, as some subgradients may not be descent directions. In the situations we will handle, subgradient computations are easy: if $f$ is differentiable at $x_0$, then $\nabla f(x_0)$ is a subgradient (in fact, this is the only subgradient at $x_0$); if $f$ is a maximum of convex differentiable functions, then a subgradient at $x_0$ is the gradient of any of the differentiable functions that attains the maximum at $x_0$.

\subsection{Matrix Analysis}

  In distance metric learning, matrices will play a key role, as they will be the structure over which distances will be defined and over which the optimization methods studied in the previous section will be applied. Within the set of all matrices, positive semidefinite matrices will be of even greater importance, so, in order to better understand the learning problems we will be dealing with, it will be necessary to delve into some of their numerous properties.

  This section examines in depth the study of matrices, on a basis of the best-known results of diagonalization in linear algebra. From this basis, we will show how to give the set of matrices a Hilbert space structure, in order to be able to apply the convexity results and optimization methods from the previous chapter. In particular, we will be interested in how to obtain projections onto the set of positive semidefinite matrices. Also related to positive semidefinite matrices, we will present several results regarding decomposition that we will need in future sections. Finally, we will study some matrix optimization problems that can be solved via eigenvectors. Table \ref{table:matrix_notation} shows the notations we will use for matrices. We will restrict the study to the real case, since the problem we will deal with is in real variables, although many of the results we will see can be extended to the complex case.

  \begin{table}[htbp]
    \begin{tabular}{rl}
      \toprule
      Notation & Concept \\
      \midrule
      $\mathcal{M}_{d'\times d}(\R)$ & Matrices of order $d' \times d$. \\
      $\mathcal{M}_d(\R)$ & Square matrices of orden $d$. \\
      $A_{ij}$ & The value of the matrix $A$ at the $i$-th row and $j$-th column. \\
      $A_{.j}$ (resp. $A_{i.}$) & The $j$-th column (resp. the $i$-th row) of the matrix $A$. \\
      $ v = (v_1,\dots,v_d)$ & A vector in $\R^d$. Vectors will be treated as column matrices. \\
      $A^T$ & The transpose of the matrix $A$. \\
      $S_d(\R)$ & Symmetric matrices of order $d$. \\
      $\gl_d(\R)$ & Invertible matrices of order $d$. \\
      $r(A), \tr(A), \det(A)$ & The rank, trace and determinant of the matrix $A$. \\
      $O_d(\R)$ & Orthogonal matrices of order $d$. \\
      $S_d(\R)^+_0$ & Positive semidefinite matrices of order $d$. \\
      $S_d(\R)^+$ & Positive definite matrices of order $d$. \\
      $S_d(\R)^-_0$ & Negative semidefinite matrices of order $d$. \\
      $S_d(\R)^-$ & Negative definite matrices of order $d$.
    \end{tabular}
    \caption{Matrices notations.}  \label{table:matrix_notation}
  \end{table}

\subsubsection{Matrices as a Hilbert Space. Projections}

Over the set of matrices we have defined a sum operation, and a matrix product, between matrices of orders $d \times r$ and $r \times n$. When working with square matrices, this sum and product give the matrix set a non-conmutative ring structure. These operations only allow us to obtain algebraic properties of matrices, but we also want to obtain geometric and topological properties. That is why we need to introduce a matrix inner product. We will introduce this inner product in the simplest way, that is, we will view matrices as vectors where we add the matrix rows one after the other, and we will consider the usual vector inner product. This matrix product is known as \emph{Frobenius inner product}.

\begin{definition}
  We define the \emph{Frobenius inner product} over the matrices space of order $d' \times d$ as the mapping $\langle \cdot, \cdot \rangle_F \colon \mathcal{M}_{d' \times d}(\R) \times \mathcal{M}_{d' \times d}(\R) \to \R$ given by
  \[ \langle A, B \rangle_F = \sum_{i=1}^{d'}\sum_{j=1}^{d} A_{ij}B_{ij} = \tr(A^TB). \]

  We define the \emph{Frobenius norm} over the matrices space of order $d' \times d$ as the mapping $\|\cdot \|_F \colon \mathcal{M}_{d' \times d}(\R) \to \R^+_0$ given by
  \[ \|A\|_F = \sqrt{\langle A, A \rangle} = \sqrt{\sum_{i=1}^{d'}\sum_{j=1}^{d}A_{ij}^2} = \sqrt{\tr(A^TA)}. \]

\end{definition}

 Frobenius norm is therefore identical to the euclidean norm in $\R^{d' \times d}$ identifying matrices with vectors as mentioned before. Viewing this norm as a matrix norm, we have to remark that Frobenius norm is sub-multiplicative, but it is not induced by any vector norm. Some interesting properties about Frobenius norm can be deduced from the definition. They are listed below.

\begin{proposition}~
  \begin{enumerate}
    \item For each $A \in \mathcal{M}_{d' \times d}(\R)$, $\|A\|_F = \|A^T\|_F$.
    \item For each $A \in \mathcal{M}_{d' \times d}(\R)$, $\|A\|_F = \sqrt{\tr(AA^T)}$
    \item If $U \in O_d(\R), V \in O_{d'}(\R)$ and $A \in \mathcal{M}_{d' \times d}(\R)$, then $\|AU\|_F = \|VA\|_F = \|VAU\|_F = \|A\|_F$.
    \item If $A \in S_d(\R)$, then $\|A\|_F^2 = \sum_{i=1}^d \lambda_i^2$, where $\lambda_1,\dots,\lambda_d$ are the eigenvalues of $A$.
  \end{enumerate}
\end{proposition}

With Frobenius inner product we can apply the convex analysis theory studied in the previous section. The positive semidefinite matrix set has a convex cone structure, that is, it is closed under non-negative linear combinations. That is why $S_d(\R)^+_0$ is usually called the positive semidefinite cone. Under the topology induced by symmetric matrices, we can also see that $S_d(\R)^+_0$ is closed, as it is the intersection of closed sets:
\[ S_d(\R)^+_0 = \{M \in S_d(\R) \colon x^TMx \ge 0 \ \forall x \in \R^d \} = \bigcap_{x\in \R^d} \{M \in S_d(\R) \colon x^TMx \ge 0\}.\]

So we understand, in particular, that $S_d(\R)^+_0$ is a closed convex set over symmetric matrices, and thus we have a well-defined projection onto the positive semidefinite cone. This property is very important for many of the optimization problems we will study, since they will try to optimize functions defined over the positive semidefinite cone. Here, the projected gradient descent method will be of great use, thus constituting one of the most basic algorithms of the paradigm of semidefinite programming. We can calculate the projection onto the positive semidefinite cone explicitly, as we will see below.

\begin{definition} 
  Let $\Sigma \in \mathcal{M}_d(\R)$ be a diagonal matrix, $\Sigma = \diag(\sigma_1,\dots,\sigma_d)$. We define the \emph{positive part} of $\Sigma$ as $\Sigma^+ = \diag(\sigma_1^+,\dots,\sigma_d^+)$, where $\sigma_i^+ = \max\{\sigma_i,0\}$. In a similar way, we define its \emph{negative part} as $\Sigma^- = \diag(\sigma_1^-,\dots,\sigma_d^-)$, where $\sigma_i^- = \max\{-\sigma_i,0\}$.

  Let $A \in S_d(\R)$ and let $A = UDU^T$ be a spectral decomposition of $A$. We define the \emph{positive part} of $A$ as $A^+ = UD^+U^T$. In a similar way we define its \emph{negative part} as $A^- = UD^-U^T$.
\end{definition}

\begin{theorem}[Semidefinite projection] \label{thm:psd_projection}
  Let $A \in S_d(\R)$. Then, $A^+$ is the projection of $A$ onto the positive semidefinite cone.
\end{theorem}

This result has been proven by \citet{psd_projection}, and is extended easily to project any square matrix onto the positive semidefinite cone, as we show below, although the most interesting case is that mentioned previously.

\begin{corollary}
  Let $A \in \mathcal{M}_d(\R)$. Then, the projection of $A$ onto the positive semidefinite cone is given by $((A + A^T)/2)^+$.
\end{corollary}

\subsubsection{Decomposition Theorems}

The positive semidefinite cone allows many of the concepts and properties that we already know about the non negative real numbers to be generalized. For example, we can similarly define concepts as the square roots, and modules or absolute values. These concepts play an important role in elaborating numerous decomposition theorems that involve positive semidefinite matrices. We will use these tools in order to prove a specific decomposition theorem that will motivate the ways of modeling the distance metric learning problem. The statement of this theorem is shown below.

\begin{theorem} \label{thm:decomposition_llt}
  Let $M \in S_d(\R)^+_0$. Then,
  \begin{enumerate}
    \item There is a matrix $L \in \mathcal{M}_d(\R)$ so that $M = L^TL$.
    \item If $K \in \mathcal{M}_d(\R)$ is any other matrix with $M = K^TK$, then $K = UL$, where $U \in O_d(\R)$ (that is, $L$ is unique up to isometries).
  \end{enumerate}

\end{theorem}

To prove this theorem, we will start with a characterization of positive semidefinite matrices by decomposition, which will also allow us to introduce the concept of square root. We will rely on several previous lemmas.

\begin{lemma}  \label{lem:poly_conmute}
  Let $A, B \in \mathcal{M}_d(\R)$ be two commuting matrices, that is, $AB = BA$. Then, $Ap(B) = p(B)A$, where $p$ denotes any polynomial over matrices (that is, a expression of the form $p(C) = a_0I+a_1C+a_2C^2+\dots+a_nC^n$, with $a_0,\dots,a_n \in \R$).
\end{lemma}

\begin{proof}
  Observe that
  \[AB^n = (AB)B^{n-1} = B(AB)B^{n-2} = \dots = B^{n-1}(AB) = B^nA, \]
  and $Ap(B) = p(B)A$ is deduced by linearity.
\end{proof}

\begin{lemma} \label{lem:poly_diag_sqrt}
  Let $D \in S_d(\R)^+_0$ be a diagonal matrix. Then, there is a polynomial over matrices $p$ so that $p(D^2) = D$.
\end{lemma}

\begin{proof}
  Suppose $D = \diag(\lambda_1,\dots,\lambda_d)$, with $0 \le \lambda_1 \le \dots \le \lambda_d$. Then, $D^2 = \diag(\lambda_1^2,\dots,\lambda_d^2)$. We take $p$ as an interpolation polynomial over the points $(\lambda_i^2,\lambda_i)$, for $i = 1,\dots,d$. If we evaluate it on $D^2$ we obtain
  \[p(D^2) = p(\diag(\lambda_1^2,\dots,\lambda_d^2)) = \diag(p(\lambda_1^2),\dots,p(\lambda_d^2)) = \diag(\lambda_1,\dots,\lambda_d) = D. \]
\end{proof}

\begin{theorem} \label{thm:decomp_sqrt}
  Let $M \in \mathcal{M}_d(\R)$. Then,
  \begin{enumerate}
    \item $M \in S_d(\R)^+_0$ if, and only if, there is $L \in \mathcal{M}_d(\R)$ so that $M = L^TL$.
    \item If $M \in S_d(\R)^+_0$, there is a single matrix $N \in S_d(\R)^+_0$ with $N^2 = M$. In addition, $M \in S_d(\R)^+ \iff N \in S_d(\R)^+$.
  \end{enumerate}
\end{theorem}

\begin{proof},
  First we will see that $L^TL$ is a positive semidefinite, for any $L \in \mathcal{M}_d(\R)$. Indeed, given $x \in \R^d$,
  \[ x^TL^TLx = (Lx)^T(Lx) = \|Lx\|^2_2 \ge 0. \]

  We will prove the second implication of the first statement finding directly the matrix $N$ of the second statement. Consider the spectral decomposition $M = UDU^T$, with $U \in O_d(\R)$ and $D = \diag(\lambda_1,\dots,\lambda_d)$, with $0 \le \lambda_1 \le \dots \lambda_d$ the eigenvalues of $M$. We define $D^{1/2} = \diag(\sqrt{\lambda_1},\dots,\sqrt{\lambda_d})$ and construct the matrix $N = UD^{1/2}U^T$. $N$ is positive semidefinite, because its eigenvalues are those of $D^{1/2}$, which are all positive, and besides,
   \[N^2 = UD^{1/2}U^TUD^{1/2}U^T = UD^{1/2}D^{1/2}U^T = UDU^T = M.\]
   Furthermore, the strict positivity of the eigenvalues of $M$ is equivalent to that of the eigenvalues of $N$, then  $M \in \mathcal{M}_d(\R)^+ \iff N \in \mathcal{M}_d(\R)^+$. Let us finally see that $N$ is unique.

  Suppose that we have $N_1, N_2 \in S_d(\R)^+_0$ with $N_1^2 = M = N_2^2$. Observe that $N_1$ and $N_2$ must have the same eigenvalues, since they are necessarily the positive square roots of the eigenvalues of $M$. Therefore, $N_1$ and $N_2$ are similar to a same diagonal matrix, that is, there are matrices $U, V \in O_d(\R)$ with $N_1 = UDU^T$ and $N_2 = VDV^T$. From $N_1^2 = N_2^2$ we have
  \[ UD^2U^T = VD^2V^T \implies V^TUD^2 = D^2V^TU, \]
  so for $W = V^TU \in O_d(\R)$ we obtain that $D^2$ and $W$ commute. Combining Lemmas \ref{lem:poly_diag_sqrt} and \ref{lem:poly_conmute}, we obtain that $D$ and $W$ also commute. Therefore,
  \[WD = DW \implies V^TUD = DV^TU \implies UDU^T = VDV^T \implies N_1 = N_2, \]
  obtaining the uniqueness.
\end{proof}

As we had anticipated, this theorem motivates the definition of square roots for positive semidefinite matrices.

\begin{definition}
  Let $M \in S_d(\R)^+_0$. We define the \emph{square root} of $M$ as the unique matrix $N \in S_d(\R)^+_0$ with $N^2 = M$. We denote it as $N = M^{1/2}$.
\end{definition}

We can also extend other concepts defined over the non-negative real numbers to the positive semidefinite matrices. For example, the square root allows us to define the concept of module for any matrix.

\begin{definition}
  Let $A \in \mathcal{M}_{d\times d'}(\R)$. We define the \emph{module} of $A$ as
  \[ |A| = (A^TA)^{1/2} \in S_{d'}(\R)^+_0. \]
\end{definition}

With the module we can state a polar decomposition theorem, which shows a decomposition that can be seen as an extension of the polar form for complex numbers.

\begin{theorem}[Polar decomposition]
  Let $A \in \mathcal{M}_{d\times d'}(\R)$, with $d' \le d$. Then, there is a matrix $U \in \mathcal{M}_{d\times d'}(\R)$ with $U^TU = I$, so that $A = U|A|$. This decomposition is called the \emph{polar decomposition} of $A$, and it is not necessarily unique, unless $A$ is square and invertible.
\end{theorem}

\begin{proof}
  First, observe that, given $x \in \R^{d'}$, we have
  \[ \|Ax\|_2^2 = (Ax)^T(Ax) = x^TA^TAx = x^T|A|^2x = x^T|A||A|x = (|A|x)^T(|A|x) = \||A|x\|_2^2.\]
  This means that $A$ and $|A|$ have the same effect on the length of any vector. As an immediate consequence, we can observe that $\ker A =  \ker |A|$, since
  \[ x \in \ker A \iff Ax = 0 \iff \|Ax\| = 0 = \||A|x\| \iff |A|x = 0 \iff x \in \ker |A|. \]
  As $d' = \dim\ker A + \dim\im A = \dim\ker |A| + \dim\im |A|$, we also conclude that $\dim\im A = \dim\im |A|$, and then $r(A) = r(|A|)$. We will denote this rank as $r \le d$.

  $|A|$ is positive semidefinite, so there is an orthonormal basis $\{w_1,\dots,w_{d'}\} \subset \R^{d'}$ consisting of eigenvectors of $|A|$, with corresponding eigenvalues $\lambda_1,\dots,\lambda_{d'}$. We can assume that $\lambda_1,\dots,\lambda_r > 0$ and $\lambda_{r+1} = \dots = \lambda_{d'} = 0$, or equivalently, $\{w_{r+1},\dots,w_{d'}\}$ is an orthonormal basis of $\ker |A| = \ker A$.

  We consider the set of vectors $\{Aw_1/\lambda_1,\dots,Aw_r/\lambda_r \}$. Note that
  \[ \left\langle \frac{1}{\lambda_i}Aw_i,\frac{1}{\lambda_{j}}Aw_j\right\rangle = \frac{1}{\lambda_i\lambda_j}\langle Aw_i, Aw_j\rangle = \frac{1}{\lambda_i\lambda_j}w_i^T|A|^2w_j = \frac{1}{\lambda_i\lambda_j}w_i^T\lambda_j^2w_j = \frac{\lambda_j}{\lambda_i}w_i^Tw_j, \]
  which equals 1 if $i = j$, and 0 otherwise, so this set is also orthonormal. In fact, this set is an orthonormal basis of $\im A$.

  We extend the previous set to an orthonormal set of size $d'$ in $\R^d$,
  \[\left\{\frac{1}{\lambda_1}Aw_1,\dots,\frac{1}{\lambda_r}Aw_r, v_{r+1}, \dots, v_{d'} \right\}.\]

  Finally, we construct the matrix $V \in \mathcal{M}_{d\times d'}(\R)$ by adding as columns the vectors in the previous set, and the matrix $W \in \mathcal{M}_{d'}(\R)$ by adding as rows the vectors $w_1,\dots,w_{d'}$. We define $U$ as $U = VW \in \mathcal{M}_{d\times d'}(\R)$. Observe that both $V$ and $W$ have orthonormal columns, and then $V^TV = I = W^TW$, obtaining that $U^TU = I$ as well. We can also observe that $Ww_i = e_i$, where $\{e_1,\dots,e_{d'}\}$ is the canonical basis of $\R^{d'}$. Therefore, we obtain
  \[ Uw_i = \begin{cases} \frac{1}{\lambda_i}w_i, & 1 \le i \le r \\ v_i, & r < i \le d' \end{cases}, \]
  and finally,
  \[ U|A|w_i = \begin{cases} \lambda_i U w_i, & 1 \le i \le r \\ 0, & r < i \le d'\end{cases} = \begin{cases} Aw_i, &1 \le i \le r \\ 0, & r < i \le d'\end{cases} = Aw_i,\]
  where the last equality holds, since $\{w_{r+1},\dots,w_{d'}\} \subset \ker A$. So, we have the equality $A = U|A|$ on the basis $\{w_1,\dots,w_{d'}\}$, concluding the proof. The uniqueness of $U$ when $A$ is square and invertible is due to the fact that $|A|$ is also invertible in that case, and then $U = A|A|^{-1}$.

\end{proof}

\begin{remark}
  When $A \in \mathcal{M}_d(\R)$ is a square matrix, the polar decomposition can be stated as $A = U|A|$, where $U \in O_d(\R)$ is an orthogonal matrix.
\end{remark}

We are now in a position to prove Theorem \ref{thm:decomposition_llt}. We recall its statement below.

\vspace{\topsep}
\noindent
{\bf Theorem} {\it Let $M \in S_d(\R)^+_0$. Then,
  \begin{enumerate}
    \item There is a matrix $L \in \mathcal{M}_d(\R)$ so that $M = L^TL$.
    \item If $K \in \mathcal{M}_d(\R)$ is any other matrix with $M = K^TK$, then $K = UL$, where $U \in O_d(\R)$ (that is, $L$ is unique up to isometries).
  \end{enumerate}
}
\vspace{\topsep}
\begin{proof}
  The first statement was proved in Theorem \ref{thm:decomp_sqrt}. Suppose then that $L, K \in \mathcal{M}_d(\R)$ verify that $M = L^TL = K^TK$. Let $L = V|L|, K = W|K|$, with $V,W \in O_d(\R)$, be polar decompositions of $L$ and $K$. Then, we have
  \begin{align*}
        L^TL = K^TK &\implies |L|^TV^TV|L| = |K|^TW^TW|K| \\
                    &\implies |L|^T|L| = |K|^T|K| \implies |L|^2 = |K|^2.
  \end{align*}

  As $|L|$ and $|K|$ are positive semidefinite, they must be the only square root of $|L|^2 = |K|^2$, that is, $|L| = |K|$. We call $N = |L| = |K|$. Returning to the polar decompositions of $L$ and $K$, it follows that
  \[ N = V^TL = W^TK \implies K = WV^TL. \]
  Therefore, taking $U = WV^T \in O_d(\R)$, we obtain the desired equality.

\end{proof}

\subsubsection{Matrix Optimization Problems} \label{ssec:matrix_opt}

To conclude the section about matrix analysis, we consider that the analysis of several specific optimization problems based on eigenvectors is necessary. These problems can be expressed as the maximization of a trace, and they do not need analytical methods, like gradient methods, to find a solution to them. It can be solved only via algebraic methods, specifically by calculating the eigenvectors of the matrices involved in the problem. These problems appear in most of the dimensionality reduction distance metric learning algorithms. We state these problems, together with their solutions, in the lines below.

\begin{theorem} \label{thm:eigen_trace_opt}
    Let $d',d \in \N $, with $d' \le d$. Let $A \in \mathcal{S}_d(\R)$, and we consider the optimization problem
    
    \begin{equation*}
    \begin{split}
        \max_{L \in \mathcal{M}_{d'\times d}(\R)} &\quad \tr\left(LAL^T\right)  \\
        \text{s.t.: } &\quad LL^T = I.
    \end{split}
    \end{equation*}
    
    The problem attains a maximum if $L = \begin{pmatrix}
    \text{---} \hspace{-0.2cm} & v_1 & \hspace{-0.2cm} \text{---} \\
    & \dots &  \\
    \text{---} \hspace{-0.2cm} & v_{d'} & \hspace{-0.2cm} \text{---}
    \end{pmatrix}$, where $v_1,\dots,v_{d'}$ are orthonormal eigenvectors of $A$ corresponding to its $d'$ largest eigenvalues. In addition, the maximum value is the sum of the $d'$ largest eigenvalues of $A$.

\end{theorem}

\begin{theorem} \label{thm:eigen_trace_ratio_opt}
    Let $d',d \in \N $, with $d' \le d$. Let $A \in S_d(\R)$ and $B \in S_d(\R)^+$, and we consider the optimization problem
    
    \begin{equation*}
    \max_{L \in \mathcal{M}_{d'\times d}(\R)} \quad \tr\left((LBL^T)^{-1}(LAL^T)\right)
    \end{equation*}
    
    The problem attains a maximum if $L = \begin{pmatrix}
    \text{---} \hspace{-0.2cm} & v_1 & \hspace{-0.2cm} \text{---} \\
    & \dots &  \\
    \text{---} \hspace{-0.2cm} & v_{d'} & \hspace{-0.2cm} \text{---}
    \end{pmatrix}$, where $v_1,\dots,v_{d'}$ are eigenvectors of $B^{-1}A$ corresponding to its $d'$ largest eigenvalues.
\end{theorem}

\begin{theorem} \label{thm:eigen_trace_ratio_sym_opt}
    Let $d',d \in \N $, with $d' \le d$. Let $A, B \in S_d(\R)^+$, and we consider the optimization problem
    
    \begin{equation*}
    \max_{L \in \mathcal{M}_{d'\times d}(\R)} \quad \tr\left((LBL^T)^{-1}(LAL^T) + (LAL^T)^{-1}(LBL^T)\right)
    \end{equation*}
    
    The problem attains a maximum if $L = \begin{pmatrix}
    \text{---} \hspace{-0.2cm} & v_1 & \hspace{-0.2cm} \text{---} \\
    & \dots &  \\
    \text{---} \hspace{-0.2cm} & v_{d'} & \hspace{-0.2cm} \text{---}
    \end{pmatrix}$, where $v_1,\dots,v_{d'}$ are the $d'$ eigenvectors of $B^{-1}A$ with the highest values for the expression $\lambda_i + 1/\lambda_i$, where $\lambda_i$ is the eigenvalue associated with $v_i$.
\end{theorem}

These theorems can be proven using tools such as the Rayleigh quotient and the Courant-Fischer theorem and its consequences. First, we will introduce the Rayleigh quotient, and we will see its relationship with the eigenvalues and eigenvectors.

\begin{definition}
  Let $A \in S_d(\R)$. We define the \emph{Rayleigh quotient} associated with $A$ as the mapping $\rho_A \colon \R^d \setminus\{0\} \to \R$ given by
  \[ \rho_A(x) = \frac{x^TAx}{x^Tx} = \frac{\langle Ax, x\rangle}{\|x\|_2^2} \quad \forall x \in \R^d \setminus\{0\}. \]

  If $B \in S_d(\R)^+$, we define the \emph{generalized Rayleigh quotient} associated with $A$ and $B$ as the mapping $\mathcal{R}_{A,B}\colon \R^d \setminus\{0\} \to \R$ given by
  \[ \mathcal{R}_{A,B}(x) = \frac{x^TAx}{x^TBx} = \frac{\langle Ax ,x \rangle}{\|x\|_B^2} \quad \forall x \in \R^d \setminus \{0\}.\]

\end{definition}

Throughout this section we will assume that $A \in S_d(\R)$ and $B \in S_d(\R)^+$ are fixed, and we will refer to Rayleigh quotients as $\rho = \rho_A$ and $\mathcal{R} = \mathcal{R}_{A,B}$. A first observation about $\rho$ and $\mathcal{R}$ is that, for $x \in \R^d \setminus \{0\}$ and $\lambda \in \R^*$, it is verified that
\[ \mathcal{R}(\lambda x) = \frac{(\lambda x)^TA(\lambda x)}{(\lambda x)^TB(\lambda x)} = \frac{\lambda^2(x^TAx)}{\lambda^2(x^TBx)} = \mathcal{R}(x). \]

Therefore, $\mathcal{R}$ takes all its values over the $(d-1)$-dimensional unit sphere, that is, $\mathcal{R}(\R \setminus \{0\}) = \mathcal{R}(\mathbb{S}^{d-1}) \subset \R$. Since $\mathcal{R}$ is continuous and the sphere is compact, it follows that $\mathcal{R}$ achieves a maximum and a minimum in $\R^d \setminus \{0\}$. The same follows with $\rho$. These maxima and minima are closely related with the problems we want to analyze. We start studying the extremes of $\rho$.

\begin{theorem}[Rayleigh-Ritz]
  Let $\lambda_{\min}$ and $\lambda_{\max}$ be the minimum and maximum eigenvalues of $A$, respectively. Then,
  \begin{enumerate}
    \item For every $x \in \R^d$, $\lambda_{\min} \|x\|^2 \le x^TAx \le \lambda_{\max}\|x\|^2$.
    \item $\lambda_{\max} = \max_{x \in \R^d \setminus \{0\}} \frac{x^TAx}{x^Tx} = \max_{\|x\|_2 = 1} x^TAx$.
    \item $\lambda_{\min} = \min_{x \in \R^d \setminus \{0\}} \frac{x^TAx}{x^Tx} = \min_{\|x\|_2 = 1} x^TAx$.
  \end{enumerate}
  Therefore, the maximum and minimum values of $\rho$ are $\lambda_{\min}$ and $\lambda_{\max}$, respectively. These values are attained in the corresponding eigenvectors.
\end{theorem}

\begin{proof}
  Let $A = UDU^T$, with $U \in O_d(\R)$ and $D = \diag(\lambda_1,\dots,\lambda_d)$, where $\lambda_1 \le \dots \le \lambda_d$, be a spectral decomposition of $A$. Let $x \in \R^d \setminus \{0\}$ and we take $y = U^Tx$. Then,
  \begin{equation} \label{eq:ray_ritz:1}
        \rho(x) = \frac{x^TAx}{x^Tx} = \frac{x^TUDU^Tx}{x^Tx} = \frac{y^TU^TUDU^TUy}{y^TU^TUy} = \frac{y^TDy}{\|y\|^2_2} = \frac{\sum\limits_{i=1}^d \lambda_i y_i^2}{\|\lambda\|^2_2}. 
  \end{equation}
  In addition, it is clear that
  \[ \lambda_1\|y\|_2^2 = \lambda_1 \sum_{i=1}^d y_i^2 \le \sum_{i=1}^d \lambda_i y_i^2 \le \lambda_d \sum_{i=1}^d y_i^2 = \lambda_d\|y\|_2^2. \]
  Applying this inequality over Eq. \ref{eq:ray_ritz:1}, it follows that
  \[ \lambda_1 \le \rho(x) \le \lambda_d. \]
  Furthermore, if $u_1$ and $u_d$ are the corresponding eigenvectors of $\lambda_1$ and $\lambda_d$, we get
  \[ \rho(u_1) = \frac{u_1^TAu_1}{u_1^Tu_1} = \frac{\lambda_1 u_1^Tu_1}{u_1^Tu_1} = \lambda_1, \quad \rho(u_d) = \frac{u_d^TAu_d}{u_d^Tu_d} = \frac{\lambda_d u_d^Tu_d}{u_d^Tu_d} = \lambda_d.  \]
  Therefore, the equality is attained, and the three statements of the theorem follow from this equality.
\end{proof}

Rayleigh-Ritz theorem shows us that $\rho(\R^d \setminus \{0\}) = [\lambda_{\min}, \lambda_{\max}]$, obtaining the extreme values in the corresponding eigenvectors. However, these are not the only eigenvalues that can act as an optimal for a Rayleigh quotient. If we restrict ourselves to lower dimensional spaces, we can obtain any eigenvalue of $A$ as an optimal for the Rayleigh quotient, as we will see below.

\begin{theorem}[Courant-Fischer] \label{ref:courant_fischer}
  Let $\lambda_1 \le \dots \le \lambda_d$ the eigenvectors of $A$, and we denote by $S_k$ a vector subspace of $\R^d$ of dimension $k$. Then, for each $k \in \{1,\dots,d\}$, we get
  \begin{align}
        \lambda_k &= \min_{S_{k} \subset \R^d} \max_{\substack{x \in S_{k} \\ \|x\|_2 = 1}} x^TAx \label{eq:courant_fischer:1},\\ 
        \lambda_k &= \max_{S_{d-k+1} \subset \R^d} \min_{\substack{x \in S_{d-k+1} \\ \|x\|_2 = 1}} x^TAx \label{eq:courant_fischer:2}.
  \end{align}
\end{theorem}

This result extends the Rayleigh-Ritz statement, and this theorem is proven by \citet{matrix_analysis} (chap. 4). There we can also find the proof of an important consequence of Courant-Fischer theorem, usually known as the Cauchy's interlace theorem.

\begin{theorem}[Cauchy's interlace] \label{thm:interlace}
    Suppose that $\lambda_1 \le \dots \le \lambda_d$ are the eigenvalues of $A$. Let $J \subset \{1,\dots,d\}$ be a set of cardinal $|J| = d'$, and let $A_J \in S_{d'}(\R)$ be the matrix given by $A_J = (A_{ij})_{i,j \in J}$, that is, the submatrix of $A$ with the entries of $A$ whose indices are in $J \times J$. Then, if $\tau_1 \le \dots \le \tau_{d'}$ are the eigenvalues of $A_J$, for each $k \in \{1,\dots,d'\}$,
    \[ \lambda_k \le \tau_k \le \lambda_{k+d-d'}. \]
\end{theorem}

The next result follows from Cauchy's interlace theorem, and the inequality it states will help us to solve our optimization problems.

\begin{corollary} \label{cor:interlace}
  Let $L \in \mathcal{M}_{d' \times d}(\R)$ with $LL^T = I$. If $\mu_1 \ge \dots \ge \mu_d$ are the eigenvalues of $A$ and $\sigma_1 \ge \dots \ge \sigma_{d'}$ are the eigenvalues of $LAL^T$ (now we are considering eigenvalues in decreasing order), then $\sigma_k \le \mu_k$, for $k = 1,\dots,d'$.
\end{corollary}

\begin{proof}
  Since $LL^T = I$, the rows of $L$ are orthonormal eigenvectors. We can extend $L$ to an orthogonal matrix $\hat{L} \in O_d(\R)$ by adding $d-d'$ orthonormal eigenvectors, and orthonormal to the rows of $L$, in its rows. We have then that $\hat{L}A\hat{L}^T$ and $A$ have the same eigenvalues, and $LAL^T$ is a submatrix of $\hat{L}A\hat{L}^T$ obtained by deleting the last $d-d'$ rows and columns. The assertion now follows from Cauchy's interlace Theorem \ref{thm:interlace}, considering eigenvalues in the opposite order.
\end{proof}

We are now in a position to prove the theorems \ref{thm:eigen_trace_opt}, \ref{thm:eigen_trace_ratio_opt} and \ref{thm:eigen_trace_ratio_sym_opt} proposed at the beginning of this section.

\vspace{\topsep}
\noindent
{\bf Theorem} {\it Let $d',d \in \N $, with $d' \le d$. Let $A \in \mathcal{S}_d(\R)$, and we consider the optimization problem
    
    \begin{equation} \label{eq:eigen_trace_opt}
    \begin{split}
        \max_{L \in \mathcal{M}_{d'\times d}(\R)} &\quad \tr\left(LAL^T\right)  \\
        \text{s.t.: } &\quad LL^T = I.
    \end{split}
    \end{equation}
    
    The problem attains a maximum if $L = \begin{pmatrix}
    \text{---} \hspace{-0.2cm} & v_1 & \hspace{-0.2cm} \text{---} \\
    & \dots &  \\
    \text{---} \hspace{-0.2cm} & v_{d'} & \hspace{-0.2cm} \text{---}
    \end{pmatrix}$, where $v_1,\dots,v_{d'}$ are orthonormal eigenvectors of $A$ corresponding to its $d'$ largest eigenvalues. In addition, the maximum value is the sum of the $d'$ largest eigenvalues of $A$.

}
\vspace{\topsep}
\begin{proof}{}
  Let $\mu_1 \ge \dots \ge \mu_d$ the eigenvalues of $A$ in decreasing order, and $\sigma_1 \ge \dots \ge \sigma_{d'}$ the eigenvalues of $LAL^T$. By Corollary \ref{cor:interlace}, for any $L \in \mathcal{M}_{d' \times d}(\R)$ with $LL^T = I$,
  \[ \tr(LAL^T) = \sum_{i=1}^{d'} \sigma_i \le \sum_{i=1}^{d'} \mu_i. \]
  In addition, when the rows of $L$ are orthonormal eigenvectors $v_1,\dots, v_{d'}$ of $A$ corresponding to $\mu_1, \dots, \mu_{d'}$, we get $LL^T = I$ and $\tr(LAL^T) = \sum_{i=1}^{d'}\mu_i$, thus equality holds for these vectors.
\end{proof}

\begin{lemma}[Simultaneous diagonalization] \label{lem:diag_simult}
  Let $A \in S_d(\R)$ and $B \in S_d(\R)^+$. Then, there is an invertible matrix $P \in \gl_d(\R)$ and a diagonal matrix $D \in \mathcal{M}_d(\R)$ with $P^TAP = D$ and $P^TBP = I$.
\end{lemma}

\begin{proof}
  We consider the matrix $C = B^{-1/2}AB^{-1/2}$. $C$ is symmetric, since $A$ and $B$ are symmetric, thus there is a matrix $U \in O_d(\R)$ so that $U^TCU$ is diagonal. We call $D = U^TCU$ and we consider $P = B^{-1/2}U \in \gl_d(\R)$. We get
    \begin{align*}
        P^TAP &= P^TB^{1/2}CB^{1/2}P = (B^{-1/2}U)^TB^{1/2}CB^{1/2}(B^{-1/2}U) = U^TCU = D, \\
        P^TBP &= (B^{-1/2}U)^TB(B^{-1/2}U) = U^TB^{-1/2}BB^{-1/2}U = U^TU = I.
    \end{align*}
\end{proof}

\vspace{\topsep}
{\bf Theorem} {\it Let $d',d \in \N $, with $d' \le d$. Let $A \in S_d(\R)$ and $B \in S_d(\R)^+$, and we consider the optimization problem
    
    \begin{equation} \label{eq:eigen_trace_ratio_opt}
    \max_{L \in \mathcal{M}_{d'\times d}(\R)} \quad \tr\left((LBL^T)^{-1}(LAL^T)\right)
    \end{equation}
    
    The problem attains a maximum if $L = \begin{pmatrix}
    \text{---} \hspace{-0.2cm} & v_1 & \hspace{-0.2cm} \text{---} \\
    & \dots &  \\
    \text{---} \hspace{-0.2cm} & v_{d'} & \hspace{-0.2cm} \text{---}
    \end{pmatrix}$, where $v_1,\dots,v_{d'}$ are eigenvectors of $B^{-1}A$ corresponding to its $d'$ largest eigenvalues. 
}
\vspace{\topsep}
\begin{proof}
  We denote $U = L^T \in \mathcal{M}_{d\times d'}(\R)$. If we take the matrix $P$ from Lemma \ref{lem:diag_simult} and the matrix $V \in \mathcal{M}_{d\times d'}(\R)$ with $U = PV$ (it exists and it is unique, since $P$ is regular), we have
  \begin{align*}
        \tr\left((LBL^T)^{-1}(LAL^T)\right) &= \tr\left((U^TBU)^{-1}(U^TAU)\right) = \tr\left((V^TP^TBPV)^{-1}(V^TP^TAPV)\right) \\
        &=\tr((V^TV)^{-1}(V^TDV)).
  \end{align*}

  Therefore, maximizing Eq. \ref{eq:eigen_trace_ratio_opt} is equivalent to maximize with respect to $V$ the expression $\tr((V^TV)^{-1}(V^TDV))$, because the parameter change is bijective. Now we consider a polar decomposition $V = Q|V|$, with $Q \in \mathcal{M}_{d\times d'}(\R)$ verifying $Q^TQ = I$. It follows that
  \begin{align*}
        \tr((V^TV)^{-1}(V^TDV)) &= \tr((|V|^TQ^TQ|V|)^{-1}(|V|^TQ^TDQ|V|^T)) \\
                                &= \tr(|V|^{-1}|V|^{-T}(|V|^TQ^TDQ|V|))\\
                                &= \tr(|V|^{-1}Q^TDQ|V|) = \tr(Q^TDQ|V||V|^{-1}) = \tr(Q^TDQ).
  \end{align*}

  If we call $W = Q^T$, what we have obtained is that the maximization of Eq. \ref{eq:eigen_trace_ratio_opt} is equivalent to maximizing in $W$ $\tr(WDW^T)$, subject to $WW^T = I$, thus obtaining the optimization problem given in Eq. \ref{eq:eigen_trace_opt}. We can suppose the diagonal of $D$ ordered in descending order, and then a matrix $W$ that solves the optimization problem can be obtained adding as rows the vectors $e_1,\dots,e_{d'}$ of the canonical basis of $\R^d$. Then, $Q$ has contains the same vectors, but added by columns. Observe that the quotient trace $T(X) = \tr\left((X^TBX)^{-1}(X^TAX)\right)$, with $X \in \mathcal{M}_{d \times d'}(\R)$, is invariant with respect to right multiplications by invertible matrices. Indeed, if $R \in \gl_{d'}(\R)$,
    \begin{align*}
        T(XR) &= tr\left((R^TX^TBXR)^{-1}(R^TX^TAXR)\right) = \tr(R^{-1}(X^TBX)^{-1}R^{-T}R^T(X^TAX)R) \\
              &= \tr((X^TBX)^{-1}(X^TAX)RR^{-1}) = T(X).
    \end{align*}

  Since $U$ maximizes $T$ and $U = PQ|V|$, then $PQ$ also maximizes $T$. In addition, as from $P^TAP = D$ and $P^TBP = I$ we obtain that
  \[ D = P^TAP = (P^TBP)^{-1}(P^TAP) = P^{-1}B^{-1}P^{-T}P^TAP = P^{-1}B^{-1}AP, \]
  we conclude that $P$ diagonalizes $B^{-1}A$, and then, it contains as columns the eigenvectors of this matrix. Since $Q$ contains the $d'$ first eigenvectors of the canonical basis by columns, $PQ$ contains as columns the $d'$ first eigenvectors of $B^{-1}A$, corresponding to its $d'$ largest eigenvalues. This ends the proof, because a solution for the problem given by Eq. \ref{eq:eigen_trace_ratio_opt}, which is equal to maximizing $T$ except for a transposition, consists in adding those vectors as rows.
\end{proof}

\vspace{\topsep}
{\bf Theorem} {\it Let $d',d \in \N $, with $d' \le d$. Let $A, B \in S_d(\R)^+$, and we consider the optimization problem
    
    \begin{equation} \label{eq:eigen_trace_ratio_sym_opt}
    \max_{L \in \mathcal{M}_{d'\times d}(\R)} \quad \tr\left((LBL^T)^{-1}(LAL^T) + (LAL^T)^{-1}(LBL^T)\right)
    \end{equation}
    
    The problem attains a maximum if $L = \begin{pmatrix}
    \text{---} \hspace{-0.2cm} & v_1 & \hspace{-0.2cm} \text{---} \\
    & \dots &  \\
    \text{---} \hspace{-0.2cm} & v_{d'} & \hspace{-0.2cm} \text{---}
    \end{pmatrix}$, where $v_1,\dots,v_{d'}$ are the $d'$ eigenvectors of $B^{-1}A$ with the highest values for the expression $\lambda_i + 1/\lambda_i$, where $\lambda_i$ is the eigenvalue associated with $v_i$.
}
\vspace{\topsep}
\begin{proof}
  First of all, given $C \in S_d(\R)^+$ we consider the optimization problem
  \begin{equation} \label{eq:eigen_trace_sym_opt}
    \max_{L \in \mathcal{M}_{d'\times d}(\R)} \quad \tr\left((LCL^T + LC^{-1}L^T)\right) = \max_{L \in \mathcal{M}_{d'\times d}(\R)} \quad \tr\left(L(C+C^{-1})L^T\right)
  \end{equation}
  Using Theorem \ref{eq:eigen_trace_opt}, a solution to this problem can be found by taking as rows of $L$ the eigenvectors of $C+C^{-1}$ corresponding to its $d'$ largest eigenvalues. Observe that the eigenvectors of $C$ and $C^{-1}$ are the same, and each one's eigenvalues are the inverse of the other. Therefore, $C + C^{-1}$ also has the same eigenvectors, and its eigenvalues have the form $\lambda + 1/\lambda$, for each $\lambda$ eigenvalue of $C$. Then, the previous solution for Eq. \ref{eq:eigen_trace_sym_opt} is equivalent to taking the eigenvectors of $C$ for which $\lambda + 1/\lambda$ is maximized.

  Finally, we only have to realize that we can follow the same proof as in Theorem \ref{thm:eigen_trace_ratio_opt}, considering Eqs. \ref{eq:eigen_trace_sym_opt} and \ref{eq:eigen_trace_ratio_sym_opt} instead of Eqs. \ref{eq:eigen_trace_opt} and \ref{eq:eigen_trace_ratio_opt}.
\end{proof}

\subsection{Information Theory} \label{ssec:information_theory}

Information theory is a branch of mathematics and computer theory, with the purpose of establishing a rigurous measure to quantify the information and disorder contained in a communication message. It was developed with the aim of finding limits in signal processing operations such as compression, storage and communication. Today, its applications extend to most fields of science and engineering.

Many concepts associated with information theory have been defined, such as entropy, which measures the amount of uncertainty or information expected in an event, mutual information, which measures the amount of information that one random variable contains about another random variable, or relative entropy, which is a way of measuring the closeness between different random variables. We will focus on the relative entropy, and the concepts derived from it. To do this, we will first define the concept of divergence. Divergence is a magnitude to measure the closeness between certain objects in a set. We should not confuse divergences with distances (as described in Section \ref{ssec:distances}), because the magnitudes we will consider may not verify some of the properties required for distances, such as symmetry or triangle inequality.

\begin{definition}
  Let $X$ be a set. A map $D(\cdot \| \cdot) \colon X \times X \to \R$ is said to be a \emph{divergence} if it verifies the following properties:
  \begin{enumerate}
    \item Non negativity: $D(x\|y) \ge 0$, for every $x, y \in X$.
    \item Coincidence: $D(x\|y) = 0$ if, and only if, $x = y$.
  \end{enumerate}
\end{definition}

We will use divergences to measure the closeness between probability distributions. The divergences we will use will be presented in the following paragraphs.

\begin{definition}
  Let $(\Omega,\mathcal{A},P)$ be a probability space and $X \colon \Omega \to \R$ be a random variable, discrete or continuous, in that space. Suppose that $p$ is the corresponding probability mass function or density function. Suppose that $q$ is another probability mass function or density function. Then, we define the \emph{relative entropy} or the \emph{Kullback-Leibler} divergence between $p$ and $q$, as
  \[ \kl(p\|q) = \mathbb{E}_p\left[\log\frac{p(X)}{q(X)}\right],\]
  as long as such expectation exists. For the discrete case, if $p$ and $q$ are valued over the same points, we have
  \[ \kl(p\|q) = \sum_{x \in X(\Omega)} p(x) \log\frac{p(x)}{q(x)}, \]
  and for the continuous case, as long as the absolute integral is finite, we have
  \[ \kl(p\|q) = \int_{-\infty}^{+\infty} p(x)\log\frac{p(x)}{q(x)}\ dx.\]
  For continuity reasons, we assume that $0 \log(0/0) = 0$.
\end{definition}

The first step is to check that, indeed, Kullback-Leibler divergence is a divergence. This result is known as the information inequality.

\begin{theorem}[Information inequality]
  Kullback-Leibler divergence is a divergence, that is, $\kl(p\|q) \ge 0$ and the equality holds if, and only if, $p(x) = q(x)$ a.e. in $X(\Omega)$ (the equality is at every point in the discrete case).
\end{theorem}

\begin{proof}
  This result is an immediate consequence of Jensen's inequality \cite{rudin} applied to the $-\log$ function, which is strictly convex. We have
  \begin{align*}
        \kl(p \| q) &= \mathbb{E}_p\left[\log\frac{p(X)}{q(X)}\right] 
                    = \mathbb{E}_p\left[-\log\frac{q(X)}{p(X)}\right] \\
                    &\ge -\log\mathbb{E}_p\left[\frac{q(X)}{p(X)}\right] 
                    = -\log\int p(x) \frac{q(x)}{p(x)} \ dx \\
                    &= -\log\int q(x) \ dx = -\log 1 = 0.
  \end{align*}
  The proof for the discrete case is similar. In addition, the strict convexity implies that equality holds iff $p/q$ is constant a.e., iff $p = q$ a.e., since they are probability density functions or mass functions. And, as in the discrete case $p$ and $q$ are valued over sets with no null probabilities, we have equality at every point. 
\end{proof}

As we have already mentioned, Kullback-Leibler divergence is useful to measure closeness between probability distributions and can be used to bring the distributions closer. However, it is not all that useful to put the distributions away, since, as Kullback-Leibler divergence is not symmetric, the values of $\kl(p\|q)$ and $\kl(q\|p)$ may differ significantly when $p$ and $q$ are not near. That is why it is sometimes helpful to work with a symmetrization of the Kullback-Leibler divergence known as the Jeffrey divergence.

\begin{definition}
  The \emph{Jeffrey divergence} between two probability distributions $p$ and $q$ for which $\kl(p\|q)$ and $\kl(q\|p)$ exist is defined by
  \[ \jf(p\|q) = \kl(p\|q) + \kl(q\|p). \]
  In the discrete case we have
  \[ \jf(p\|q) = \sum_{x\in X(\Omega)} (p(x) - q(x))(\log p(x) - \log q(x)).\]
  And, for the continuous case,
  \[ \jf(p\|q) = \int_{-\infty}^{\infty} (p(x) - q(x))(\log p(x) -\log q(x)) \ dx.\]
\end{definition}

It is clear that Jeffrey divergence is a divergence, as a consequence of the information inequality, and it is also symmetric. Observe that both divergences are functions only of the probability distributions, that is, they only depend on the values set on the distributions. This fact allows divergence to be extended to random vectors, as long as we know its probability density functions or mass functions.

A case of special interest in the algorithms we will discuss in subsequent sections is the calculation of divergences between multivariate gaussian distributions. Recall that, if $\mu \in \R^d$ and $\Sigma \in S_d(\R)^+$, a random vector $X = (X_1,\dots,X_d)$ follows a multivariate gaussian distribution with mean $\mu$ and covariance $\Sigma$, if it has the following probability density function:
\[ p(x|\mu,\Sigma) = \frac{1}{(2\pi)^{d/2}\det(\Sigma)^{1/2}}\exp\left(-\frac{1}{2}(x-\mu)^T\Sigma^{-1}(x-\mu)\right) .\]

It is well-known that $\mathbb{E}[X] = \mu$ and $\cov(X) = \mathbb{E}[(X-\mathbb{E}[X])(X-\mathbb{E}[X])^T] = \Sigma$, thus gaussian distributions are completely defined by its mean and covariance. We want to establish an easy way to compute the calculation of divergences between gaussian distributions. To do this, we will find relationships between the studied divergences and matrix divergences. Matrix divergences are an alternative to the Frobenius norm for measuring the closeness between matrices. We are interested in the ones known as Bregman divergences.

\begin{definition}
  Let $K \subset \mathcal{M}_d(\R)$ be an open convex set, and $\phi \colon K \to \R$ a strictly convex and differentiable function. The \emph{Bregman divergence} corresponding to $\phi$ is the map $D_{\phi}(\cdot \|\cdot) \colon K \times K \to \R$ given by
  \[ D_{\phi}(A\|B) = \phi(A)-\phi(B) - \tr(\nabla \phi(B)^T(A-B)). \]
\end{definition}

Effectively, Bregman divergences are also divergences, as we can write the expression above as $D_{\phi}(A\|B) = \phi(A)-\phi(B) - \langle \nabla \phi(B), A - B\rangle_F$, which is known to be non negative when $\phi$ is strictly convex, and to take the zero value if and only if $A = B$. In our situation, we are interested in choosing the \emph{log-det} function to construct a Bregman divergence, that is, the function $\phi_{ld} \colon S_d(\R)^+ \to \R$ given by
\[ \phi_{ld}(M) = -\log\det(M). \]
This function is known to be strictly convex and its gradient is $\nabla f(M) = M^{-1}$, for each $M$ in $S_d(\R)^+$ \cite{convexoptimization}, hence we can construct the known as \emph{log-det divergence} through the expression
\[ D_{ld}(A\|B) = \log\det(B) - \log\det(A) - \tr(B^{-1}(A-B)) = \tr(AB^{-1}) - \log\det(AB^{-1}) - d.  \]

Once defined the log-det divergence, we are able to express the Kullback-Leibler and Jeffrey divergences between gaussian distributions in terms of this new matrix divergence.

\begin{theorem} \label{thm:kl_gaussian}
  Kullback-Leibler divergence between two multivariate gaussian distributions defined by the probability density functions $p_1(x|\mu_1,\Sigma_1)$ and $p_2(x|\mu_2,\Sigma_2)$, with $\mu_1,\mu_2 \in \R^d$ and $\Sigma_1,\Sigma_2 \in S_d(\R)^+$, verifies that
  \[ \kl(p_1\|p_2) = \frac{1}{2}D_{ld}(\Sigma_1\|\Sigma_2) + \frac{1}{2}\|\mu_1 - \mu_2\|_{\Sigma_1^{-1}}^2, \]
  where $\|\cdot\|_{\Sigma}$ denotes the norm defined by the positive definite matrix $\Sigma$, that is, $\| v \|_{\Sigma} = \sqrt{v^T\Sigma v}$, for every $v \in \R^d$.
\end{theorem}

Proof of this result can be found in \citet{davis2007differential} (Section 3.1). A simpler version of this theorem can be stated immediately, when we consider equal-mean gaussian distributions.

\begin{corollary} \label{cor:kl_gaussian}
  Kullback-Leibler divergence between two multivariate gaussian distributions defined by the probability density functions $p_1$ and $p_2$ with equal means and covariances $\Sigma_1$ and $\Sigma_2$, verifies that
  \[ \kl(p_1\|p_2) = \frac{1}{2}D_{ld}(\Sigma_1\|\Sigma_2). \]
\end{corollary}

Using these results, we can also express the Jeffrey divergence between gaussian distributions in terms of its mean vectors and covariance matrices. The following expressions can be easily deduced from the theorems above. For more details, see also \cite[App. B]{dmlmj}.

\begin{corollary} \label{thm_jf_gaussian}
  Jeffrey divergence between two multivariate gaussian distributions defined by the probability density functions $p_1(x|\mu_1,\Sigma_1)$ and $p_2(x|\mu_2,\Sigma_2)$ with $\mu_1,\mu_2 \in \R^d$ and $\Sigma_1,\Sigma_2 \in S_d(\R)^+$, verifies that
  \[ \jf(p_1\|p_2) = \frac{1}{2}\tr(\Sigma_1\Sigma_2^{-1}+\Sigma_1^{-1}\Sigma_2) - d + \frac{1}{2}\|\mu_1-\mu_2\|^2_{\Sigma_1^{-1}+\Sigma_2^{-1}}. \]
\end{corollary}

\begin{corollary} \label{cor:jf_gaussian}
  Jeffrey divergence between two multivariate gaussian distributions defined by the probability density functions $p_1$ and $p_2$ with equal means and covariances $\Sigma_1$ and $\Sigma_2$, verifies that
  \[ \jf(p_1\|p_2) = \frac{1}{2}\tr(\Sigma_1\Sigma_2^{-1}+\Sigma_1^{-1}\Sigma_2) - d. \]
\end{corollary}

\section{Algorithms for Distance Metric Learning: Detailed Explanation} \label{app:algs}

This appendix describes some of the most popular techniques currently being used in supervised distance metric learning. We also add a review of the principal component analysis, although not supervised, because of its importance for other distance metric learning algorithms. Some of these techniques, such as PCA or LDA \cite{cunningham2015linear}, are statistical procedures developed over the last century, which are still of great relevance in many problems nowadays. Other more recent proposals are in the state of the art, as is the case of NCMML \cite{ncmml} or DMLMJ \cite{dmlmj}, among others. Several of the most popular classic distance metric learning algorithms, such as LMNN \cite{lmnn} or NCA \cite{nca}, have also been included.

The analyzed techniques are grouped into six subsections. Each of these subsections describes algorithms that share the main purpose, although the purposes described in each section are not exclusive. In the first section (Appendix \ref{ssec:dim_red}) we will study the techniques oriented specifically to dimensionality reduction. Next, the techniques with the purpose of learning distances that improve the nearest neighbors classifiers will be developed (Appendix \ref{ssec:algs_knn}), followed by those techniques that aim to improve classifiers based on centroids (Appendix \ref{ssec:algs_ncm}). The fourth subsection includes methods based on the information theory concepts studied in Appendix \ref{ssec:information_theory}. Subsequently, several distance metric learning mechanisms with less specific goals are described (Appendix \ref{ssec:other_dmls}). Finally, kernel-based versions of some of the above algorithms are analyzed, to be able to work in high-dimensionality spaces (Appendix \ref{ssec:kernel_dml}).

For each of the techniques we will analyze the problem they try to solve or optimize, the mathematical formulations of those problems and the algorithms proposed to solve them.

\subsection{Dimensionality Reduction Techniques} \label{ssec:dim_red}

Dimensionality reduction techniques try to learn a distance by searching for a linear transformation from the dataset space to a lower dimensional euclidean space. These kinds of algorithms share many features. For instance, they are usually efficient and their execution involves the calculation of eigenvectors. It is important to point out that there are other non-linear or unsupervised dimensionality reduction techniques \cite{lee2007nonlinear}, but they are beyond the scope of this paper (with the exception of kernel versions in Appendix \ref{ssec:kernel_dml}). The algorithms we will describe are PCA \cite{pca}, LDA \cite{lda} and ANMM \cite{anmm}.

\subsubsection{Principal Component Analysis (PCA)} \label{desc:pca}

PCA \cite{pca} is one of the most popular dimensionality reduction techniques in unsupervised distance metric learning. Although PCA is an unsupervised learning algorithm, it is necessary to talk about it in our work, firstly because of its great relevance, and more particularly, because when a supervised distance metric learning algorithm does not allow a dimensionality reduction, PCA can be first applied to the data in order to be able to use the algorithm later in the lower dimensional space.

Principal component analysis can be understood from two different points of view, which end up leading to the same optimization problem. The first of these approaches consists of finding two linear transformations, one that compresses the data to a smaller space, and another that decompresses them in the original space, so that in the process of compression and decompression the minimum information is lost.

Let us focus on this first approach. Suppose we have the dataset $\mathcal{X} = \{x_1,\dots,x_N\} \subset \R^d$, and fix $0 < d' < d$. Let us also assume that data are centered, that is, that the mean of the dataset is zero. If it is not the case, it is enough to apply previously to the data the transformation $x \mapsto x - \mu$, where $\mu = \sum x_i / N$ is the dataset mean. We are looking for a compression matrix $L \in \mathcal{M}_{d' \times d}(\R)$, and a decompression matrix $U \in \mathcal{M}_{d \times d'}(\R)$, so that, after compressing and decompressing each data the squares of the euclidean distances to the original data are minimal. In other words, the problem we are trying to solve is
\begin{equation} \label{eq:pca:compress}
    \min_{\substack{L \in \mathcal{M}_{d'\times d}(\R) \\ U \in \mathcal{M}_{d\times d'}(\R)}} \quad \sum_{i=1}^{N} \|x_i - ULx_i\|_2^2.
\end{equation}

To find a solution to this problem, first of all we will see that $U$ and $L$ matrices have to be related in a very particular way.
\begin{lemma}
  If $(U,L)$ is a solution  of the problem given in Eq. \ref{eq:pca:compress}, then $LL^T = I$ (in $\R^{d'}$) and $U = L^T$.
\end{lemma}
\begin{proof}
  We fix $U \in \mathcal{M}_{d\times d'}(\R)$ and $L \in \mathcal{M}_{d' \times d}(\R)$. We can assume that both $U$ and $L$ are full-rank, otherwise the rank of $UL$ is lower than $d'$. Note that in that case, it is always possible to extend $U$ and $L$ matrices to full-rank matrices (by replacing linear combinations in the columns by linear independent vectors as long as the dimension allows it) so that the subspace generated extends the one generated by $UL$, and in such a case, the error obtained in Eq. \ref{eq:pca:compress} for the extension will be, at most, the error obtained for $U$ and $L$.

  We consider the linear map $x \mapsto ULx$. The image of this map, $R = \{ULx \colon x \in \R^d\}$, is a vector subspace of $\R^{d}$ of dimension $d'$. Let $\{u_1,\dots,u_{d'}\}$ be an orthonormal basis of $R$, and let $V \in \mathcal{M}_{d'\times d}(\R)$ the matrix that has, by rows, the vectors $u_1,\dots,u_{d'}$. It is verified then that the image of $V$ has dimension $d'$ and that $VV^T = I$. In addition, if we consider $V^T$ as a linear map, we see that its image is $R$ (since $V^Te_i = u_i, i = 1,\dots, d'$, where $\{e_1,\dots,e_{d'}\}$ is the canonical basis of $\R^{d'}$).

  Therefore, every vector of $R$ can be written as $V^Ty$, with $y \in \mathbb{R}^{d'}$. Given $x \in \R^d, y \in \R^{d'}$, we have
  \begin{align*}
    \|x-V^Ty\|_2^2 &= \langle x- V^Ty, x - V^Ty \rangle \\
                   &= \|x\|^2 - 2\langle x,V^T y\rangle + \|V^Ty\|^2 \\
                   &= \|x\|^2 - 2\langle y,Vx \rangle + y^TVV^Ty \\
                   &= \|x\|^2 - 2\langle y,Vx \rangle + y^Ty \\
                   &= \|x\|^2 + \|y\|^2 - 2 \langle y,Vx \rangle.
  \end{align*}

  If we calculate the gradient with respect to $y$ from the last previous expression, we obtain $\nabla_y \|x-V^Ty\|_2^2 = 2y - 2Vx$, which, by equating to zero, allows us to obtain a single critical point, $y = Vx$. The convexity of this function (it is the composition of the euclidean norm with an affine map) assures us that this critical point is a global minimum. Therefore, this tells us that, for each $x \in \R^d$, the distance to $x$ in the set $R$ achieves its minimum at the point $V^TVx$. In particular, for the dataset $\mathcal{X}$ we conclude that
  \[ \sum_{i=1}^N \|x_i - ULx_i\|_2^2 \ge \sum_{i=1}^N\|x_i - V^TV x_i\|^2_2. \]

  Since $U$ and $L$ were fixed, we can find a matrix $V$ with these properties for any $U$ and $L$ in the conditions of the problem, which concludes the proof.
\end{proof}

The above lemma allows us to reformulate our problem in terms of only the matrix $L$,
\begin{equation} \label{eq:pca:compress2}
    \min_{\substack{L \in \mathcal{M}_{d'\times d}(\R) \\LL^T = I}} \quad \sum_{i=1}^{N} \|x_i - L^TLx_i\|_2^2.
\end{equation}

Let us note now that, for $x \in \R^d$ and $L \in \mathcal{M}_{d'\times d}(\R)$, it is verified that
\begin{align*}
    \|x - L^TLx\|_2^2 &= \langle x - L^TLx, x - L^TLx \rangle \\
                      &= \|x\|^2 - 2\langle x,L^TLx \rangle + \langle L^TLx, L^TLx \rangle \\
                      &= \|x\|^2 - 2x^TL^TLx + x^TL^TLL^TLx \\
                      &= \|x\|^2 - x^TL^TLx \\
                      &= \|x\|^2 - \tr(x^TL^TLx) \\
                      &= \|x\|^2 - \tr(Lxx^TL^T).
\end{align*}

Thus, if we remove terms that do not depend on $L$, we can transform the problem in Eq. \ref{eq:pca:compress2} into the following equivalent problem:
\begin{equation} \label{eq:pca:traceproblem}
    \max_{\substack{L \in \mathcal{M}_{d'\times d}(\R) \\LL^T = I}} \quad \tr\left(L \Sigma L^T\right),
\end{equation}
where $\Sigma = \sum_{i=1}^N x_i x_i^T$ is, except for a constant, the covariance matrix corresponding to the data in $\mathcal{X}$. This matrix is symmetric, and Theorem \ref{thm:eigen_trace_opt} guarantees that we can find a maximum of the problem if we build $L$ adding the $d'$ orthonormal eigenvectors corresponding to the $d'$ largest eigenvalues of $\Sigma$. The directions that determine these vectors are the \emph{principal directions}, and the components of the data transformed in the orthonormal system determined by the principal directions are the so-called \emph{principal components}.

To conclude, the second approach from which the principal components problem can be dealt with consists of selecting the orthogonal directions for which the variance is maximized. We know that if $\Sigma$ is the covariance matrix of $\mathcal{X}$, when applying a linear transformation $L$ to the data, the new covariance matrix is given by $L\Sigma L^T$. If we want a transformation that reduces the dimensionality and for which the variance is maximized in each variable, what we are looking for is to take the trace of the previous matrix, which leads us back again to Eq. \ref{eq:pca:traceproblem}. The symmetry of $\Sigma$ ensures that we can take the main orthonormal directions that maximize the variance for each possible value of $d'$.

Finally, it is important to note that the matrix $L \in \mathcal{M}_{d}(\R)$ (taking all dimensions) that is constructed by adding $\Sigma$ eigenvectors row by row is the orthogonal matrix that diagonalizes $\Sigma$, and therefore, when $L$ is applied to the data, the transformed data have as the covariance matrix the diagonal matrix $L\Sigma L^T = \diag(\lambda_1,\dots,\lambda_d)$, where $\lambda_1 \ge \dots \ge \lambda_d$ are the eigenvalues of $\Sigma$. This tells us that the eigenvalues of the covariance matrix represent the amount of variance explained by each of the principal directions. This provides an additional advantage to PCA, since it allows the percentage of variance that explains each principal component to be analyzed, in order to be able to later choose a dimension that adjusts to the amount of variance that we want to keep in the transformed data.

Figure \ref{fig:pca} graphically exemplifies how principal component analysis works.

\begin{figure}[htbp]
    \centering
    \includegraphics[width=0.75\textwidth]{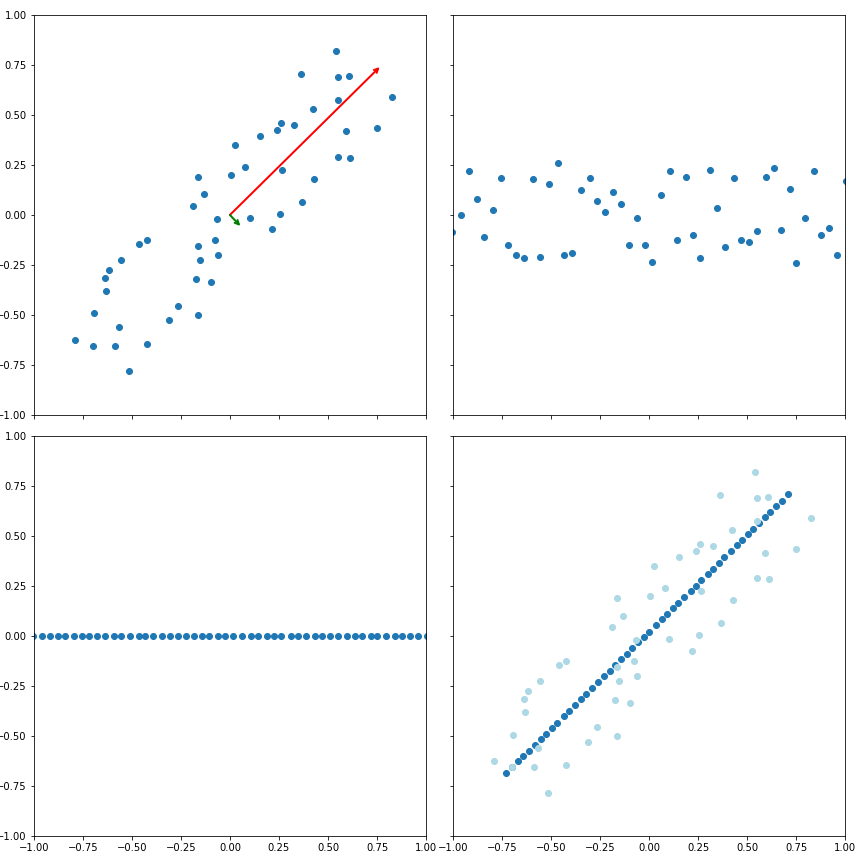}
    \caption{A graphical example of PCA. The first image shows a dataset, along with the principal directions (proportional according to the explained variance) learned by PCA. To the right, the data is projected at maximum dimension. We observe that this projection consists of rotating the data making the axes coincide with the principal directions. At the bottom left, data is projected onto the first principal component. Finally, to the right, the data recovered through the decompression matrix, along with the original data. We can see that the PCA projection is the one that minimizes the quadratic decompression error. In this particular case the decompressed data is on the regression line of the original data, due to the dimensions of the problem.} \label{fig:pca}
\end{figure}

\subsubsection{Linear Discriminant Analysis (LDA)} \label{desc:lda}

LDA \cite{lda} is a classical distance metric learning technique with the purpose of learning a projection matrix that maximizes the separation between classes in the projected space, that is, it tries to find the directions that best distinguish the different classes, as shown in Figure \ref{fig:lda}.

\begin{figure}[htbp]
    \centering
    \includegraphics[width=0.75\textwidth]{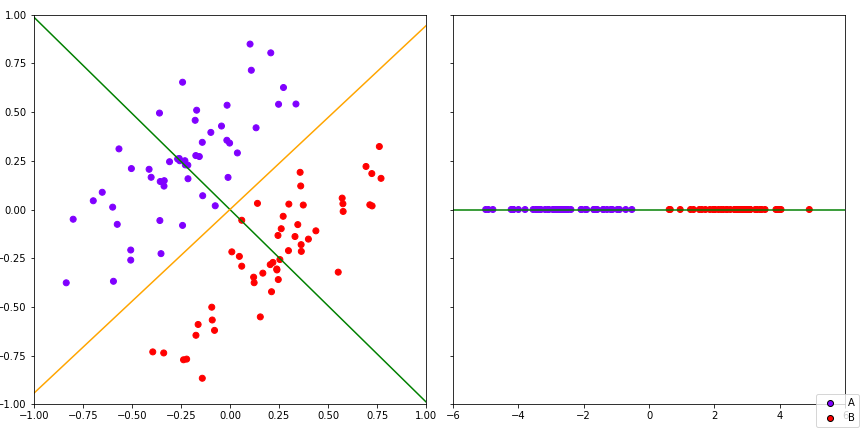}
    \caption{Graphical example of LDA and comparison with PCA. The first image shows a dataset, with the first principal direction determined by PCA, in orange, and the direction determined by LDA, in green. We observe that if we project the data on the direction obtained by LDA they separate, as it is shown in the right image. In contrast, the direction obtained by PCA only allows us to maximize the variance of the whole dataset, since it does not consider the information of the labels.} \label{fig:lda}
\end{figure}

Figure \ref{fig:lda} also allows us to compare the results of the projections obtained by PCA and LDA, showing the most remarkable difference between the two techniques: PCA does not take into account the labels information, while LDA does use it. We can observe that the directions obtained by PCA and LDA do not present any type of relationship, the latter being the only one of them that provides a data projection oriented to supervised learning.

It is also possible to observe in Figure \ref{fig:lda} that it makes no sense to look for a second independent direction that continues to maximize class separation, while in PCA it always makes sense to look inductively for orthogonal directions that maximize variance. If the dataset shown in the figure had a third class, we could find a second direction that maximizes the separation between classes, thus offering the possibility of projecting onto a plane. In general, we will see that if we have $r$ classes we will be able to find at most (and as long as the dimension of the original space allows it) $r-1$ directions that maximize the separation. This indicates that the projections that LDA is going to learn will be, in general, towards a quite low dimension, and always limited by the number of classes in the dataset.

Suppose we have the labeled dataset $\mathcal{X} = \{x_1,\dots,x_N\} \subset \R^d$, where $\mathcal{C}$ is the set of all the classes in the problem and $y_1,\dots,y_N \in \mathcal{C}$ are the corresponding labels. Suppose that the number of classes in the problem is $|\mathcal{C}| = r$. For each $c \in \mathcal{C}$ we define the set $\mathcal{C}_c = \{ i \in \{1,\dots,N\} \colon y_i = c\}$, and $N_c = |\mathcal{C}_c|$. We consider the mean vector of each class,
\[\mu_c = \frac{1}{N_c} \sum_{i \in \mathcal{C}_c} x_i,\]
and the mean vector for the whole dataset,
\[\mu = \frac{1}{N}\sum_{c \in \mathcal{C}}\sum_{i \in \mathcal{C}_c}x_i = \frac{1}{N}\sum_{i=1}^N x_i. \]

We will define two scatter matrices, one \emph{between-class}, denoted as $S_b$, and the other \emph{within-class}, denoted as $S_w$. The between-class scatter matrix is defined as
\begin{equation*}
    S_b = \sum_{c \in \mathcal{C}} N_c(\mu_c - \mu)(\mu_c - \mu)^T.
\end{equation*}
And the within-class scatter matrix is defined as
\begin{equation*}
    S_w = \sum_{c \in \mathcal{C}} \sum_{i \in \mathcal{C}_c}(x_i- \mu_c)(x_i - \mu_c)^T.
\end{equation*}

Note that these matrices represent, except multiplicative constants, the covariances between the data of different classes, taking the class means as representatives for each class in the first case, and the sum, for each class, of the covariances of that class data, in the second case. Since we want to maximize the separation between classes we will formulate the problem of optimization as the search for a projection $L \in \mathcal{M}_{d' \times d}(\R)$ that maximizes the quotient of the between-class variances and within-class variances determined by the previous matrices. The problem is established as
\begin{equation} \label{eq:lda}
    \max_{\substack{L \in \mathcal{M}_{d'\times d}(\R) }} \quad \tr\left((LS_wL^T)^{-1}(L S_b L^T)\right).
\end{equation}

Theorem \ref{thm:eigen_trace_ratio_opt} assures us that, in order to maximize the problem given in Eq. \ref{eq:lda}, $L$ has to be composed by the eigenvectors corresponding to the largest eigenvalues of $S_w^{-1}S_b$, as long as $S_w$ is invertible. In practice, this happens in most problems where $N \gg d$, because $S_w$ is the sum of $N$ outer products, each of which may add a new dimension to the matrix rank. If $N \gg d$ it is likely that $S_w$ is full-rank. This, together with the fact that $S_w$ is positive semidefinite, would guarantee $S_w$ to be positive definite, thus entering into the theorem hypothesis. 

It is interesting to remark the similarity between the optimization problem in Eq. \ref{eq:lda} and the expression of the Calinski-Harabasz index \cite{calinski1974dendrite}, an index used in clustering to measure the separation of the established clusters, and that uses the same scatter matrices, and a similar quotient formulation.

Furthermore, let us note, as it was already mentioned at the beginning of this section, that at most we can get $r-1$ eigenvectors with a non zero corresponding eigenvalue. This is because the maximum rank of $S_b$ is $r-1$, because its rank coincides with the rank of the matrix $A$ that has as columns the vectors $\mu_c - \mu$ (we get $S_b = A \diag(N_{c_1},\dots,N_{c_r})A^T$), which can have as maximum rank $r$, and this matrix also includes the linear combination $\sum N_c(\mu_c - \mu) = 0$, so at least one column is linearly dependent of the others. Therefore, $S_w^{-1}S_b$ also has a maximum rank of $r-1$. Consequently, the projection matrix that maximizes Eq. \ref{eq:lda} is also going to have, at most, this rank, thus the projection will be contained in a space of this dimension. Therefore, the choice of a dimension $d' > r-1$ will not provide any additional information to that provided by the projection onto dimension $r-1$.

To conclude, although we have seen that LDA allows us to reduce dimensionality by adding supervised information as opposed to the non supervised PCA, it can also present some limitations:

\begin{itemize}
  \item If the size of the dataset is too small, the within-class scatter matrix may be singular, preventing the calculation of $S_w^{-1}S_b$. In this situation, several mechanisms are proposed to keep this technique going. One of the most used consists of regularizing the problem, considering, instead of $S_w$, the matrix $S_w + \varepsilon I$, where $\varepsilon > 0$, making $S_w + \varepsilon I$ be positive definite. The problem of the singularity of $S_w$ also arises if there are correlated attributes. This case can be avoided by eliminating redundant attributes in a preprocessing prior to learning.

  \item The definition of the scatter matrices assumes, to some extent, that the data in each class are distributed according to a multivariate gaussian distribution. Therefore, if the data presented other distributions, the projection learned might not be of enough quality. 

  \item As already mentioned, LDA only allows the extraction of $r-1$ attributes, which may be suboptimal in some cases, as a lot of information could be lost.
\end{itemize}

\subsubsection{Average Neighborhood Margin Maximization (ANMM)} \label{desc:anmm}

ANMM \cite{anmm} is a distance metric learning technique specifically oriented to dimensionality reduction. It therefore follows the same path as the aforementioned PCA and LDA, trying to solve some of their limitations.

The objective of ANMM is to learn a linear transformation $L \in \mathcal{M}_{d' \times d}(\R)$, with $d' \le d$, that projects the data onto a lower dimensional space, so that the similarity between the elements of the same class and the separation between classes is maximized, following the criterion of maximization of margins that we will show next.

We consider the training dataset $\mathcal{X} = \{x_1,\dots,x_N\} \subset \R^d$, with corresponding labels $y_1,\dots,y_N$, and we fix $\xi, \zeta \in \N$, and euclidean distance as the initial distance. From these variables we will create two types of neighborhoods.

\begin{definition}
  Let $x_i \in \mathcal{X}$.

  We define the \emph{$\xi$-nearest homogeneous neighborhood} of $x_i$ as the set of the $\xi$ samples in $\mathcal{X}\setminus\{x_i\}$ nearest to $x_i$ that belong to its same class. We will denote it by $\mathcal{N}_i^o$.

  We define the \emph{$\zeta$-nearest heterogeneous neighborhood} of $x_i$ as the set of the $\zeta$ samples in $\mathcal{X}$ nearest to $x_i$ that belong to a different class. We will denote it by $\mathcal{N}_i^e$. 

\end{definition}

ANMM is intended to maximize the concept of \emph{average neighborhood margin}, which we define below.

\begin{definition}
  Given $x_i \in \mathcal{X}$, its \emph{average neighborhood margin} $\gamma_i$ is defined as
  \begin{equation*}
      \gamma_i = \sum\limits_{k \colon x_k \in \mathcal{N}_i^e} \frac{\|x_i - x_k \|^2}{|\mathcal{N}_i^e|} - \sum\limits_{j \colon x_j \in \mathcal{N}_i^o} \frac{\|x_i - x_j \|^2}{|\mathcal{N}_i^o|}.
  \end{equation*}

  The (global) average neighborhood margin $\gamma$ is defined as
  \begin{equation*}
        \gamma = \sum_{i=1}^N \gamma_i.
    \end{equation*}
\end{definition}

Note that, for each $x_i \in \mathcal{X}$, its average neighborhood margin represents the difference between the average distance from $x_i$ to its heterogeneus neighbors, and the average distance from $x_i$ to its homogeneous neighbors. Therefore, maximizing this margin allows, locally, to move data from different classes away, and pulling those of the same class. Figure \ref{fig:average_neighbor_margin} graphically describes the concept of average neighborhood margin.

\begin{figure}[htbp]
    \centering
    \includegraphics[width=0.7\textwidth]{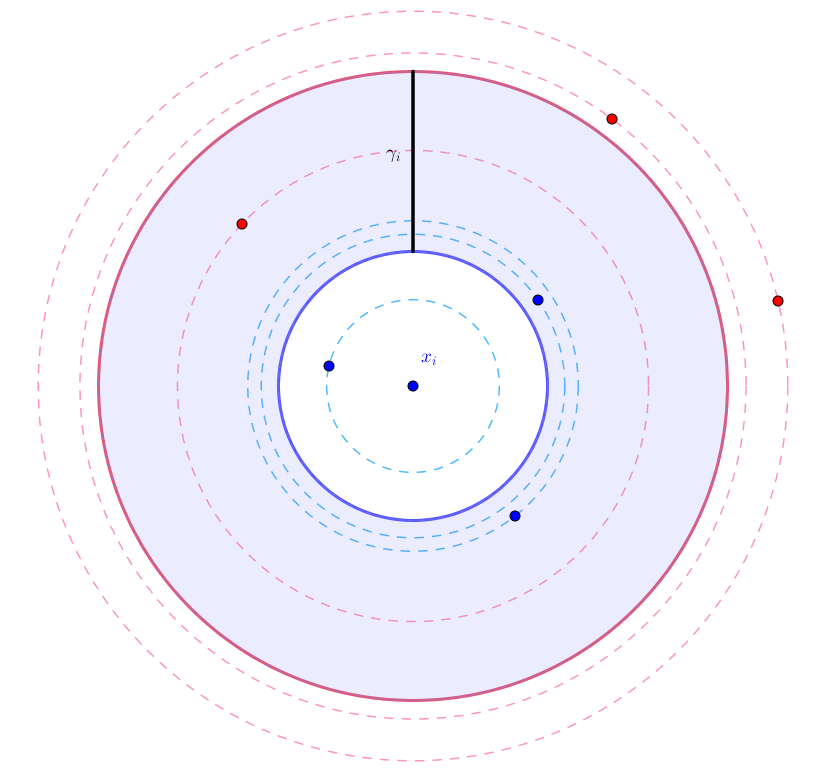}
    \caption{Graphical description of the average neighborhood margin, for the sample $x_i$, for $\xi = \zeta = 3$. The blue and red circumferences determine the average distance from $x_i$ to data of the same and different classes, respectively.} \label{fig:average_neighbor_margin}
\end{figure}

We are now looking for a linear transformation $L$ that maximizes the margin associated with the projected data, $\{Lx_i \colon i = 1,\dots,N\}$. For such data, we have the average neighborhood margin corresponding to that transformation,
\[ \gamma^L = \sum_{i=1}^{N} \gamma_i^L = \sum_{i=1}^N\left( \sum\limits_{k \colon x_k \in \mathcal{N}_i^e} \frac{\|Lx_i - Lx_k \|^2}{|\mathcal{N}_i^e|}- \sum\limits_{j \colon x_j \in \mathcal{N}_i^o} \frac{\|Lx_i - Lx_j \|^2}{|\mathcal{N}_i^o|} \right). \]

Observe that, thanks to the linearity of the trace operator, we can express
\begin{align*}
    \sum_{i=1}^n\sum\limits_{k \colon x_k \in \mathcal{N}_i^e} \frac{\|Lx_i - Lx_k \|^2}{|\mathcal{N}_i^e|} &= \tr\left( \sum_{i=1}^N \sum\limits_{k \colon x_k \in \mathcal{N}_i^e} \frac{(Lx_i - Lx_k)(Lx_i - Lx_k)^T}{|N_i^e|} \right) \\
    &= \tr\left[ L \left( \sum_{i=1}^N \sum\limits_{k \colon x_k \in \mathcal{N}_i^e} \frac{(x_i-x_k)(x_i-x_k)^T}{|N_i^e|}  \right) L^T\right]\\
    &= \tr(LSL^T),
\end{align*}
where $S = \sum_{i}\sum_{k\colon x_k \in \mathcal{N}_i^e}\frac{(x_i-x_k)(x_i-x_k)^T}{|\mathcal{N}_i^e|}$ is called the \emph{scatter matrix}. In a similar way, if we define $C = \sum_{i}\sum_{j\colon x_j \in \mathcal{N}_i^o}\frac{(x_i-x_j)(x_i-x_j)^T}{|\mathcal{N}_i^o|}$, which we will call the \emph{compactness matrix}, we get
\begin{equation*}
    \sum_{i=1}^n\sum\limits_{j \colon x_j \in \mathcal{N}_i^o} \frac{\|Lx_i - Lx_j \|^2}{|\mathcal{N}_i^o|} = \tr(LCL^T).
\end{equation*}

And therefore, combining both expressions,
\begin{equation} \label{eq:margin_caract}
    \gamma^L = \tr(L(S-C)L^T).
\end{equation}

The maximization of $\gamma^L$ as presented in Eq. \ref{eq:margin_caract} is not restrictive enough, because it is enough to multiply $L$ by positive constants to get a value of $\gamma^L$ as large as we want. That is why the constraint $LL^T = I$ is added, so we end up with the next optimization problem:
\begin{align*}
    \max_{L \in \mathcal{M}_{d'\times d}(\R)} &\quad \tr\left(L(S-C)L^T\right)  \\
    \text{s.t.: } &\quad LL^T = I.
\end{align*}

Observe that $S-C$ is symmetric, as it is the difference between two positive semidefinite matrices (each of them is the sum of outer products). Theorem \ref{thm:eigen_trace_opt} tells us that the matrix $L$ we are looking for can be built by adding, by rows, the $d'$ eigenvectors of $S-C$ corresponding to its $d'$ largest eigenvalues.

To conclude, note that ANMM solves some of the issues of the previously mentioned PCA and LDA. On the one hand, it is a supervised learning algorithm, hence it uses the class information that is ignored by PCA. On the other hand, faced with the shortcomings of LDA, we can see that:

\begin{itemize}
  \item It does not have computational problems with small samples, for which scatter or compactness matrices may be singular, because it does not have to calculate their inverse matrices.
  \item It does not make any assumption about the class distributions. The formulation of the problem is purely geometric.
  \item It admits any size for dimensionality reduction. It does not impose that this size must be lower than the number of classes.
\end{itemize}

Finally, we can also observe that, if we keep the maximum dimension $d$, the condition $LL^T = I$ implies that $L$ is orthogonal and $L^TL= I$, thus we are just learning an isometry, as already happened with PCA. Therefore, distance-based classifiers will only be able to experience improvements when the chosen dimension is strictly smaller than the original one.

\subsection{Algorithms to Improve Nearest Neighbors Classifiers} \label{ssec:algs_knn}

In the following paragraphs we will analyze algorithms specifically designed to work with nearest neighbors classifiers. The algorithms we will study are known as LMNN \cite{lmnn} and NCA \cite{nca}.

\subsubsection{Large Margin Nearest Neighbors (LMNN)} \label{desc:lmnn}

LMNN \cite{lmnn} is a distance metric learning algorithm aimed specifically at improving the accuracy of the $k$-nearest neighbors classifier. It is based on the premise that this classifier will label a sample more reliably if its $k$ neighbors share the same label, and to do so it tries to learn a distance that maximizes the number of samples that share its label with as many neighbors as possible.

In this way, the LMNN algorithm tries to minimize an error function that penalizes, on the one hand, the large distance between each sample and those considered its ideal neighbors, and on the other hand, the small distances between examples of different classes.

Suppose we have a dataset $\mathcal{X} = \{x_1,\dots,x_N\} \subset \R^d$ with corresponding labels $y_1,\dots,y_N$. To work, the algorithm makes use of the concept of \emph{target neighbors}. Given a sample $x_i \in \mathcal{X}$, its 
$k$ target neighbors are those examples of the same class as $x_i$ and different from this, for which it is desired to be considered as neighbors in the nearest neighbors classification. If $x_j$ is a target neighbor of $x_i$, then we will write it as $j \istargetof i$. Observe that the relationship given by $\istargetof$ may not be symmetric. Target neighbors are fixed during the learning process. If we have some prior information about our dataset we can use it to determine the target neighbors. Otherwise, a good option is to use the nearest neighbors for the euclidean distance as target neighbors.

Once the target neighbors have been established, for each distance and for each sample in $\mathcal{X}$ we can create a perimeter determined by the the furthest target neighbor. We are looking for distances for which there are no samples of other classes in this perimeter. It is necessary to emphasize that with this perimeter there are not enough separation guarantees, because a feasible distance could have collapsed all the target neighbors in a point, and then the perimeter would have radius zero. For this reason, a margin determined by the radius of the perimeter is considered, to which a positive constant is added. We will see that there is no loss of generality, because of the function that we will define, in supposing that this constant is 1. Any sample of a different class that invades this margin will be called an \emph{impostor}. Our objective, therefore, will be, in addition to bringing each sample as close as possible to its target neighbors, to try to keep impostors as far away as possible.

In mathematical terms, if our distance is determined by the linear transformation $L \in \mathcal{M}_d(\R)$, and $x_i, x_j \in \mathcal{X}$ with $j \istargetof i$, we will say that $x_l$ is an impostor for these samples if $y_l \ne y_i$ and $\|L(x_i - x_j)\|^2 \le \|L(x_i - x_j)\|^2+1$. In Figure \ref{fig:targets_impostors} the concepts of target neighbor and impostor are graphically described. Finally, note that the margin is defined in terms of the squared distances, instead of considering only the distance. This will make the problem formulation easy to solve.

\begin{figure}[htbp]
    \centering
    \subfloat{\fbox{\includegraphics[height = 5cm]{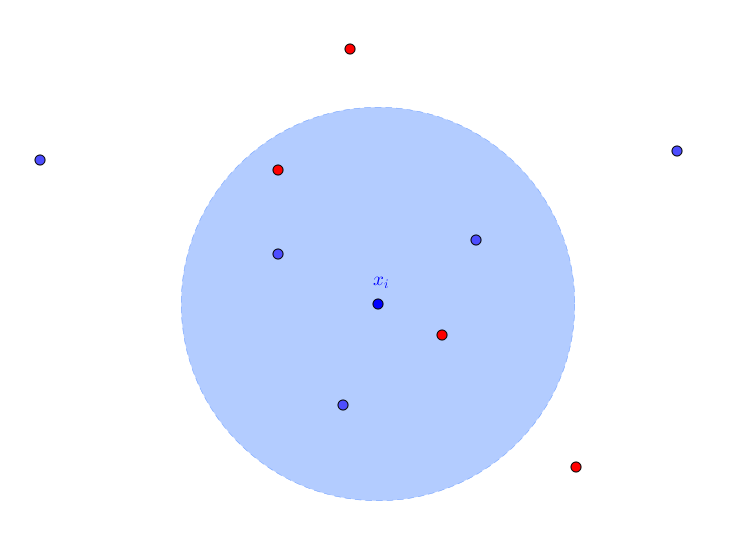}}}
    \subfloat{\fbox{\includegraphics[height = 5cm]{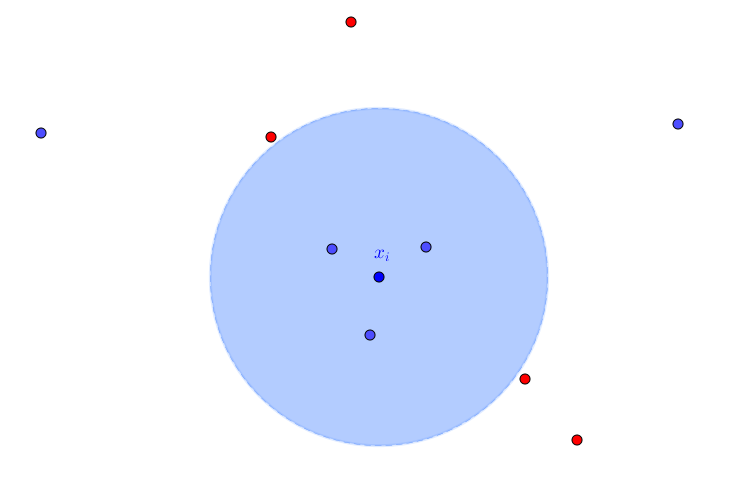}}}

    \caption{Graphical description of target neighbors and impostors (with $k = 3$) for the sample $x_i$. The blue circle represents the margin determined by the target neighbors. All the points of different classes in this circle are impostors. LMNN's goal will be to bring the target neighbors as close as possible and to remove the impostors from the circle. Therefore, data of the same class that are not target neighbors will not have any influence, and impostors will no longer be penalized as soon as they leave the margin, as shown in right image. This gives a local nature to this learning technique.} \label{fig:targets_impostors}
\end{figure}

We now proceed to define accurately the terms of the objective function. As already mentioned, it will be composed of two terms. The first one will penalize distant target neighbors and the second one will penalize nearby impostors. The first term is defined as
\[ \varepsilon_{pull}(L) = \sum_{i=1}^N \sum_{j\istargetof i} \|L(x_i-x_j)\|^2. \]

The minimization of this error causes a pulling force between the data samples. The second term is defined as
\[ \varepsilon_{push}(L) = \sum_{i=1}^{N}\sum_{j\istargetof i}\sum_{l=1}^{N} (1 - y_{il})[1 + \|L(x_i-x_j)\|^2 - \|L(x_i-x_l)\|^2]_{+}, \]
where $y_{il}$ is a binary variable which takes the value $1$ if $y_i = y_l$, and $0$ if $y_i \ne y_l$, and the operator $[\cdot]_{+}\colon \R \to \R^+_0$ is defined as $[z]_+ = \max\{z,0\}$. Thus, this error adds up when $y_{il} = 0$ (that is, $x_l$ is in different class to $x_i$), and the second factor is strictly positive (that is, the margin defined by the target neighbors is exceeded). The minimization of this second term causes a pushing force between the data samples.

Finally, the objective function results from combining these two terms. After fixing $\mu \in ]0,1[$, we define
\begin{equation} \label{eq:lmnn:L}
\varepsilon(L) = (1 - \mu)\varepsilon_{pull}(L) + \mu\varepsilon_{push}(L).
\end{equation}

The authors state that, experimentally, the choice of $\mu$ does not cause great differences in results, so it is usually taken $\mu = 1/2$. Minimizing this function will lead us to learn the distance we were looking for. Note that this function is sub-differentiable, but not convex, so if we use a subgradient descent method under this approach we may be stuck in a local optimal. However, we can reformulate the objective function in order to make it act over the positive semidefinite cone. If for every $L \in \mathcal{M}_d(\R)$ we take $M = L^TL \in S_d(\R)^+_0$, we know that $\|x_i - x_j\|_M^2 = \|L(x_i - x_j)\|_2^2$, and consequently,
\begin{equation} \label{eq:lmnn:M}
 \varepsilon(M) = (1-\mu) \sum_{i=1}^{N}\sum_{j\istargetof i} \|x_i - x_j\|_M^2 + \mu \sum_{i=1}^{N}\sum_{j\istargetof i}\sum_{l=1}^N [ 1 + \|x_i - x_j\|_M^2 - \|x_i - x_l\|_M^2]_{+}
\end{equation}
is a convex function in $M$ that takes the same values as $\varepsilon(L)$. The minimization of $\varepsilon(M)$ in this case is subject to the constraint $M \in S_d(\R)^+_0$, so the projected subgradient method, with projections onto the positive semidefinite cone, can be used to optimize this function. In addition, we can easily calculate a subgradient $G \in \partial \varepsilon / \partial M$ given by
\[ G = (1-\mu) \sum_{i,j\istargetof i} O_{ij} + \mu \sum_{(i,j,l) \in \mathcal{N}} (O_{ij} - O_{il}), \]
where $\mathcal{N}$ is the set of triplets $(i,j,l)$ for which $x_l$ is an impostor over $x_i$ with the margin determined by $x_j$, and $O_{ij} = (x_i - x_j)(x_i - x_j)^T$ are the outer products obtained from the distances differentiation. The first term of the gradient is constant, while the second term only varies in each iteration with the changes of the impostors that enter or leave the set $\mathcal{N}$. These considerations allow a fairly efficient gradient calculation. 

As for dimensionality reduction, two different alternatives are presented. If we keep the optimization with respect to $M$, it is not feasible to add rank restrictions, as it is shown in Example \ref{ex:rank_not_convex}. Therefore, the use of PCA is suggested prior to the algorithm execution, to project the data onto its first principal components, and then apply LMNN on the projected data. The other alternative is to optimize the objective function with respect to $L \in \mathcal{M}_{d' \times d}(\R)$, with $d' < d$ using a gradient descent algorithm. In this case the optimization is not convex, but we learn directly a linear transformation that reduces the dimensionality without making changes in the optimization of Eq. \ref{eq:lmnn:L}. Authors also state, based on empirical results, that this non-convex optimization gives good results.

Other proposals made for the improvement of this algorithm consist of applying LMNN multiple times, learning new metrics each time, and using these metrics to determine increasingly accurate target neighbors, or learning different metrics locally. Finally, although the distance learned by LMNN is designed to be used by the $k$-neighbors classifier, it is possible to use the objective function itself as a classification method. These classification models are called \emph{energy-based}. Thus, to classify a test sample $x_t$, for each possible label value $y_t$, we look for $k$ target neighbors in the training set for class $y_t$, and evaluate the \emph{energy} for the metric learned, finally assigning to $x_t$ the value of $y_t$ that provides the lowest energy. According to the objective function, energy will penalize large distances between $x_t$ and its target neighbors, impostors on the $x_t$ perimeter, and perimeters of other classes invaded by $x_t$. Therefore,
\begin{equation*}
\begin{split}
 y_t^{pred} &= \arg\min_{y_t} \left\{ (1-\mu) \sum_{j \istargetof t} \|x_t-x_j\|_M^2 \right. \\
            &+ \mu \sum_{j \istargetof t,l} (1-y_{tl})\left[ 1 + \|x_t-x_j\|_M^2 - \|x_t-x_l\|_M^2\right]_+  \\
            &+ \left. \mu \sum_{i,j \istargetof i}(1 - y_{it}) \left[ 1 + \|x_i-x_j\|_M^2 - \|x_i-x_t\|_M^2\right]_+ \right\} .
\end{split}
\end{equation*}

\subsubsection{Neighborhood Components Analysis (NCA)} \label{desc:nca}

NCA \cite{nca} is another distance metric learning algorithm aimed specifically at improving the accuracy of the nearest neighbors classifiers. Its aim is to learn a linear transformation with the goal of minimizing the leave-one-out error expected by the nearest neighbor classification. Additionally, this transformation could be used to reduce the dimensionality of the dataset, and thus make the classifier more efficient.

We consider the training set $\mathcal{X} = \{x_1,\dots,x_N\} \subset \R^d$, labeled by $y_1, \dots, y_N$. We want to learn a distance, determined by a linear transformation $L \in \mathcal{M}_d(\R)$, that optimizes the accuracy of the nearest neighbors classifier. Ideally, we would optimize the performance of the classifier over the test dataset, but we only have the training set. Therefore, our goal will be to try to optimize the classification leave-one-out error on the training set. The choice of the leave-one-out error is due to the nature of the nearest neighbors classifier: as we will learn and evaluate over the same set, the nearest neighbor of each sample would be the sample itself, which would not allow the results to be interpreted correctly if the sample is kept while evaluating it.

However, the function that maps each transformation $L$ to the leave-one-out error for the distance corresponding to $L$ has no guarantee of differentiability, not even continuity, so it is not easy to deal with it for optimization (observe that the image of this function is a finite set, and its domain is a connected set, so it cannot be continous unless it is constant, which does not happen in non-trivial examples).

To do this, NCA tries to approach the problem in a stochastic way, that is, instead of operating with the leave-one-out error directly, it operates with its expected value for the probability that we will define below.

Given two samples $x_i, x_j \in \mathcal{X}$, we define the probability that $x_i$ has $x_j$ as its nearest neighbor, for the distance determined by the mapping $L$, as follows:
\begin{equation*}
    \begin{split}
    p_{ij}^L = \frac{\exp\left( - \|Lx_i - Lx_j \|^2 \right)}{\sum\limits_{k \ne i} \exp\left(-\|Lx_i - Lx_k \|^2\right)}\ \ (j \ne i),  
    \end{split}
    \quad\quad
    \begin{split}
    p_{ii}^L = 0.
    \end{split}
\end{equation*}

Notice that, indeed, $p_{i*}$ defines a probability measure on the set $\{1,\dots,N\}$, for each $i \in \{1,\dots,N\}$. Under this probability law, we can define the probability that the sample $x_i$ is correctly classified as the sum of the probabilities that $x_i$ has as its nearest neighbor each sample of its same class, that is
\begin{equation*}
    p_i^L = \sum_{j \in C_i} p_{ij}^L \text{, where } C_i = \{j \in \{1,\dots,N\}\colon y_j = y_i\}.
\end{equation*}

Finally, the expected number of correctly classified samples, and the function we will try to maximize, is obtained as
\begin{equation*}
    f(L) = \sum_{i=1}^N p_i^L = \sum_{i=1}^N \sum_{j \in C_i} p_{ij}^L = \sum_{i=1}^N \sum_{\substack{j \in C_i \\ j \ne i}} \frac{\exp\left(-\|Lx_i - Lx_j \|^2\right)}{\sum\limits_{k \ne i} \exp\left( -\|Lx_i - Lx_k\|^2 \right)}.
\end{equation*}

This function is differentiable, and its derivative can be computed as
\begin{equation*}
   \nabla f(L) = 2L \sum_{i=1}^N \left( p_i^L \sum_{k=1}^N p_{ik}^L O_{ik} - \sum_{j \in C_i} p_{ij}^LO_{ij} \right),
\end{equation*}
where $O_{ij} = (x_i - x_j)(x_i - x_j)^T$ represent again the outer products between the differences of the samples in $\mathcal{X}$. Once the gradient is known, we can optimize the objective function using a gradient ascent method. Note that the objective function is not concave, and can therefore be trapped in local optima. Another issue for this algorithm is the possibility of overfitting, if the expected leave-one-out error of the learned distance is too low. Authors affirm, based on the experimentaal results, that normally there is no overfitting, even if we ascend a lot in the objective function.

\subsection{Algorithms to Improve Nearest Centroids Classifiers} \label{ssec:algs_ncm}

In this block we will analyze, following the previous lines, algorithms specifically oriented to improve distance-based classifiers, focusing in this case on the classifiers based on centroids. The algorithms we will study are NCMML and NCMC \cite{ncmml}.

\subsubsection{Nearest Class Mean Metric Learning (NCMML)} \label{desc:ncmml}

NCMML \cite{ncmml} is a distance metric learning algorithm specifically designed to improve the nearest class mean (NCM) classifier. To do this, it uses a probabilistic approach similar to that used by NCA to improve the accuracy of the nearest neighbors classifier.

Nearest class mean classifier, during learning process, calculates the mean vectors of each class subset. Then, when predicting a new sample, it assigns the class of the nearest mean vector found. It is a very efficient and simple classifier, although its simplicity makes it a rather weak classifier against datasets that are not grouped around their mean. We will learn in the following lines how to learn a distance for this classifier.

We consider the training set $\mathcal{X} = \{x_1,\dots,x_N\} \subset \R^d$, with labels $y_1,\dots,y_N \in \mathcal{C}$, where $\mathcal{C} = \{c_1,\dots,c_r\}$ is the set of available classes. For each $c \in \mathcal{C}$, we call $\mu_c \in \R^d$ the mean vector of the samples belonging to the class $c$, that is, $\mu_c = \frac{1}{N_c}\sum_{i\colon y_i = c}x_i$, where $N_c$ is the number of elements of $\mathcal{X}$ that belong to class $c$. Given a linear transformation $L \in \mathcal{M}_{d'\times d}(\R)$, we will define, for each $x \in \mathcal{X}$ and each $c \in \mathcal{C}$, the probability that $x$ will be labeled with the class $c$ (according to the nearest class mean criterion) as follows:
\begin{equation*}
    p_L(c|x) = \frac{\exp\left(-\frac{1}{2} \|L(x - \mu_c)\|^2\right)}{\sum\limits_{c' \in \mathcal{C}} \exp\left(-\frac{1}{2} \|L(x - \mu_{c'})\|^2\right)}.
\end{equation*} 

Note that $p_L(\cdot|x)$ effectively defines a probability in the set $\mathcal{C}$. Once the above probability is defined, the objective function that NCMML tries to maximize is the log-likelihood for the labeled data in the training set, that is,
\begin{equation*}
\mathcal{L}(L) = \frac{1}{N}\sum_{i=1}^N\log p_L(y_i|x_i).
\end{equation*} 

This function is differentiable and its gradient is given by
\begin{equation*}
\nabla \mathcal{L}(L) = \frac{1}{N} \sum_{i=1}^N \sum\limits_{c\in \mathcal{C}} \alpha_{ic} L (\mu_c - x_i)(\mu_c - x_i)^T,
\end{equation*}
where $\alpha_{ic} = p_L(c|x_i) - [\![ y_i = c ]\!]$ and $[\![ R ]\!]$ denotes the indicator function for the condition $R$. The maximization of this function using gradient methods is the task carried out by NCMML.

\subsubsection{Nearest Class with Multiple Centroids (NCMC)} \label{desc:ncmc}

Although nearest class mean classifier is a simple, intuitive and efficient classifier in both learning and prediction processes, it has one major drawback, and that is that it assumes that classes are grouped around their center, which is an overly restrictive hypothesis. In Figure \ref{fig:problem_ncm} we can see an example where NCM is unable to give good results.

\begin{figure}[htbp]
    \centering
    \includegraphics[width=1.0\textwidth]{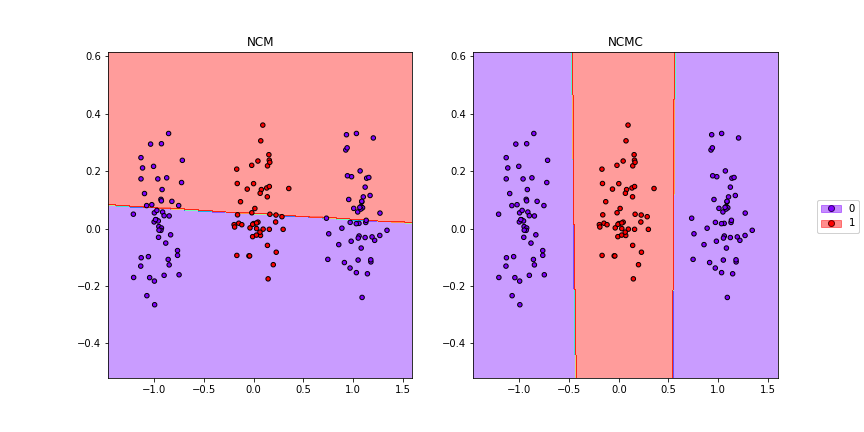}
    \caption{Dataset where the NCM classifier does not provide good results, because the centroids of both classes are very close and both fall between the points of class 1. We will see that, by choosing more than one centroid in an appropiate way, we can classify this set as shown in right image.} \label{fig:problem_ncm}
\end{figure}

One way to solve this problem is, instead of considering the center of each class to classify new samples, to find subgroups within each class that present a quality grouping, and to consider the center for each of its subgroups. In this way we would have a set of centroids for each class, and at the time of classifying a new sample, it would suffice to select the nearest centroid and assign it the class of which it is centroid.

In this new classifier, which we will call NCMC (\emph{nearest class with multiple centroids}), the clustering algorithms come into play. There are numerous algorithms \cite{clustering_algorithms} to obtain a set of clusters from a dataset, each with its advantages and disadvantages. Due to the form of our problem, in which we are interested not only in obtaining a set of clusters for each class, but also a center for each cluster, the algorithm that meets the most suitable conditions, besides being simple and efficient is $k$-Means.

To classify with NCMC, the use $k$-Means reduces to applying the segmentation algorithm within each subset of data associated with each of the classes of the problem. In this way, we obtain in a simple way the set of centroides we wanted for each class, and on which we can carry out the classification of new data simply by searching for the nearest centroid. For this algorithm, as it happens with $k$-Means, it is necessary to previously establish the number of centroids for each class. These numbers can be estimated by cross validation.

Once the NCMC classifier is defined, the distance learning process \cite{ncmml} is similar to NCM. Following the notation used in NCMML, in this case, instead of a set of class centers $\{\mu_c\}$, with $c \in \mathcal{C}$, we have a set of centroids, $\{m_{c_j}\}_{j=1}^{k_c}$, with $k_c \in \N$, for each $c \in \mathcal{C}$. In this case, the probabilities associated with each class for the correct prediction of $x \in \mathcal{X}$ are given by $p_L(c|x) = \sum_{j=1}^{k_c} p_L(m_{c_j}|x)$, where the centroids are those whose probability is defined by the softmax function
\begin{equation*}
    p_L(m_{c_j}|x) = \frac{\exp\left( -\frac{1}{2} \|L(x-m_{c_j})\|^2 \right)}{ \sum\limits_{c \in \mathcal{C}} \sum\limits_{i=1}^{k_c} \exp\left( -\frac{1}{2} \|L(x-m_{c_i})\|^2 \right) }.
\end{equation*}

Again, we maximize the log-likelihood function $\mathcal{L}(L) = \frac{1}{N}\sum_{i=1}^N p_L(y_i|x_i)$, whose gradient is given by
\begin{equation*}
    \nabla \mathcal{L}(L) = \frac{1}{N} \sum_{i=1}^N \sum_{c \in \mathcal{C}} \sum_{j=1}^{k_c} \alpha_{ic_j} L (m_{c_j}-x_i)(m_{c_j}-x_i)^T,
\end{equation*}
where \[\alpha_{ic_j} = p_L(m_{c_j}|x_i) - [\![ y_i = c ]\!] \frac{p_L(m_{c_j}|x_i)}{\sum_{j'=1}^{k_c} p_L(m_{c_{j'}}|x_i)}.\]

The log-likelihood maximization by gradient methods is the task carried out by the distance learning technique for NCMC classifier, which we will call with the same name as the classifier.

\subsection{Information Theory Based Algorithms} \label{ssec:algs_information_theory}

In this section we will study several distance metric learning algorithms based on information theory, specifically, in the Kullback-Leibler and Jeffrey divergences. Their working scheme is similar. First of all, they establish different probability distributions on the data, and then they try to bring them closer or further away by using the divergences. The algorithms we will study are ITML \cite{itml}, DMLMJ \cite{dmlmj} and MCML \cite{mcml}.

\subsubsection{Information Theoretic Metric Learning (ITML)} \label{desc:itml}

ITML \cite{itml} is a distance metric learning technique whose objective is to find a metric as close as possible to an initial distance, understanding this closeness from the point of view of relative entropy, as we will formulate later, making that metric satisfy certain similarity constraints for the trained data.

ITML starts with a dataset $\mathcal{X} = \{x_1,\dots,x_N\} \subset \R^d$, not necessarily labeled, but for which it is known that certain pairs of samples considered similar must be at a distance lower than or equal to $u$, and other pairs of samples considered not similar must be at a distance greater than or equal to $l$, where $u, l \in \R^+$ are pre-defined constants, with relative small and large values, respectively, with respect to the dataset.

From the data with the indicated restrictions, ITML considers an initial distance corresponding to a positive definite matrix $M_0$, and tries to find a positive definite matrix $M$, as similar as possible to $M_0$, and that respects the imposed similarity constraints. The way to measure the similarity between $M$ and $M_0$ is done using information theory tools.

As we saw in Appendix \ref{ssec:information_theory}, there is a correspondence between positive definite matrices and multivariate gaussian distributions, if we fix the same mean vector $\mu$ for every distribution. Given $M \in S_d(\R)^+$ we can then construct a normal distribution through its density function,
\[ p(x|M) = \frac{1}{(2\pi)^{n/2}\det(M)^{1/2}}\exp\left( (x-\mu)^TM^{-1}(x-\mu) \right). \]

Reciprocally, from this distribution, if we calculate the covariance matrix, we recover the matrix $M$. Using this correspondence, we will measure the closeness between $M_0$ and $M$ through the Kullback-Leibler divergence between their corresponding gaussian distributions, that is,
\[ \kl(p(x|M_0))\|p(x|M)) = \int p(x|M_0)\log\frac{p(x|M_0)}{p(x|M)}dx. \]

Once we have defined the mechanism to measure the proximity between the metrics, we can formulate the optimization problem of the technique ITML. If we call $S$ and $D$ to the sets of pairs of indices on the elements of $\mathcal{X}$ that represent the samples considered similar and not similar, respectively, and we start from the initial metric $M_0$, the problem is
\begin{equation} \label{eq:itml:prob1}
    \begin{split}
    \min_{M \in S_d(\R)^+} &\quad \kl(p(x|M_0)\|p(x|M))  \\
    \text{s.t.: } &\quad d_M(x_i,x_j) \le u, \quad (i,j) \in S \\
                  &\quad d_M(x_i,x_j) \ge l, \quad (i,j) \in D.
    \end{split}
\end{equation}

We have seen in Theorem \ref{cor:kl_gaussian} that the Kullback-Leibler divergence between two gaussian distributions with the same mean can be expressed in terms of the \emph{log-det}  matrix divergence. This allows us to reformulate Eq. \ref{eq:itml:prob1} in a way that is easier to deal with computationally:
\begin{equation} \label{eq:itml:prob2}
    \begin{split}
    \min_{M \in S_d(\R)^+} &\quad D_{ld}(M_0\|M)  \\
    \text{s.t.: } &\quad \tr(M(x_i-x_j)(x_i-x_j)^T) \le u, \quad (i,j) \in S \\
                  &\quad \tr(M(x_i-x_j)(x_i-x_j)^T) \ge l, \quad (i,j) \in D.
    \end{split}
\end{equation}

We may not be able to find a metric $M$ that simultaneously satisfies every constraint, so the problem may not have a solution. Therefore, ITML introduces in Eq. \ref{eq:itml:prob2} slack variables through which we obtain a problem whose optimization establishes a trade-off between the minimization of the divergence and the fulfillment of the constraints, in order to arrive to an approximate solution of the original problem, in case there is no solution for this. Finally, the computational technique used in the resolution of this optimization problem is the \emph{Bregman projections method} discussed in Appendix \ref{sssec:opt_methods}.

\subsubsection{Distance Metric Learning through the Maximization of the Jeffrey Divergence (DMLMJ)} \label{desc:dmlmj}

DMLMJ \cite{dmlmj} is another distance metric learning technique based on information theory. In this case, the tool that is used by DMLMJ is the Jeffrey divergence, to separate as much as possible the distribution associated with similar points from that associated to dissimilar points, in the sense that we will see below.

We consider the training set $\mathcal{X} = \{x_1,\dots,x_N\} \subset \R^d$ with corresponding labels $y_1,\dots,y_N$, and we set $k \in \N$. As we have already commented, DMLMJ tries to maximize, with respect to the Jeffrey divergence, the separation between distributions of similar and not similar points. To do this, we will introduce several concepts.

\begin{definition}
Given $x_i \in \mathcal{X}$, the \emph{$k$-positive neighborhood} of $x_i$ is defined as the set of the $k$ nearest neighbors of $x_i$ in $\mathcal{X}\setminus\{x_i\}$ whose class is the same as $x_i$. It is denoted by $V_k^+(x_i)$.

The $k$-negative neighborhood of $x_i$ is defined as the set of the $k$ nearest neighbors of $x_i$ in $\mathcal{X}$ whose class is different from that of $x_i$. It is denoted by $V_k^-(x_i)$.

The \emph{$k$-positive difference space} of the labeled dataset is defined as the set
\[ S = \{x_i - x_j \colon x_i \in \mathcal{X}, x_j \in V_k^+(x_i)\}. \]

Similarly, the \emph{$k$-negative difference space} of the labeled dataset is defined as the set
\[ D = \{x_i - x_j \colon x_i \in \mathcal{X}, x_j \in V_k^-(x_i) \}. \]

\end{definition}

Sets $S$ and $D$ represent, therefore, the vectors with the differences between the samples in $\mathcal{X}$ and its $k$ nearest neighbors, from the same or a different class, respectively. We refer to $P$ and $Q$ as the probability distributions in the spaces $S$ and $D$, respectively, assuming that they are multivariate gaussians. We will also assume that both distributions have zero mean. This assumption is reasonable, since in practice, in most cases, if $x_i$ is a neighbor of $x_j$, $x_j$ is also a neighbor of $x_i$, then both differences will appear in the difference space, averaging zero. Finally, we will call the corresponding covariance matrices $\Sigma_S$ and $\Sigma_D$, respectively.

If we now apply a linear transformation to the data, $x \mapsto Lx$, with $L \in \mathcal{M}_{d' \times d}(\R)$, the transformed distributions will still have mean zero, and covariances $L\Sigma_S L^T$ and $L\Sigma_D L^T$, respectively. We will call these distributions $P_L$ and $Q_L$. The goal of DMLMJ is to find a transformation that maximizes the Jeffrey divergence between $P_L$ and $Q_L$, that is, the problem to optimize is:
\[ \max_{L \in \mathcal{M}_{d'\times d}(\R)} \quad f(L) =  \jf(P_L\|Q_L) = \kl(P_L\|Q_L) + \kl(Q_L\|P_L).\]

As it was shown in Proposition \ref{cor:jf_gaussian}, Jeffrey divergence between the gaussian distributions $P_L$ and $Q_L$ can be rewritten as
\[ f(L) = \frac{1}{2} \tr\left( (L\Sigma_S L^T)^{-1} (L\Sigma_D L^T) + (L \Sigma_D L^T)^{-1} (L \Sigma_S L^T) \right) - d'. \]

Since $d'$ is constant, we obtain the equivalent problem
\begin{align*}
\max_{L \in \mathcal{M}_{d'\times d}(\R)} \quad J(L) &=  \tr\left( (L\Sigma_S L^T)^{-1} (L\Sigma_D L^T) + (L \Sigma_D L^T)^{-1} (L \Sigma_S L^T) \right).
\end{align*}

Theorem \ref{thm:eigen_trace_ratio_sym_opt} tells us that, to maximize $J(L)$, we can choose the $d'$ eigenvectors of $\Sigma_S^{-1}\Sigma_D$, $v_1,\dots,v_{d'}$ corresponding to the largest values of $\lambda_i + 1/\lambda_i$, with $\lambda_i$ being the eigenvalue of $\Sigma_S^{-1}\Sigma_D$ associated with $v_i$, and add this eigenvectors to the rows of $L$. The transformation $L$ constructed from these eigenvectors determines the distance that is learned by the DMLMJ technique.

Finally, the only additional requirement necessary to complete the construction of $L$ is the calculation of the covariance matrices $\Sigma_S$ and $\Sigma_D$. Bearing in mind that it has been assumed that the mean of the distributions of $S$ and $D$ is 0, we can obtain these matrices quite simply from the difference vectors, as shown below:
\begin{align*}
    \Sigma_S &= \frac{1}{|S|}\sum_{i=1}^{N} \left[ \sum_{x_j \in V_k^+(x_i)} (x_i-x_j)(x_i-x_j)^T\right], \\
    \Sigma_D &= \frac{1}{|D|}\sum_{i=1}^{N} \left[ \sum_{x_j \in V_k^-(x_i)} (x_i-x_j)(x_i-x_j)^T\right].
\end{align*}

Let us observe that we can also see this algorithm as a dimensionality reduction algorithm and even as an algorithm oriented to improve the nearest neighbors classifier, due to its local character.

\subsubsection{Maximally Collapsing Metric Learning (MCML)} \label{desc:mcml}

MCML \cite{mcml} is a supervised distance metric learning technique, based on the idea that if all the samples of the same class were projected to the same point, and data of different classes were projected to different points and sufficiently far away, we would have, over the projected data, an ideal class separation. Its purpose is to learn a distance metric that allows to collapse as much possible, within the limitations of the metric, all the samples of the same class in a single point, arbitrarily far from the points where the samples of the remaining classes will collapse.

We consider the dataset $\mathcal{X} = \{x_1,\dots,x_N\} \subset \R^d$, with corresponding labels $y_1,\dots,y_N$. We want to learn a metric determined by $M \in S_d(\R)^+$ that tries to collapse the classes as much as possible according to the approach of the previous paragraph. The way to deal with this problem will consist once again in using the tools provided by the information theory. To do this, we first introduce a conditional distribution on the points of the dataset, analogous to that established in the case of NCA. If $i, j \in \{1,\dots,N\}$, with $i \ne j$, we define the probability that $x_j$ will be classified with the class of $x_i$ according to the distance between $x_i$ and $x_j$ as follows:
\begin{equation*}
    p^{M}(j|i) = \frac{\exp(-\|x_i-x_j\|^2_M)}{\sum\limits_{k\ne i} \exp(-\|x_i-x_k\|^2_M)}.
\end{equation*}

Furthermore, the ideal distribution we are looking for is a binary distribution for which the probability that a sample is correctly classified is 1, and 0 otherwise, that is,
\begin{equation*}
    p_0(j|i) \propto \begin{cases}1, &\quad y_i = y_j \\ 0, &\quad y_i \ne y_j\end{cases}.
\end{equation*}

Note that during the training process we know the real classes of the data, therefore we can deal with this last probability. Besides, we can observe that if we get a metric $M$ whose associated distribution $p^M$ coincides with $p_0$, then, under very mild sufficiency conditions on the data, we will be able to collapse the classes in infinitely distant points.

Indeed, suppose there are at least $r+2$ samples in each class, were $r$ is the rank of $M$, and that $p^M(j|i) = p^0(j|i)$ for any $i, j \in \{1,\dots,N\}$. Then, on the one hand, from $p^M(j|i) = 0$ for $y_i \ne y_j$, it follows that $\exp(-\|x_i-x_j\|^2_M) = 0$, which undoubtedly leads to $x_i$ and $x_j$ being infinitely distant when their classes are different. On the other hand, from $p^M(i|j) \propto 1$ for any $x_i, x_j$ with $y_i = y_j$, it follows that the value $\exp(-\|x_i-x_j\|^2_M)$ is constant for all the members of the same class, and consequently, all the points in the same class are equidistant. As $M$ has rank $r$, it is inducing a distance on a subspace of dimension $r$, where it is known that at most there can be $r+1$ different points and equidistant between them. Since we are assuming that there are at least $r+2$ points per class, all the points of the same class must have a distance of 0 between them with respect to $M$, thus collapsing into a single point. 

Once both distributions are set, the objective of MCML is, as we have already commented, to approximate $p^M(\cdot|i)$ to $p_0(\cdot|i)$ as much as possible, for each $i$, using the relative entropy between both distributions. The optimization problem is, therefore, to minimize this divergence,
\begin{equation*}
    \min_{M \in S_d(\R)^+_0}\quad f(M) = \sum_{i=1}^N \kl \left[ p_0(\cdot|i) \| p^M(\cdot|i) \right].
\end{equation*} 

We can rewrite the objective function in terms of elementary functions:
\begin{equation} \label{eq:mcml:fobj2}
    \begin{split}
        f(M) &= \sum_{i=1}^N \sum_{j=1}^N p_0(j|i) \log \frac{p_0(j|i)}{p^M(j|i)} = \sum_{i=1}^N \sum_{j \colon y_i = y_j}\log\frac{1}{p^M(j|i)} \\
             &= \sum_{i=1}^N \sum_{j \colon y_i = y_j}-\log p^M(j|i) \\
             &= - \sum_{i=1}^N \sum_{j \colon y_i = y_j}\left(-\|x_i-x_j\|_M^2-\log\sum_{k \ne i} \exp(-\|x_i-x_k\|^2)\right)\\
             &= \sum_{i=1}^N \sum_{j \colon y_i = y_j}\|x_i-x_j\|_M^2+\sum_{i=1}^N\log\sum_{k \ne i} \exp(-\|x_i-x_k\|^2).
    \end{split}
\end{equation}

This function is differentiable, and each summand of the previous expression is convex in $M$, the first because it is a distance function in $M$ (which is affine), and the second because it is a \emph{log-sum-exp} function (see \cite{convexoptimization}, sec.~3.1.5) composed with a distance function. In addition, the restriction $M \in S_d(\R)^+_0$ is convex, so we can use the projected gradient descent algorithm with projections onto the positive semidefinite cone to optimize the objective function. This requires an expression of the gradient of the objective function, which can be calculated from its expression in Eq. \ref{eq:mcml:fobj2}:
\begin{equation*}
    \nabla f(M) = \sum_{i,j\colon y_i = y_j}(x_i - x_j)^T(x_i-x_j) - \sum_i \frac{-\sum\limits_{k \ne i} (x_i-x_k)^T(x_i-x_k) \exp(-\|x_i-x_k\|_M^2)}{ \sum\limits_{k \ne i} \exp(-\|x_i-x_k\|^2_M)}.
\end{equation*}

\subsection{Other Distance Metric Learning Techniques} \label{ssec:other_dmls}

In this section we will study some different proposals for distance metric learning techniques. The algorithms we will analyze are LSI \cite{lsi}, DML-eig \cite{dmleig} and LDML \cite{ldml}.

\subsubsection{Learning with Side Information (LSI)} \label{desc:lsi}

LSI \cite{lsi}, also sometimes referred to as MMC (\emph{Mahalanobis metric for clustering}) is a distance metric learning technique that works with a dataset that is not necessarily labeled, which contains certain pairs of samples that are known to be similar and, optionally, pairs of samples that are known not to be similar. It is possibly one of the first algorithms that has helped make the concept of distance metric learning more well known.

LSI tries to learn a metric $M$ that respects this additional information. This is why it can be used both in supervised learning, where similar pairs will correspond to data with the same label, and in unsupervised learning with similarity constraints, such as, for example, clustering problems where it is known that certain samples must be grouped in the same cluster.

We now formulate the problem to be optimized by LSI. Suppose we have the dataset $\mathcal{X} = \{x_1,\dots,x_N\} \subset \R^d$, and we know additionally the set $S = \{(x_i, x_j) \in \mathcal{X}\times \mathcal{X} \colon x_i \textit{ and } x_j \ \allowbreak\textit{are similar.}\}$. In addition, we may know the set $D =  \{(x_i,x_j) \in \mathcal{X}\times\mathcal{X} \colon x_i \textit{ and } x_j \ \allowbreak\textit{are dissimilar.} \}$. If we do not have the latter, we can take $D$ as the complement of $S$ in $\mathcal{X}\times\mathcal{X}$.

The first intuition to address this problem, given the information we have, is to minimize the distances between pairs of similar points, that is, to minimize $\sum_{(x_i,x_j)\in S} \|x_i - x_j \|_M^2$, where $M \in S_d(\R)^+_0$. However, this will lead us to the solution $M = 0$, which would not give us any productive information. That is why LSI adds the additional constraint $\sum_{(x_i,x_j) \in D} \|x_i - x_j\|_M \ge 1$, which leads us to the optimization problem
\begin{equation*} \label{eq:lsi}
\begin{split}
    \min_{M} &\quad \sum_{(x_i,x_j)\in S}  \|x_i - x_j \|_M^2 \\
    \text{s.t.: } &\quad \sum_{(x_i,x_j) \in D} \|x_i - x_j\|_M \ge 1 \\
                  &\quad M \in S_d(\R)^+_0.
\end{split}
\end{equation*}

Note several observations regarding this formula. First, the choice of constant 1 in the constraint is irrelevant; if we choose any constant $c > 0$ we get a metric proportional to $M$. Secondly, the optimization problem is convex, because the sets determined by the restrictions are convex and the function to optimize is also convex. Finally, we may consider a restriction on the set $D$ of the form $\sum_{(x_i,x_j) \in D} \|x_i - x_j\|_M^2 \ge 1$. However, it is possible to rewrite that problem into a formulation similar to that used on the 2-class LDA, where the metric learned would have a rank of 1, which may not be optimal.

To easily optimize this problem, authors propose the equivalent problem
\begin{equation} \label{eq:lsi:equiv}
\begin{split}
    \max_{M} &\quad \sum_{(x_i,x_j)\in D}  \|x_i - x_j \|_M \\
    \text{s.a.: } &\quad \sum_{(x_i,x_j) \in S} \|x_i - x_j\|_M^2 \le 1 \\
                  &\quad M \in S_d(\R)^+_0.
\end{split}
\end{equation}

This problem with two convex constraints can be solved by a projected gradient ascent method. In this problem, constraints are easy to satisfy separately. The first constraint consists of a projection onto an affine half-space, while the second constraint consists of a projection onto the positive semidefinite cone. The method of iterated projections makes it possible to fulfill both restrictions by repeteadly projecting onto both sets until convergence is obtained.

\subsubsection{Distance Metric Learning with Eigenvalue Optimization (DML-eig)} \label{desc:dmleig}

DML-eig \cite{dmleig} is a distance metric learning algorithm inspired by the LSI algorithm of the previous section, proposing a very similar optimization problem but offering a completely different resolution method, based on eigenvalue optimization.

We consider, as in the previous case, a training dataset $\mathcal{X} = \{x_1,\dots,x_N\} \subset \R^d$, for which we know two sets of pairs, $S$ and $D$, of data considered similar and dissimilar, respectively. In the previous section, in order to optimize Eq. \ref{eq:lsi:equiv} an ascending gradient method with iterated projections was proposed, which may take a long time to converge. DML-eig proposal consists of a slight modification of the objective function, keeping the same constraints, which leads us to the problem
\begin{equation} \label{eq:dmleig:1}
\begin{split}
    \max_{M} &\quad \min_{(x_i,x_j)\in D}  \|x_i - x_j \|_M^2 \\
    \text{s.t.: } &\quad \sum_{(x_i,x_j) \in S} \|x_i - x_j\|_M^2 \le 1 \\
                  &\quad M \in S_d(\R)^+_0.
\end{split}
\end{equation}

To address this problem, it is useful to introduce a notation that simplifies the indexing of the data. First, we will denote $X_{ij} = (x_i - x_j)(x_i-x_j)^T$ to the outer products between the differences of the elements in $\mathcal{X}$. To access pairs of elements $(i,j)$ we will use a single index $\tau \equiv (i,j)$. This index can be assumed ordered when necessary, to access the components of a vector of appropiate size. The previous outer product $X_{ij}$ can also be written as $X_{\tau}$. Finally, for sets $S$ and $D$, we also assume that they are made by indexes $\tau$ associated with a pair $(i,j)$ such that $x_i$ and $x_j$ are similar or dissimilar, respectively. Thus, if we denote $X_S = \sum_{(i,j)\in S}X_{ij}$, Eq. \ref{eq:dmleig:1} can be rewritten in terms of Frobenius dot product as
\begin{equation} \label{eq:dmleig:2}
\begin{split}
    \max_{M} &\quad \min_{\tau \in D}  \langle X_{\tau}, M \rangle \\
    \text{s.t.: } &\quad \langle X_S, M \rangle \le 1 \\
                  &\quad M \in S_d(\R)^+_0.
\end{split}
\end{equation}

Let us see how the formulation of the problem we are looking for is established in terms of eigenvalue optimization. For each symmetric matrix $X \in S_d(\R)$ we denote its highest eigenvalue as $\lambda_{\max}(X)$. Associated with the set $D$ of dissimilar pairs we will define the simplex
\[ \Delta = \left\{u \in \R^{|D|} \colon u_\tau \ge 0 \ \forall \tau \in D, \sum_{\tau \in D} u_{\tau} = 1 \right\}. \]

We also consider the set
\[ \mathcal{P} = \{M \in \mathcal{M}_d(\R)^+_0 \colon \tr(M) = 1 \}.  \]

$\mathcal{P}$ is the intersection of the positive semidefinite cone with an affine subspace of $\mathcal{M}_d(\R)$. Sets with this structure are known as \emph{spectrahedra}.

So, if $X_S$ is positive semidefinite, and we define, for each $\tau \in D$, $\widetilde{X}_{\tau} = X_S^{-1/2}X_{\tau}X_S^{-1/2}$, we can prove \cite{dmleig} that the problem given by Eq. \ref{eq:dmleig:2} is equivalent to the following problem:
\begin{equation*} \label{eq:dmleig:3}
      \max_{S \in \mathcal{P}} \min_{u \in \Delta} \sum_{\tau \in D} u_{\tau}\langle \widetilde{X}_{\tau},S\rangle,
\end{equation*}
which in turn can be rewritten as an eigenvalue optimization problem:
\begin{equation} \label{eq:dmleig:4}
      \min_{u \in \Delta} \max_{S \in \mathcal{P}} \left\langle \sum_{\tau \in D} u_{\tau} \widetilde{X}_{\tau},S\right\rangle = \min_{u \in \Delta} \lambda_{\max}\left( \sum_{\tau\in D}u_{\tau}\widetilde{X}_{\tau} \right).
\end{equation}

The problem of minimizing the largest eigenvalue of a symmetric matrix is well-known and there are some iterative methods that allow this minimum to be reached \cite{overton1988minimizing}. Furthermore, \citet{dmleig} also propose an algorithm to solve the problem  $\max_{S \in \mathcal{P}} \min_{u \in \Delta} \sum_{\tau \in D} u_{\tau}\langle \widetilde{X}_{\tau},S\rangle + \mu \sum_{\tau \in D} u_{\tau} \log u_{\tau}$, where $\mu > 0$ is a smoothing parameter, by means of which the problem in Eq. \ref{eq:dmleig:4} can be approximated.

\subsubsection{Logistic Discriminant Metric Learning (LDML)} \label{desc:ldml}

LDML \cite{ldml} is a distance metric learning algorithm in which the optimization model makes use of the logistic function. Authors affirm that this technique is quite useful to learn distances on sets of labeled images, being able to be used therefore in problems like face identification.

Recall that the \emph{logistic} or \emph{sigmoid} function is the map $\sigma \colon \R \to \R$ given by
\[ \sigma(x) = \frac{1}{1+e^{-x}}. \]

This function presents a graph with a sigmoidal shape, is differentiable, strictly increasing and takes values between 0 and 1, reaching these values in their limits at infinity. These properties allow the logistic function to be the cumulative distribution function of a random variable, which gives it an important probabilistic utility. Its graph presents an asymptotic behaviour from small values (in absolute value), with an exponential growth in zones close to zero.  This makes logistic function very useful for modeling binary signals. It also presents a derivative that is easy to calculate, and can be expressed in terms of the logistic function itself, $\sigma'(x) = \sigma(x)(1-\sigma(x))$.

Suppose we have the dataset $\mathcal{X} = \{x_1,\dots,x_N\} \subset \R^d$, with corresponding labels $y_1,\dots,y_N$. In LDML, logistic function is used to define a probability, which will assign the greater probability the smaller the distance between points. To measure the distance, LDML will use a positive semidefinite matrix, resulting in the expression of the probability as
\begin{equation*}
    p_{ij,M} = \sigma(b - d_M(x_i,x_j)^2),
\end{equation*} 
where $b$ is a positive threshold value that will determine the maximum value achievable by the logistic function, and that can be estimated by cross validation. Associated with this probability, we can define a random variable that follows a Bernouilli distribution, and that takes the values 0 and 1, according to whether the pair $(x_i,x_j)$ belongs to the same class. This distribution is determined by the probability mass function
\[ f_{ij,M}(x) = (p_{ij,M})^{x}(1-p_{ij,M})^{1-x}, \quad x  \in \{0,1\}. \]

The function that LDML tries to maximize is the log-likelihood of the previous distribution for the given dataset, that is,
\begin{equation*}
    \mathcal{L}(M) = \sum_{i,j=1}^N y_{ij}\log p_{ij,M} + (1-y_{ij})\log(1-p_{ij,M}),
\end{equation*}
where $y_{ij}$ is a binary variable that takes the value 1 if $y_i = y_j$ and 0 otherwise. This function is differentiable and concave (it is a positive combination of functions that can be expressed as a minus log-sum-exp function, which is concave), so we have a convex maximization problem. Keeping in mind the properties of the logistic function, if  $x_{ij} \equiv (x_i-x_j)(x_i-x_j)^T$ and $p_{ij} \equiv p_{ij,M}$, the gradient has the expression
\begin{align*}
    \mathcal{\nabla L}(M) &= \sum_{i,j=1}^N y_{ij}\frac{-x_{ij} p_{ij}(1 - p_{ij})}{p_{ij}} + (1-y_{ij})\frac{x_{ij}p_{ij}(1-p_{ij})}{1-p_{ij}} \\
                          &= \sum_{i,j=1}^N -y_{ij}x_{ij}(1 - p_{ij}) + (1-y_{ij})x_{ij}p_{ij} \\
                          &= \sum_{i,j=1}^N x_{ij}((1-y_{ij})p_{ij}-(1-p_{ij})y_{ij}) \\
                          &= \sum_{i,j=1}^N x_{ij}(p_{ij}-y_{ij}),
\end{align*}
The projected gradient method with projections onto the positive semidefinite cone is the semidefinite programming algorithm that is used in LDML to obtain the metric that optimizes its objective function.

\subsection{Kernel Distance Metric Learning} \label{ssec:kernel_dml}

In this part we will analyze some of the kernelized versions of the algorithms presented throughout this section. An introduction to the use of the kernel trick for distance metric learning was already made in Section \ref{sssec:kernel_dml}. Below we will study the kernel algorithms for LMNN, ANMM, DMLMJ and LDA.


\subsubsection{Kernel Large Margin Nearest Neighbors (KLMNN)} \label{desc:klmnn}

KLMNN \cite{klmnn,lmnn} is the kernelized version of LMNN. In it, the data in $\mathcal{X}$ is sent to the feature space to learn in that space a distance that minimizes the objective function set in the LMNN problem.

Although the problem formulated in the non-kernelized version was made with respect to a positive semidefinite matrix $M$, using the error function given in Eq. \ref{eq:lmnn:M}, when working in feature spaces we are more interested in dealing with a linear map, even if the convexity of the problem is lost, in order to be able to use the representer theorem. Therefore, adapting the error function proposed in Eq. \ref{eq:lmnn:L} to the feature space, the LMNN problem for the kernelized version consists of
\begin{equation*} \label{eq:klmnn:L}
\begin{split}
    \min_{L\in \mathcal{L}(\mathcal{F},\R^{d})} \quad \varepsilon(L) &= (1-\mu)\sum_{i=1}^N\sum_{j \istargetof i} \|L(\phi(x_i) - \phi(x_j))\|^2 \\
                &+ \mu \sum_{i=1}^N\sum_{j \istargetof i} \sum_{l=1}^N(1-y_{il})[1 + \|L(\phi(x_i) - \phi(x_j))\|^2-\|L(\phi(x_i) - \phi(x_l))\|^2]_+. 
 \end{split}
\end{equation*}

As a consequence of the representer theorem, it follows that, for each $x_i \in \mathcal{X}$, $L\phi(x) = AK_{.i}$, where $A \in \mathcal{M}_{d' \times N}(\R)$ is the matrix given by the representer theorem, and $K_{.i}$ represents the $i$-th column of the kernel matrix for the training set. Using this in the error expression, we obtain
\begin{align*}
& (1-\mu)\sum_{i=1}^N\sum_{j \istargetof i} \|L(\phi(x_i) - \phi(x_j))\|^2 \\
                & \quad + \mu \sum_{i=1}^N\sum_{j \istargetof i} \sum_{l=1}^N(1-y_{il})[1 + \|L(\phi(x_i) - \phi(x_j))\|^2-\|L(\phi(x_i) - \phi(x_l))\|^2]_+ \\
&= (1-\mu)\sum_{i=1}^N\sum_{j \istargetof i} \|A(K_{.i} - K_{.j})\|^2 \\
                & \quad+ \mu \sum_{i=1}^N\sum_{j \istargetof i} \sum_{l=1}^N(1-y_{il})[1 + \|A(K_{.i} - K_{.j})\|^2-\|A(K_{.i} - K_{.l})\|^2]_+.
 \end{align*}

The above expression depends only on $A$ and kernel functions, and minimizing it as a function of $A$ (we will denote it $\varepsilon(A)$) we get the same value as minimizing $\varepsilon(L)$. Note also that the expression $\varepsilon(A)$ also requires the calculation of target neighbors and impostors, but these depend only on the distances in the feature space, which, as we have already seen, are computable, as shown in Eq. \ref{eq:dist_features}. Therefore, all the components of $\varepsilon(A)$ are computationally manipulable, so if we apply a gradient descent method on $\varepsilon(A)$ we can reduce the value of the objective function, always keeping in mind that we can be stuck in a local optimum, because the problem is not convex. Finally, once a matrix $A$ that minimizes $\varepsilon(A)$ is found, we will have determined the corresponding map $L$ thanks to the representer theorem, and we can use $A$ together with the kernel functions to transform new data.

\subsubsection{Kernel Average Neighborhood Margin Maximization (KANMM)} \label{desc:kanmm}

KANMM \cite{anmm} is the kernelized version of ANMM. In it, the data in $\mathcal{X}$ is sent to the feature space via the map $\phi \colon \R^d \to \mathcal{F}$, where ANMM is applied to obtain the linear map we are looking for.

Recall that the first step for the application of ANMM was to obtain the homogeneous and heterogeneous neighborhoods for each sample $x_i \in \mathcal{X}$. Note that for this calculation it is only necessary to compare distances in the feature space, which we have seen can be done thanks to the kernel function, through Eq. \ref{eq:dist_features}. We will denote the neighborhoods in the feature space as $N_{\phi(x_i)}^o$ y $N_{\phi(x_i)}^e$, respectively, for each $x_i$.

The scatter and compactness matrices (or endomorphisms, more in general) in the feature space are given by
\begin{align*}
    S^{\phi} = \sum\limits_{i,k \colon \phi(x_k) \in N_{\phi(x_i)}^e} \frac{(\phi(x_i)-\phi(x_k))(\phi(x_i)-\phi(x_k))^T}{|N_{\phi(x_i)}^e|} \\
    C^{\phi} = \sum\limits_{i,j \colon \phi(x_j) \in N_{\phi(x_i)}^o} \frac{(\phi(x_i)-\phi(x_j))(\phi(x_i)-\phi(x_j))^T}{|N_{\phi(x_i)}^o|}.
\end{align*}

The problem to be optimized is therefore expressed as
\begin{equation}
\begin{split}
    \max_{L \in \mathcal{L}(\mathcal{F},\R^{d'})} &\quad \tr\left(L(S^{\phi}-C^{\phi})L^T\right)  \\
    \text{s.t.: } &\quad LL^T = I.
\end{split}
\end{equation}

According to the representer theorem, $L\varphi(x_i) = AK_{.i}$, where $A$ is the matrix of coefficients of the representation theorem and $K_{.i}$ represents the $i$-th column of the kernel matrix for the training set. Then,
\begin{equation*}
    \begin{split}
        L(\phi(x_i)-\phi(x_j))(\phi(x_i)-\phi(x_j))^TL^T = A(K_{.i}-K_{.j})(K_{.i}-K_{.j})^TA^T,
    \end{split}
\end{equation*}
and if we consider the matrices
\begin{align*}
    \widetilde{S}^{\phi} = \sum\limits_{i,k \colon \phi(x_k) \in N_{\phi(x_i)}^e} \frac{(K_{.i}-K_{.k})(K_{.i}-K_{.k})^T}{|N_{\phi(x_i)}^e|} \\
    \widetilde{C}^{\phi} = \sum\limits_{i,j \colon \phi(x_j) \in N_{\phi(x_i)}^o} \frac{(K_{.i}-K_{.j})(K_{.i}-K_{.j})^T}{|N_{\phi(x_i)}^o|},
\end{align*}
it follows that the average neighborhood margin is given by
\begin{equation*}
    \gamma^L = \tr(L(S^{\phi}-C^{\phi})L^T) = \tr(LS^{\phi}L^T-LC^{\phi}L^T) = \tr(A \widetilde{S}^{\phi} A^T - A\widetilde{C}^{\phi}A^T = \tr(A(\widetilde{S}^{\phi}-\widetilde{C}^{\phi})A^T)).
\end{equation*} 

If we impose the restriction $AA^T = I$, Theorem \ref{thm:eigen_trace_opt} tells us again that we can take matrix $A$ that which contains as rows the eigenvectors of $\widetilde{S}^{\phi}-\widetilde{C}^{\phi}$ corresponding to its $d'$ largest eigenvalues. Observe that we can calculate both matrices from the kernel function, and the matrix $A$ we obtain determines the linear map, as a consequence of the representer theorem. Therefore, we have finally obtained a kernel-based method for applying ANMM in feature spaces.

\subsubsection{Kernel Distance Metric Learning through the Maximization of the Jeffrey Divergence (KDMLMJ)} \label{desc:kdmlmj}

KDMLMJ \cite{dmlmj} is the kernelized version of DMLMJ. In it, the data in $\mathcal{X}$ is sent to the feature space, where a distance is learned after applying DMLMJ.

Again, it is possible to calculate the $k$-positive and $k$-negative neighborhoods, $V_k^+(\phi(x_i))$ and $V_k^-(\phi(x_i))$, for each $x_i \in \mathcal{X}$, thanks to Eq. \ref{eq:dist_features}. It is not the same with the endomorphisms associated with the difference spaces,
\begin{align*}
    \Sigma_S^{\phi} &= \frac{1}{|S|}\sum_{i=1}^{N} \left[ \sum_{\phi(x_j) \in V_k^+(\phi(x_i))} (\phi(x_i)-\phi(x_j))(\phi(x_i)-\phi(x_j))^T\right] \\
    \Sigma_D^{\phi} &= \frac{1}{|D|}\sum_{i=1}^{N} \left[ \sum_{\phi(x_j) \in V_k^-(\phi(x_i))} (\phi(x_i)-\phi(x_j))(\phi(x_i)-\phi(x_j))^T\right].
\end{align*}

The optimization problem is given by
\[ \max_{L \in \mathcal{L}(\mathcal{F}, \R^{d'})} \quad J(L) =  \tr\left( (L\Sigma_S^{\phi} L^T)^{-1} (L\Sigma_D^{\phi} L^T) + (L \Sigma_D^{\phi} L^T)^{-1} (L \Sigma_S^{\phi} L^T) \right).\]

Again we have, as a consequence of the representer theorem, that $L\phi(x_i) = AK_{.i}$ for each $x_i \in \mathcal{X}$, where $A$ is the matrix provided by the representer theorem, and $K_{.i}$ is the $i$-th column of the kernel matrix for the training set. If, reasoning as in the previous section, we define the matrices
\begin{align*}
    U &= \frac{1}{|S|}\sum_{i=1}^{N} \left[ \sum_{\phi(x_j) \in V_k^+(\phi(x_i))} (K_{.i}-K_{.j})(K_{.i}-K_{.j})^T\right] \\
    V &= \frac{1}{|D|}\sum_{i=1}^{N} \left[ \sum_{\phi(x_j) \in V_k^-(\phi(x_i))} (K_{.i}-K_{.j})(K_{.i}-K_{.j})^T\right],
\end{align*}
we obtain that
\begin{equation*}
    \begin{split}
        \tr\left( (L\Sigma_S^{\phi} L^T)^{-1} (L\Sigma_D^{\phi} L^T) + (L \Sigma_D^{\phi} L^T)^{-1} (L \Sigma_S^{\phi} L^T) \right) = \\
        \tr\left( (AUA^T)^{-1} (AV A^T) + (A V A^T)^{-1} (A U A^T) \right).
    \end{split}
\end{equation*}

As with DMLMJ, Theorem \ref{thm:eigen_trace_ratio_sym_opt} tells us that we can find a matrix $A$ that maximizes this last equality by taking the eigenvectors of $U^{-1}V$ for which the value $\lambda + 1/\lambda$ is maximized, where $\lambda$ is the associated eigenvalue. As matrices $U$ and $V$ can be obtained from the kernel function, and $A$ determines $L$ by the representer theorem, we have obtained an algorithm for the application of DMLMJ in the feature space.

\subsubsection{Kernel Discriminant Analysis (KDA)} \label{desc:kda}

KDA \cite{kda} is the kernelized version of linear discriminant analysis. The kernelization of this algorithm will make it possible to find non-linear directions that nicely separate the data according to the criteria established in the discriminant analysis. Once again, we send the data in $\mathcal{X}$ to the feature space using the mapping $\phi\colon \R^d \to \mathcal{F}$. On that space we will apply linear discriminant analysis.

Suppose, as in LDA, that the set of possible classes is $\mathcal{C}$, of cardinal $r$, and for each $c \in \mathcal{C}$ we define $\mathcal{C}_c = \{i \in \{1,\dots,N\} \colon y_i = c\}$ and $N_c = |\mathcal{C}_c|$, with $\mu_c^{\phi}$ the mean vector of the class $c$, and $\mu^{\phi}$ the mean vector of the whole dataset, considering it within the feature space. The problem we want to solve in this case is
\begin{equation} \label{eq:kda}
    \max_{\substack{L \in \mathcal{L}(\mathcal{F}, \R^{d'}) }} \quad \tr\left((L S_w^{\phi} L^T)^{-1} (LS_b^{\phi}L^T)\right),
\end{equation}
where $S_b^{\phi}$ and $S_w^{\phi}$ are the operators that measure the between-class and within-class scatter, respectively, and are given by
\begin{align*}
    S_b^{\phi} &= \sum_{c \in \mathcal{C}}(\mu_c^{\phi} - \mu^{\phi})(\mu_c^{\phi} - \mu^{\phi})^T \\
    S_w^{\phi} &= \sum_{c \in \mathcal{C}}\sum_{i \in \mathcal{C}_c}(\phi(x_i)-\mu_c^{\phi})(\phi(x_i)-\mu_c^{\phi})^T.
\end{align*}

Again, we use the representer theorem, so that if $L \in \mathcal{L}(\mathcal{F},\R^{d'})$, then, for each $x \in \R^d$,
\[L\phi(x) = A \begin{pmatrix} K(x_1,x) \\ \vdots \\ K(x_N,x) \end{pmatrix}, \]
where $A$ is in the conditions of the representer theorem. Let us look again for an expression of the problem given in Eq. \ref{eq:kda} that depends only on the kernel function and the matrix $A$. To do this, we have to observe that for the mean vectors of each class we have
\[ L\mu_c^{\phi} = L\left(\frac{1}{N_c} \sum_{i \in \mathcal{C}_c} \phi(x_i)\right) = \frac{1}{N_c}\sum_{i \in \mathcal{C}_c}L\phi(x_i) = \frac{1}{N_c}\sum_{i \in \mathcal{C}_c}AK_{.i},  \]
where $K_{.i}$ is the $i$-th column of the kernel matrix for the training set. Similarly, for the global mean vector, we have
\[ L\mu^{\phi} = \frac{1}{N} \sum_{i=1}^NAK_{.i}. \]

Consequently,
\begin{align*}
    L(\mu_c^{\phi} - \mu^{\phi})(\mu_c^{\phi} - \mu^{\phi})^TL^T &= (L\mu_c^{\phi} - L\mu^{\phi})(L\mu_c^{\phi} - L\mu^{\phi})^T \\
                                     &=  \left( \frac{1}{N_c}\sum_{i \in \mathcal{C}_c} AK_{.i} - \frac{1}{N}\sum_{i=1}^N AK_{.i} \right)\left( \frac{1}{N_c}\sum_{i \in \mathcal{C}_c} AK_{.i} - \frac{1}{N}\sum_{i=1}^N AK_{.i} \right)^T.
\end{align*}

Note that the last expression depends only on $A$ and the kernel function. Moreover, for $x_i \in \mathcal{X}$ with $y_i = c$, we have
\begin{multline*}
    L(\phi(x_i) - \mu_c^{\phi})(\phi(x_i) - \mu_c^{\phi})^TL^T = (L\phi(x_i) - L\mu_c^{\phi})(L\phi(x_i) - L\mu_c^{\phi})^T \\
                                     = \left(AK_{.i} - \frac{1}{N_c}\sum_{j \in \mathcal{C}_c} AK_{.j} \right)\left(AK_{.i} - \frac{1}{N_c}\sum_{j \in \mathcal{C}_c} AK_{.j} \right)^T \\
                                     = \left(AK_{.i} - \frac{1}{N_c}\sum_{j \in \mathcal{C}_c} AK_{.j} \right)\left(K_{.i}^TA^T - \frac{1}{N_c}\sum_{j \in \mathcal{C}_c} K_{.j}^TA^T \right) \\
                                     = AK_{.i}K_{.i}^TA^T - \frac{1}{N_c}\sum_{j\in \mathcal{C}_c}AK_{.i}K_{.j}^TA^T - \frac{1}{N_c}\sum_{j\in \mathcal{C}_c}AK_{.j}K_{.i}^TA^T + \frac{1}{N_c^2}\sum_{j \in \mathcal{C}_c}\sum_{l \in \mathcal{C}_c}AK_{.j}K_{.l}^TA^T.
\end{multline*}

By summing in $i \in \mathcal{C}_c$, we obtain
\begin{multline*}
    \sum_{i\in\mathcal{C}_c}L(\phi(x_i) - \mu_c^{\phi})(\phi(x_i) - \mu_c^{\phi})^TL^T \\
    = \sum_{i \in \mathcal{C}_c}\left[AK_{.i}K_{.i}^TA^T - \frac{1}{N_c}\sum_{j\in \mathcal{C}_c}AK_{.i}K_{.j}^TA^T - \frac{1}{N_c}\sum_{j\in \mathcal{C}_c}AK_{.j}K_{.i}^TA^T + \frac{1}{N_c^2}\sum_{j \in \mathcal{C}_c}\sum_{l \in \mathcal{C}_c}AK_{.j}K_{.l}^TA^T\right] \\
            = \sum_{i \in \mathcal{C}_c}AK_{.i}K_{.i}^TA^T - \frac{2}{N_c}\sum_{i \in \mathcal{C}_c}\sum_{j \in \mathcal{C}_c}AK_{.i}K_{.j}^TA^T + \frac{1}{N_c^2}\sum_{i \in \mathcal{C}_c}\sum_{j \in \mathcal{C}_c}\sum_{l \in \mathcal{C}_c}AK_{.j}K_{.l}^TA^T \\
            = \sum_{i \in \mathcal{C}_c}AK_{.i}K_{.i}^TA^T - \frac{2}{N_c}\sum_{i \in \mathcal{C}_c}\sum_{j \in \mathcal{C}_c}AK_{.i}K_{.j}^TA^T + \frac{N_c}{N_c^2}\sum_{j \in \mathcal{C}_c}\sum_{l \in \mathcal{C}_c}AK_{.j}K_{.l}^TA^T \\
            = \sum_{i \in \mathcal{C}_c}AK_{.i}K_{.i}^TA^T - \frac{1}{N_c}\sum_{i \in \mathcal{C}_c}\sum_{j \in \mathcal{C}_c}AK_{.i}K_{.j}^TA^T \\
            = AK_cK_c^TA^T - AK_c\left(\frac{1}{N_c}\mathbbm{1}\right)K_c^TA^T \\
            = AK_c\left(I - \frac{1}{N_c}\mathbbm{1}\right)K_c^TA^T,
\end{multline*} 
where $\mathbbm{1}\in \mathcal{M}_{N_c}(\R)$ is a square matrix with the value 1 in all its entries, and $K_c \in \mathcal{M}_{N\times N_c}$ is a kernel matrix whose entries are the values of the kernel function between all the samples in $\mathcal{X}$ and the samples with class $c$. Again, this last expression depends only on $A$ and the kernel function.

If we finally define
\begin{equation*}
    \begin{split}
        U_c &= \frac{1}{N_c}\sum_{i \in \mathcal{C}_c}K_{.i} \in \R^N, c \in \mathcal{C}\\
        U_{\mu} &= \frac{1}{N}\sum_{j=1}^N K_{.i} \in \R^N \\
        U &= \sum_{c \in \mathcal{C}} N_c(U_c - U_{\mu})(U_c - U_{\mu})^T \in S_N(\R) \\
        V &= \sum_{c \in \mathcal{C}} K_c\left(I - \frac{1}{N_c}\mathbbm{1}\right)K_c^T \in S_N(\R),
    \end{split}
\end{equation*}
we can conclude that
\begin{equation*}
    \tr\left((L S_w^{\phi} L^T)^{-1}(LS_b^{\phi}L^T)\right) = \tr\left((AVA^T)^{-1}(AUA^T) \right),
\end{equation*}
where $U$ and $V$ are computable using the kernel function. Therefore, we obtain a problem equivalent to the original given in Eq. \ref{eq:kda}, but in terms of $A$, for which Theorem \ref{thm:eigen_trace_ratio_opt} states that, if $U$ is positive definite, we can maximize the value of the trace by taking as rows of $A$ the eigenvectors of $V^{-1}U$ corresponding to its $d'$ largest eigenvalues. In this way, since $A$ determines $L$ thanks to the representer theorem, we obtain a kernel-based method for the application of discriminant analysis in feature spaces.

\bibliography{main}

\end{document}